\documentclass[aos]{imsart}

\RequirePackage{amsthm,amsmath,amsfonts,amssymb}

\usepackage[authoryear]{natbib}

\usepackage[T1]{fontenc}
\usepackage[utf8]{inputenc}
\usepackage{bm}
\usepackage{verbatim}
\usepackage{float}
\usepackage{algorithm}
\usepackage{algpseudocode}
\usepackage{placeins}
\usepackage{needspace}
\usepackage{changepage}
\algrenewcommand{\algorithmicrequire}{\textbf{Input:}}
\usepackage{color}
\usepackage{threeparttable}
\usepackage{graphicx}
\usepackage{mdframed}
\usepackage{xcolor}

\usepackage{multirow}
\usepackage{makecell}
\usepackage{booktabs}
\usepackage{array}
\usepackage{hhline}

\usepackage{lscape}
\usepackage{mathtools}
\usepackage{romanbar}
\usepackage{enumerate}
\usepackage{rotating}
\usepackage{tikz}
\usetikzlibrary{matrix}
\usetikzlibrary{calc,intersections}
\usepackage{mathrsfs}

\startlocaldefs
\theoremstyle{plain}

\newtheorem{theorem}{Theorem}[section]
\newtheorem{corollary}{Corollary}[section]
\newtheorem{lemma}[theorem]{Lemma}

\theoremstyle{remark}
\newtheorem{definition}[theorem]{Definition}

\newtheorem{assumption}{Assumption}[section]

\newcommand{\argmin}{\mathop{\mathrm{argmin}}}

\endlocaldefs

\RequirePackage[colorlinks,citecolor=blue,urlcolor=blue]{hyperref}
\begin{document}

\begin{frontmatter}
\title{Federated Targeted Maximum Likelihood Estimation} \runtitle{Federated Targeted Maximum Likelihood Estimation} \runauthor{D. Li, F. Wang and K. Gan} \pdftitle{Federated Targeted Maximum Likelihood Estimation} \pdfauthor{Diyang Li, Fei Wang, Kyra Gan}

\thankslabel[a]{e1} \thankslabel[b]{e2} \thankslabel[c]{e3}

\begin{aug}
\author[A]{\fnms{Diyang}~\snm{Li}\ead[label=e1]{diyang01@cs.cornell.edu}}, \author[A,B]{\fnms{Fei}~\snm{Wang}\ead[label=e2]{few2001@med.cornell.edu}} \and \author[A]{\fnms{Kyra}~\snm{Gan}\ead[label=e3]{kyragan@cornell.edu}}

\address[A]{Cornell University\printead[presep={,\ }]{e1,e3}} \address[B]{Weill Cornell Medicine\printead[presep={,\ }]{e2}}
\end{aug}

\begin{abstract}
The evidence behind a scientific or operational decision is often held by hospitals, banks, or registries that cannot pool individual observations. Cross-silo federated learning moves computation to the data and exchanges agreed summaries. Targeted maximum likelihood estimation (TMLE) refines a flexible initial fit through low-dimensional, target-directed fluctuations guided by the efficient influence function, yielding plug-in estimators that respect the model and support efficient inference. TMLE itself, however, has remained a fully centralized procedure. No general federated form of it has existed. Our paper introduces the first federated TMLE algorithm. We federate targeting itself, for an arbitrary target, loss, and fluctuation family, through two complementary frameworks. FedTMLE-G aggregates local gradients and reproduces centralized targeting step for step. FedTMLE-L lets each institution complete its own fluctuation fit before a single exchange of fitted updates, trading synchronized fidelity for local autonomy. For gradient aggregation, we develop a finite-precision protocol that transmits changes rather than values and certifies targeting accuracy within explicit bounds on exchanges and bits. A description-length analysis of the accepted updates then shows that this finite communication leaves numerical targeting error negligible against sampling uncertainty. The cost of computing an estimator is thus distinct from the complexity of selecting it. Our analysis also indicates that keeping data local is not itself a privacy guarantee of TMLE, since instability of full-record reconstruction need not prevent recovery of a specified sensitive attribute. For a personalized version of local averaging, institutions retain their own estimates and leave once local targeting is complete. A nonconvex convergence bound charges the improvement forfeited through averaging to disagreement among local fits and exposes a tradeoff between equal institutional influence and the sampling variability of small silos.
\end{abstract}


\end{frontmatter}

\setlength{\textfloatsep}{12pt plus 2pt minus 2pt} \setlength{\intextsep}{12pt plus 2pt minus 2pt} \setlength{\floatsep}{12pt plus 2pt minus 2pt}

\section{Introduction}\label{sec:intro} Scientific and operational decisions often depend on evidence held by institutions that cannot pool their records. A hospital may observe too few patients to assess a treatment reliably; banks may need a broader view of credit risk while retaining customer-level transactions locally. Obligations to patients and customers, local governance, and institutional boundaries constrain how these records can be used \citep{rieke2020future,yang2019federated}. Cross-silo federated learning allows institutions to collaborate by computing on their own datasets and exchanging agreed summaries \citep{mcmahan2017communication,kairouz2021advances}. This can expand the evidence available for analysis without creating a central repository of sensitive records \citep{dayan2021federated}. For scientific inference, the resulting procedure must also provide a statistically reliable estimate of the quantity of interest.

Targeted learning organizes estimation around a specified feature of the data-generating distribution, such as a treatment effect or a policy's expected outcome \citep{vanderlaan2011targeted}. Targeted maximum likelihood estimation (TMLE) starts from an initial estimate, possibly fitted by flexible machine-learning methods \citep{vanderlaan2007super}, and refines it through low-dimensional, target-directed changes called fluctuations \citep{vanderlaan2006targeted}. The efficient influence function guides these changes to correct the leading estimation error relevant to the target. Evaluating the target at the updated distribution preserves its model-implied constraints and supports efficient inference under appropriate statistical conditions \citep{gruber2010bounded,vandervaart1998asymptotic}. For many causal targets, TMLE also provides double robustness, so consistency can hold when one nuisance model is misspecified but its complementary model is estimated consistently \citep{vanderlaan2006targeted,bang2005doubly}. In its general iterative form, TMLE fits a fluctuation, updates the distribution, and constructs the next fluctuation at that distribution \citep{vanderlaan2016onestep,li2025optimization}. This sequence connects empirical optimization to the statistical estimating equation.

Targeted learning knows how to turn a fit into a defensible estimate but cannot see past the records in front of it; federated learning sees across institutions but comes back with a predictor. Each has spent years perfecting precisely what the other lacks, yet no general federated or distributed TMLE exists, and an estimate that could draw on every hospital in a consortium still draws on the patients of one. The strengths of both go unused where they would matter most. Therefore, we introduce the first general federated formulation of standard TMLE with a unified theory of communication, statistical accuracy, information disclosure, and personalized optimization. We federate the targeting algorithm itself through repeated, coordinated refinement across private datasets. The target, loss, and admissible fluctuation family remain abstract, so the construction is not tied to one inferential target. Each accepted fit changes the working model for the next fit. The analysis must therefore track optimization progress, statistical completion, and communication across this evolving sequence of models.

Both proposed frameworks start from a shared initial estimate. FedTMLE-G aggregates local gradients to solve the pooled fluctuation problem (Algorithm~\ref{alg:gradient-aggregation}). With matched optimization and sampling rules, it reproduces centralized targeting in exact arithmetic and inherits the corresponding optimization behavior. FedTMLE-L lets each institution complete its fit in the common current fluctuation submodel, then averages the fitted fluctuation parameters in one exchange (Algorithm~\ref{alg:local-averaging}). Local computation can thus proceed between exchanges, but the averaged minimizers need not solve the pooled problem, even with sample-size weights. The two frameworks can consequently follow different targeting trajectories.

A small targeting error may require many exchanges, but not every exchange needs a long message. 
For FedTMLE-G, we quantify the communication needed to certify a prescribed targeting accuracy through a finite-precision protocol with bounds on both rounds and transmitted bits (Section~\ref{sec:theory-gradient}). Clients encode changes from previously communicated gradients. The analysis controls these changes across inner steps and submodel reconstruction without requiring agreement among local gradients. Slowly varying gradients can then be sent with fewer bits while preserving the rounded reference computation. The number of synchronization rounds is accounted for separately.

The statistical analysis accounts for the data-dependent choice of the final estimate through its accepted targeting history. Given the initial state, this history reconstructs the estimate without recording every intermediate gradient query. Theorem~\ref{thm:statistical} uses its description length to control data-dependent fitting and stopping. Under the stated statistical conditions, finite communication suffices to make numerical targeting error negligible relative to sampling uncertainty (Corollary~\ref{cor:inference}). An inner solve may require many exchanges while contributing only one accepted update to the description used in this bound.

Keeping records local does not, by itself, make communicated gradients confidential \citep{zhu2019deep,melis2019exploiting}. A server may infer properties of a silo's data from its responses, particularly when it can select successive queries. Section~\ref{sec:robustness} studies record reconstruction and recovery of a specified sensitive attribute under an explicit bounded-uncertainty model. These tasks can have different sensitivity to transcript uncertainty, allowing an attribute to be recovered even when reconstruction of the full record is unstable. The analysis makes the family of queries accepted by each client part of the disclosure model.

For FedTMLE-L, nonzero local fits can cancel, leaving the shared density unchanged even though local targeting is unfinished. We study a personalized formulation in which a client whose exact local fit selects zero retains its current density and targeting state and leaves, while the other clients continue (Section~\ref{sec:personalized}). The retained estimate incorporates any earlier shared updates, and completion is certified on that client's own data. Our nonconvex convergence bound controls the improvement still available at each client's active or retained density. It separates sampling variation from the improvement forfeited through averaging (Theorem~\ref{thm:personalized-convergence}). Nearby local fits with similar derivatives incur a smaller averaging penalty; persistent disagreement can obstruct completion. The bound also shows how equal institutional weighting differs from sample-size weighting in its sensitivity to small silos. Retired clients perform no further local work or communication, and their outputs are unaffected by later shared updates.

Experiments (Section~\ref{sec:experiments}) compare targeting progress under unequal client sizes, different data partitions, and limited computation and communication. The ordering observed over elapsed time can differ from that obtained under a fixed communication budget, reflecting the different amounts of local computation and synchronization required by the two frameworks.

\section{Preliminaries}\label{sec:preliminaries}
TMLE reduces the empirical correction in a plug-in estimator by minimizing the empirical loss over a sequence of finite-dimensional fluctuation models. This section reviews the statistical decomposition that motivates the correction and then sets out the iterative procedure in the general form of \citet{vanderlaan2006targeted}. Following \citet{li2025optimization}, we keep two operations distinct throughout: optimization within a fixed fluctuation model, and the density update that follows it.

\subsection{The target parameter and its canonical gradient}\label{sec:plugin}
Let $O_1,\ldots,O_n$ be independent and identically distributed observations with unknown density $p_0$ in a model $\mathcal M$, which may be nonparametric or semiparametric. All densities are taken with respect to a common dominating measure, denoted by $do$. The subscript in $p_0$ marks the true density; superscripts index targeting iterations. The quantity of interest is a fixed-dimensional, pathwise differentiable functional $\Psi:\mathcal M\to\mathbb R^d$. A plug-in estimator starts from a density estimate $p^0\in\mathcal M$ and reports $\Psi(p^0)$. The initial estimate may be fitted by any method, including flexible ensemble learners \citep{vanderlaan2007super}, but accurate prediction or density estimation alone need not make the error negligible at the scale required for inference on $\Psi(p_0)$.

For integrable scalar or vector functions $f$, write $P_nf:=n^{-1}\sum_{i=1}^nf(O_i)$ and $P_0f:=\int f(o)p_0(o)\,do$. Vector norms $\|\cdot\|$ are Euclidean, and $\|\cdot\|_1$ and $\|\cdot\|_\infty$ denote the sum and the maximum of absolute coordinates; for a scalar function, $\|\cdot\|_\infty$ is the uniform norm on the stated full-measure set. For a scalar or vector function $f$, $\|f\|_{L^2(p_0)}^2:=P_0\|f\|^2$, and $\|\cdot\|_{\mathrm{op}}$ is the induced operator norm of a matrix or multilinear form. Unsubscripted logarithms are natural; $\log_2$ is reserved for bit counts.

The efficient influence function (EIF) $D(p)(o)\in\mathbb R^d$ is the canonical gradient of $\Psi$ at $p$: it represents the derivative of $\Psi$ along regular perturbations of $p$, in that for every regular one-dimensional submodel $\{p(t)\}$ through $p(0)=p$ with score $s(o)=\left.\frac{\partial}{\partial t}\log p(t)(o)\right|_{t=0}$,
\begin{equation*}
	\left.\frac{d}{dt}\Psi(p(t))\right|_{t=0} =\int D(p)(o)s(o)p(o)\,do,
\end{equation*}
and its components belong to the tangent space of $\mathcal M$ at $p$, the closure in $L^2(p)$ of the linear span of the model's scores. The tangent-space requirement is what singles out the canonical gradient among the influence functions representing the same derivative. In an unrestricted nonparametric model the tangent space is the whole mean-zero subspace of $L^2(p)$; in a semiparametric model the restriction matters. In either case $D(p)$ is centered and square-integrable under $p$, that is, $\int D(p)(o)p(o)\,do=\mathbf0$ and $\int\|D(p)(o)\|^2p(o)\,do<\infty$, and at $p_0$ its covariance is the efficiency bound for regular estimation of $\Psi(p_0)$ \citep{kennedy2016semiparametric,vandervaart1998asymptotic}.

The centering identity holds under $p$ and need not hold under $p_0$ or under the empirical distribution: $D(p^0)$ has mean zero under $p^0$, but its empirical mean $P_nD(p^0)$ need not vanish. This discrepancy is what TMLE targets.

\subsection{The estimating equation and statistical error}\label{sec:statistical-error}
For a candidate density $p$ at which the quantities below are finite, define the von Mises remainder
\begin{equation}\label{eq:remainder}
	R_2(p,p_0) :=\Psi(p)-\Psi(p_0)+P_0D(p).
\end{equation}
The subscript reflects the usual expectation that $R_2$ is second order in the error of $p$; whether it is depends on smoothness of the target and on control of the relevant estimation errors, and we keep $R_2$ as an explicit term in our results. Using that $D(p_0)$ is centered under $p_0$, one obtains the exact decomposition
\begin{equation}\label{eq:plugin-expansion}
	\begin{aligned}
		\Psi(p)-\Psi(p_0) ={}&P_nD(p_0)-P_nD(p)\\
		&\qquad +(P_n-P_0)\{D(p)-D(p_0)\}+R_2(p,p_0).
	\end{aligned}
\end{equation}
The first term has the desired canonical-gradient form. The second, $-P_nD(p)$, is the empirical correction left in the plug-in estimator, a sample quantity rather than a bias in expectation. The third is the empirical-process term from evaluating the influence function at an estimated density, and the last is the nonlinear remainder \citep{vanderlaan2011targeted,vandervaart1998asymptotic}.

A one-step correction adds $P_nD(p^0)$ directly to $\Psi(p^0)$. TMLE instead moves the density: it replaces $p^0$ by a density $p\in\mathcal M$ that approximately solves the estimating equation
\begin{equation}\label{eq:eif-equation}
	\frac1n\sum_{i=1}^nD(p)(O_i)=\mathbf0,
\end{equation}
and reports the plug-in value $\Psi(p)$, which therefore stays in the model-implied range $\Psi(\mathcal M)$ \citep{vanderlaan2006targeted,gruber2010bounded}. If the empirical-process term and $R_2(p,p_0)$ are both $o_{\mathbb P}(n^{-1/2})$, then driving the empirical EIF mean to the same order gives the asymptotically linear expansion $\Psi(p)-\Psi(p_0)=P_nD(p_0)+o_{\mathbb P}(n^{-1/2})$. Both conditions concern the final, data-dependent density rather than the initial estimator, and neither follows from solving \eqref{eq:eif-equation}. The empirical-process term is usually controlled by $L^2(p_0)$ convergence of $D(p)$ to $D(p_0)$ combined with a Donsker condition \citep{kennedy2016semiparametric,vandervaart1996weak} or with sample splitting \citep{zheng2011crossvalidated,chernozhukov2018double}; Section~\ref{sec:theory} takes a different route, through the description length of the targeting history. Equation~\eqref{eq:eif-equation} may have many solutions. Targeting starts from $p^0$ and moves to one of them; it does not seek a unique density or the solution closest to $p_0$ in any global sense.

\subsection{Targeting as a sequence of fluctuation optimizations}\label{sec:tmle-template}
Let $L(p)(o)$ be a loss for the density, calibrated so that its population risk is minimized at the truth, $p_0\in\argmin_{p\in\mathcal M}\int L(p)(o)p_0(o)\,do$; the log-likelihood loss $L(p)(o)=-\log p(o)$ is the standard example. Targeting also requires the loss and the fluctuation submodel to be linked by a derivative identity. We state the identity as a joint condition on the pair and otherwise leave their functional forms free.

\begin{definition}[Targeted fluctuation]\label{def:fluctuation}
	At a density $p^k\in\mathcal M$, a targeted fluctuation is a family $\{p^k(\epsilon):\epsilon\in\mathcal E^k\}\subseteq\mathcal M$, indexed by a set $\mathcal E^k\subseteq\mathbb R^d$ containing an open neighborhood of $\mathbf0$, such that
	\begin{equation}\label{eq:local-score}
		p^k(\mathbf0)=p^k, \qquad \left.\nabla_\epsilon L\bigl(p^k(\epsilon)\bigr)(o) \right|_{\epsilon=\mathbf0}=D(p^k)(o).
	\end{equation}
	The identity is required of the versions of $L$ and $D$ evaluated at the observations, and $L(p^k(\epsilon))(o)$ is differentiable in $\epsilon$ wherever its gradient is used.
\end{definition}

The admissible set $\mathcal E^k$ restricts $\epsilon$ so that every permitted fluctuation is a normalized, nonnegative density in $\mathcal M$; it may be all of $\mathbb R^d$ or a proper subset. The sign in \eqref{eq:local-score} fixes our derivative convention. In a nonparametric model with the log-likelihood loss, for instance, the exponential tilt $p^k(\epsilon)\propto p^k\exp\{-\epsilon^\top D(p^k)\}$ satisfies \eqref{eq:local-score} with $\mathcal E^k=\mathbb R^d$ whenever $D(p^k)$ is bounded. The identity is imposed only at the base point: at a nonzero $\epsilon$, the loss derivative need not equal $D(p^k(\epsilon))$, and we do not require a universal least-favorable submodel \citep{vanderlaan2016onestep}.

In optimization notation, write
\begin{equation}\label{eq:fluctuation-objective}
	\ell_k(\epsilon):=P_nL(p^k(\epsilon)), \qquad r^k:=\|P_nD(p^k)\|,
\end{equation}
so that, by \eqref{eq:local-score},
\begin{equation*}
	\nabla\ell_k(\mathbf0)=P_nD(p^k), \qquad r^k=\|\nabla\ell_k(\mathbf0)\|.
\end{equation*}
An exact targeting step minimizes $\ell_k$ over $\mathcal E^k$ and updates the density, as in Algorithm~\ref{alg:centralized-tmle}. The fluctuation map is held fixed during the minimization; the EIF and the submodel are recomputed at the updated density. Whenever an exact minimizer is invoked, we take one to exist and, if it is not unique, select it by a fixed measurable rule. First-order conditions are used only at interior minimizers.

\begin{algorithm}[!htbp]
	\caption{Iterative TMLE}\label{alg:centralized-tmle} \small
	\begin{algorithmic}[1]
		\Require $\{O_i\}_{i=1}^n$, $p^0\in\mathcal M$, $L,D,\Psi$; $K\in\mathbb N$, $\delta\geq0$.
		\For{$k=0,\ldots,K$}
		\State $\displaystyle r^k\gets \left\|\frac1n\sum_{i=1}^nD(p^k)(O_i)\right\|$.
		\If{$r^k\leq\delta\ \lor\ k=K$}
		\State \Return $(p^k,\Psi(p^k),r^k)$.
		\EndIf
		\State $p^k(\cdot):\mathcal E^k\to\mathcal M$ satisfying \eqref{eq:local-score}.
		\State $\displaystyle\epsilon^k\in \argmin_{\epsilon\in\mathcal E^k}\ell_k(\epsilon) =\argmin_{\epsilon\in\mathcal E^k} \frac1n\sum_{i=1}^nL\bigl(p^k(\epsilon)\bigr)(O_i)$.
		\State $p^{k+1}\gets p^k(\epsilon^k)$.
		\EndFor
	\end{algorithmic}
\end{algorithm}

Each exact step satisfies $\ell_{k+1}(\mathbf0)=\ell_k(\epsilon^k)\leq\ell_k(\mathbf0)$, since $\epsilon=\mathbf0$ is feasible. The decrease is strict whenever $r^k>0$, that is, whenever the EIF equation at $p^k$ is unsolved, because the direction $-\nabla\ell_k(\mathbf0)$ is then locally feasible and
\begin{equation*}
	\left.\frac{d}{dt} \ell_k\bigl(-t\nabla\ell_k(\mathbf0)\bigr)\right|_{t=0} =-(r^k)^2<0.
\end{equation*}

The iteration produces two first-order conditions that must be kept apart. At an interior minimizer, the inner problem gives $\nabla\ell_k(\epsilon^k)=\mathbf0$, stationarity in the current submodel. The targeting equation at the updated density is $\nabla\ell_{k+1}(\mathbf0)=\mathbf0$, evaluated at the origin of the submodel rebuilt at $p^{k+1}$. Because \eqref{eq:local-score} holds only at base points, these are different conditions, and it is the second that the stopping rule in Algorithm~\ref{alg:centralized-tmle} tests. The relation between a zero fit and a solved equation is likewise one-sided: an exact zero fit $\epsilon^k=\mathbf0$ implies $r^k=0$, whereas $r^k=0$ says only that the origin is a stationary point of $\ell_k$, which in a nonconvex problem need not be a global minimizer.

Over an infinite sequence of steps, convergence of the losses, of the fluctuation parameters to zero, and of the densities themselves are separate questions. For example, if $\nabla\ell_k$ is $\beta$-Lipschitz along the segment from $\mathbf0$ to $\epsilon^k$, with $\beta$ uniform in $k$, then $r^k\leq\beta\|\epsilon^k\|+\|\nabla\ell_k(\epsilon^k)\|$: the EIF residual is controlled by the size of the accepted fluctuation and the inner optimization error, but neither term controls the density sequence without further regularity \citep{li2025optimization}. Since $r^k$ is the norm of the empirical correction in \eqref{eq:plugin-expansion} at $p^k$, Algorithm~\ref{alg:centralized-tmle} returns it with the estimate. Setting $\delta=0$ tests for an exact root; when the iteration budget $K$ is exhausted first, the returned residual need not satisfy that test.

\section{Algorithms}\label{sec:algorithms}
Silo (client) $m$ holds the observations $\{O_i:i\in\mathcal I_m\}$, where $\mathcal I_1,\ldots,\mathcal I_M$ partition $\{1,\ldots,n\}$, and $N_m:=|\mathcal I_m|>0$, so that $\sum_{m=1}^MN_m=n$. The observations follow the same i.i.d.\ sampling model as in Section~\ref{sec:preliminaries}, although the realized local empirical distributions may differ. All silos start from a common initial estimator $p^0$, and the question is how to federate targeting from that point on. If $p^0$ is fitted on the targeting sample itself, the statistical analysis in Section~\ref{sec:theory} accounts for that dependence rather than conditioning on it.

At each outer iteration, all silos work with the same loss, EIF, fluctuation coordinates, and submodel. Write the local empirical loss and its restriction to the current submodel as
\begin{equation*}
	\ell_m(p):=\frac1{N_m}\sum_{i\in\mathcal I_m}L(p)(O_i), \qquad \ell_{m,k}(\epsilon):=\ell_m(p^k(\epsilon)),
\end{equation*}
so that $\ell_m$ takes a density as its argument, whereas $\ell_{m,k}$, like $\ell_k$, takes fluctuation coordinates, and gradients of the latter are with respect to those coordinates. The pooled objective and the pooled EIF mean decompose as
\begin{equation*}
	\ell_k=\sum_{m=1}^M\frac{N_m}{n}\ell_{m,k}, \qquad \nabla\ell_k(\mathbf0)=\sum_{m=1}^M\frac{N_m}{n} \nabla\ell_{m,k}(\mathbf0),
\end{equation*}
identities that hold in finite samples for any partition of the data.

Both algorithms take the loss $L$, the EIF $D$, the target $\Psi$, and the fluctuation construction as shared inputs. The server and the silos communicate synchronously, and $a\xrightarrow{x}b$ denotes a message $x$ from $a$ to $b$. Communicating $p^k$ and $p^k(\cdot)$ means sending, or updating, a shared representation from which both can be evaluated, including any normalizing constants; its dimension need not be $d$. Observations never leave their silo, although this alone is no privacy guarantee, as Section~\ref{sec:robustness} shows. Both algorithms return the pooled residual $r^k$ of \eqref{eq:fluctuation-objective} together with the estimate, so that a small update can be told apart from a small empirical EIF mean.

\subsection{FedTMLE-G: gradient aggregation}\label{sec:gradient-aggregation}
FedTMLE-G distributes the gradient evaluation and leaves the centralized inner optimization otherwise unchanged. At every admissible $\epsilon$,
\begin{equation*}
	\nabla\ell_k(\epsilon) =\frac1n\sum_{m=1}^M\sum_{i\in\mathcal I_m} \nabla_\epsilon L\bigl(p^k(\epsilon)\bigr)(O_i),
\end{equation*}
so the server recovers the pooled gradient by adding the local gradient sums and dividing by $n$, or equivalently by weighting the local gradient means by $N_m/n$. Away from the origin, the silos evaluate the loss gradient in the fixed current submodel; the EIF identity \eqref{eq:local-score} is used only at $\epsilon=\mathbf0$.

Write $\epsilon^{k,t}$ for the fluctuation parameter at inner step $t$, and let $T_k\geq1$ be the number of inner steps. Each step uses a nonempty pooled batch $\mathcal B_t^k$ of indices, with $\mathcal B_{m,t}^k=\mathcal B_t^k\cap\mathcal I_m$ its restriction to silo $m$, every occurrence of an index retained. Full-batch gradient descent (GD) takes $\mathcal B_t^k=\{1,\ldots,n\}$; stochastic gradient descent (SGD) draws a pooled batch of fixed size $B$ uniformly at random, a multiset when sampling is with replacement. The batches may be generated online from a shared sampling rule, and an empty local batch contributes zero.

\begin{algorithm}[!htbp]
	\caption{FedTMLE-G}\label{alg:gradient-aggregation} \small
	\begin{algorithmic}[1]
		\Require $\{O_i:i\in\mathcal I_m\}_{m=1}^M$, $p^0$, $K$, $\delta$; $\{T_k,\eta_{k,t}>0,\mathcal B_t^k\}$.
		\For{$k=0,\ldots,K$}
		\State $\textsf{server}\xrightarrow{\ p^k,\,p^k(\cdot)\ }m, \quad m=1,\ldots,M$.
		\State $m\xrightarrow{\ \nabla\ell_{m,k}(\mathbf0)\ }\textsf{server}, \quad m=1,\ldots,M$.
		\State $\displaystyle r^k\gets\left\|\sum_{m=1}^M \frac{N_m}{n}\nabla\ell_{m,k}(\mathbf0)\right\|$.
		\If{$r^k\leq\delta\ \lor\ k=K$}
		\State \Return $(p^k,\Psi(p^k),r^k)$.
		\EndIf
		\State $\epsilon^{k,0}\gets\mathbf0$.
		\For{$t=0,\ldots,T_k-1$}
		\For{$m=1,\ldots,M$ \textbf{in parallel}}
		\State $\mathcal B_{m,t}^k\gets\mathcal B_t^k\cap\mathcal I_m, \qquad\textsf{server}\xrightarrow{\ \epsilon^{k,t},\,\mathcal B_{m,t}^k\ }m$.
		\State $\displaystyle g_m^{k,t}\gets \sum_{i\in\mathcal B_{m,t}^k} \left.\nabla_\epsilon L\bigl(p^k(\epsilon)\bigr)(O_i) \right|_{\epsilon=\epsilon^{k,t}}$.
		\State $m\xrightarrow{\ g_m^{k,t}\ }\textsf{server}$.
		\EndFor
		\State $\displaystyle\epsilon^{k,t+1}\gets\epsilon^{k,t} -\frac{\eta_{k,t}}{|\mathcal B_t^k|}\sum_{m=1}^Mg_m^{k,t}$.
		\EndFor
		\State $\epsilon^k\gets\epsilon^{k,T_k},\qquad p^{k+1}\gets p^k(\epsilon^k)$.
		\EndFor
	\end{algorithmic}
\end{algorithm}

The base density and the fluctuation map stay fixed during the $T_k$ inner steps, and the step sizes must keep the iterates inside $\mathcal E^k$. The pooled batch is drawn as a whole, by the server or by the shared rule, rather than silo by silo: independent local sampling, such as drawing a batch of the same size from every silo, does not in general sample uniformly from the $n$ observations, and it need not reproduce pooled SGD even when reweighting makes the resulting gradient unbiased.

\paragraph*{Inherited optimization behavior} In exact arithmetic, FedTMLE-G reproduces the inner parameters and the outer densities of the centralized implementation whenever the initialization, fluctuation construction, optimizer (including its step-size and feasibility rules), stopping rules, and realized pooled batches coincide. With exact inner minimization it inherits the empirical-loss descent of standard TMLE, and under the corresponding regularity conditions its untruncated targeting sequence inherits the stopping-condition relationships and the iterative convergence of \citet[Theorems~1 and~4]{li2025optimization}. Under this matched computation, repartitioning the data across silos changes neither the loss trajectory nor the EIF residuals nor the limiting density. Trajectory-dependent behavior carries over as well, including the possibility of passing an EIF root inside a fluctuation submodel before reaching the fitted minimizer \citep{li2025optimization}. When the inner problem is solved with finitely many GD or SGD steps, the comparator is the correspondingly inexact centralized implementation, and the exact-solve convergence results do not cover arbitrary inner optimization errors.

Communication follows the pattern of synchronous gradient aggregation \citep{mcmahan2017communication}. Every inner gradient step is one synchronized exchange, and the full-data EIF check that opens an outer iteration is another, although under full-batch GD that check also supplies the first inner gradient. Each contributing silo uploads a $d$-dimensional gradient per exchange; transmitting the density representation and the batch assignments is a separate cost.

\subsection{FedTMLE-L: local averaging}\label{sec:local-averaging}
FedTMLE-L replaces the pooled inner solve by exact local fits in the common fluctuation submodel,
\begin{equation*}
	\hat\epsilon_m^k\in \argmin_{\epsilon\in\mathcal E^k}\ell_{m,k}(\epsilon), \qquad m=1,\ldots,M,
\end{equation*}
which the server averages with fixed weights $\alpha_m\geq0$, $\sum_m\alpha_m=1$:
\begin{equation}\label{eq:local-average}
	\epsilon^k=\sum_{m=1}^M\alpha_m\hat\epsilon_m^k, \qquad p^{k+1}=p^k(\epsilon^k).
\end{equation}
What is averaged is the fluctuation parameter, in the shared coordinates, and the average is applied to the shared base density. Neither updated densities nor target values are averaged, and no silo runs a TMLE trajectory of its own to completion.

\begin{algorithm}[!htbp]
	\caption{FedTMLE-L}\label{alg:local-averaging} \small
	\begin{algorithmic}[1]
		\Require $\{O_i:i\in\mathcal I_m\}_{m=1}^M$, $p^0$, $K$, $\delta$; $\alpha_m\geq0$, $\sum_{m=1}^M\alpha_m=1$.
		\For{$k=0,\ldots,K$}
		\State $\textsf{server}\xrightarrow{\ p^k,\,p^k(\cdot)\ }m, \quad m=1,\ldots,M$.
		\For{$m=1,\ldots,M$ \textbf{in parallel}}
		\State $m\xrightarrow{\ \nabla\ell_{m,k}(\mathbf0)\ }\textsf{server}$.
		\If{$k<K$}
		\State $\displaystyle\hat\epsilon_m^k\in \argmin_{\epsilon\in\mathcal E^k}\ell_{m,k}(\epsilon)$.
		\State $m\xrightarrow{\ \hat\epsilon_m^k\ }\textsf{server}$.
		\EndIf
		\EndFor
		\State $\displaystyle r^k\gets\left\|\sum_{m=1}^M \frac{N_m}{n}\nabla\ell_{m,k}(\mathbf0)\right\|$.
		\If{$r^k\leq\delta\ \lor\ k=K$}
		\State \Return $(p^k,\Psi(p^k),r^k)$.
		\EndIf
		\State $\displaystyle\epsilon^k\gets \sum_{m=1}^M\alpha_m\hat\epsilon_m^k$.
		\If{$\epsilon^k\notin\mathcal E^k\ \lor\ \epsilon^k=\mathbf0$}
		\State \Return $(p^k,\Psi(p^k),r^k)$.
		\EndIf
		\State $p^{k+1}\gets p^k(\epsilon^k)$.
		\EndFor
	\end{algorithmic}
\end{algorithm}

A silo needs no response from the server between its EIF evaluation and its local fit, so the two uploads travel together: one attempted outer update costs one collection of local results, however much local optimization it involved. At the final iteration $k=K$ only the EIF evaluation is needed. All silos remain active throughout, including those whose previous local fit was zero.

\paragraph*{Weights and the estimating equation} Sample-size weights $\alpha_m=N_m/n$ let each local fit count in proportion to its silo, whereas uniform weights $\alpha_m=1/M$ give every silo the same influence on the averaged fluctuation. Neither choice turns the average into a minimizer. The weighted loss $\sum_m\alpha_m\ell_{m,k}(\epsilon)$ has derivative $\sum_{m=1}^M(\alpha_m/N_m)\sum_{i\in\mathcal I_m}D(p^k)(O_i)$ at zero, but an average of minimizers does not in general minimize it, and even with sample-size weights \eqref{eq:local-average} need not solve the pooled inner problem. The stopping test in the algorithm is the pooled equation \eqref{eq:eif-equation} in either case.

\paragraph*{Feasibility and stationarity} The average in \eqref{eq:local-average} must lie in $\mathcal E^k$; admissibility of each local fit does not ensure this, whereas a convex admissible domain does. The method also depends on the choice of common fluctuation coordinates, since a nonlinear reparameterization changes the average of the local solutions.

Stationarity does not survive averaging either. Each interior local minimizer satisfies $\nabla\ell_{m,k}(\hat\epsilon_m^k)=\mathbf0$, but these conditions hold at different points, so adding them does not establish stationarity at $\epsilon^k$; likewise, $\ell_{m,k}(\hat\epsilon_m^k)\leq\ell_{m,k}(\mathbf0)$ compares each silo's own loss at its fit and at the origin and does not imply pooled descent for the averaged update. In its pattern of local computation followed by parameter averaging, FedTMLE-L resembles federated averaging \citep{mcmahan2017communication,karimireddy2020scaffold}; the difference is that its coordinates change with the EIF and the base density at every outer iteration, which the targeting analysis in Section~\ref{sec:personalized} has to take into account.

\paragraph*{Cancellation and continued participation} Nonzero local solutions can cancel in the average, leaving $p^k(\epsilon^k)=p^k$ while $r^k$ remains positive. Continuing from this state under fixed selection rules would reproduce the same local solutions and the same zero aggregate, so when $\epsilon^k=\mathbf0$ and $r^k>\delta$ the algorithm stops and returns the unchanged density with its nonzero residual, reporting stagnation rather than certifying completion; it makes no attempt to resolve the cancellation.

For an individual silo, a zero local fit $\hat\epsilon_m^k=\mathbf0$ implies $N_m^{-1}\sum_{i\in\mathcal I_m}D(p^k)(O_i)=\mathbf0$, since the origin is an interior point of $\mathcal E^k$. This solves the silo's own estimating equation at the current density, but the next shared update can move the density off that root, so a silo that wants the common output of Algorithm~\ref{alg:local-averaging} has to keep participating after a zero fit. Section~\ref{sec:personalized} studies the alternative, in which such a silo freezes its own output and leaves. When all local fits are zero at the same shared density, every local equation and the pooled equation hold there.

\section{Theoretical Properties and Scope}\label{sec:theory}
This section develops the theory behind the two frameworks. We begin with FedTMLE-G and ask how much communication is needed to reach a prescribed empirical EIF residual. Equivalence to centralized targeting does not answer this question, so we specify a full-batch implementation with finite-precision messages and a certified stopping rule (Section~\ref{sec:theory-gradient}). Its bit count is governed by the changes between successive gradients, whereas the accompanying statistical bound is governed by the accepted fluctuations and, when one is required, a code for the initialization. Section~\ref{sec:robustness} then examines what the gradient channel discloses, separating the instability of record reconstruction from the recovery of a specified attribute by a corrupted server that selects its own queries. Section~\ref{sec:personalized} analyzes the personalized outputs of local averaging, retained by each client at its own retirement time. All proofs are in Section~\ref{app:proofs}.

\subsection{Assumptions and computational setting}\label{sec:theory-assumptions}
We keep the notation $\ell_k$, $\ell_{m,k}$, $\epsilon^{k,t}$, and $r^k$ of Sections~\ref{sec:tmle-template}--\ref{sec:algorithms}, together with the empirical-measure and norm conventions. Optimization statements are made for the realized observations; for probabilistic statements, the algorithmic regularity conditions are required to hold almost surely, and a same-sample initial estimator is never conditioned on.

\begin{assumption}[Feasible smooth targeting]\label{ass:smooth}
	The fluctuation construction satisfies \eqref{eq:local-score}. There is a finite $\beta>0$ such that, for every silo $m$, every outer iteration $k$, and every pair of parameters $u,v$ in the working region,
	\begin{equation*}
		\|\nabla\ell_{m,k}(u)-\nabla\ell_{m,k}(v)\| \leq\beta\|u-v\|.
	\end{equation*}
	The working region contains the origin, every queried parameter, and every GD step segment, and lies in $\mathcal E^k$. The empirical loss is bounded below on the model, so that
	\begin{equation}\label{eq:loss-budget}
		\Delta:=\ell_0(\mathbf0)-\inf_{p\in\mathcal M}P_nL(p)<\infty.
	\end{equation}
	The infimum in \eqref{eq:loss-budget} may be replaced by a common finite lower bound on the admissible densities traversed by the procedure.
\end{assumption}

\begin{assumption}[Regularity under re-centering]\label{ass:recentering}
	There is a finite $\kappa\geq0$ such that, on the same working regions,
	\begin{equation*}
		\left\|\frac1{N_m}\sum_{i\in\mathcal I_m} \{D(p^k(u))-D(p^k(v))\}(O_i)\right\| \leq\kappa\|u-v\|.
	\end{equation*}
	The initial local gradients are finite, and $G$ is a constant with $G\geq\max_m\|\nabla\ell_{m,0}(\mathbf0)\|_\infty$.
\end{assumption}

The two Lipschitz conditions concern different functions of the fluctuation parameter. Assumption~\ref{ass:smooth} controls the gradient of the loss within a fixed submodel; Assumption~\ref{ass:recentering} controls how each local empirical EIF mean changes as the base density moves along the submodel. The second does not identify the endpoint derivative of the old submodel with the EIF of the new one, and the two differ precisely because \eqref{eq:local-score} holds only at base points. Uniform bounds on the observation-level derivatives along admissible connecting segments suffice for both conditions, and neither requires a unique minimizer or a contraction of the targeting sequence \citep{bottou2018optimization}. The constant $G$ enters the communication analysis only through the first messages; the protocol does not need to know it, and later gradients need not obey any common magnitude bound.

The communication accounting takes the loss, the EIF, and the fluctuation rule to be public, so that every party can reconstruct $p^k$ and its submodel from the initial shared state and the accepted history $(\epsilon^0,\ldots,\epsilon^{k-1})$ without access to any other silo's observations. The shared state must therefore carry every normalization and nuisance quantity that this reconstruction needs; if the fluctuation rule requires further pooled statistics, their transmission is a separate cost.

Communication takes place over reliable synchronous point-to-point links. We write $R$ for the total number of upload--broadcast rounds and $\mathsf B$ for the total number of bits in the numerical messages and two-bit control flags, summed over all links. Initialization, network headers, and cryptographic overhead are setup or transport costs and are accounted for separately \citep{bonawitz2017practical}. Local gradient evaluation is treated as an oracle, and the only numerical error analyzed is the rounding specified below; all parties use the same arithmetic and tie-breaking rule, with integer accumulators that do not overflow. The constant $\beta$ may be replaced by any publicly represented positive rational upper bound.

\begin{assumption}[Statistical decoding and envelope]\label{ass:statistical-code}
	For statistical assertions, the shared construction and its public tuning rules are fixed independently of the targeting sample. Either the initialization is independent of that sample, in which case we condition on it and set $H_{\mathrm{init}}=0$, or every sample-dependent part of the initial shared state, including any sample-dependent tuning information, has a finite prefix-free binary description under a decoder fixed before the sample is observed; $H_{\mathrm{init}}$ denotes its length. Decoding is measurable and determines the terminal density from the initial state and the accepted history. For the same finite, deterministic $G$, the versions of the EIF satisfy
	\begin{equation*}
		|D_j(p)(o)|\leq G,\qquad j=1,\ldots,d,
	\end{equation*}
	uniformly over $p_0$ and every density obtained from a valid decoded history, for $P_0$-almost every $o$.
\end{assumption}

Assumption~\ref{ass:statistical-code} is used only for the statistical result. Its decoding requirement is mild: the argument ranges over the countably many outputs of a finitely described implementation, not over the statistical model, which need not be countable \citep{cover2006elements}. The point of the description is adaptivity. An initialization fitted on the same sample is itself a data-dependent choice, so it must be part of the description rather than treated as fixed conditioning information \citep{dwork2015preserving}.

The next two assumptions concern only the personalized FedTMLE-L analysis in Section~\ref{sec:personalized}. We use the local empirical losses $\ell_m$ and their fluctuation restrictions $\ell_{m,k}$ from Section~\ref{sec:algorithms}. For a generic retained density $p$, $p(\epsilon)$ denotes its associated fluctuation, with every retained quantity held fixed.

\begin{assumption}[Regularity for exact local fits]\label{ass:local-regularity}
The fluctuation rule is a deterministic function of the retained state. Each local fluctuation problem attains its minimum at a selected interior point. Selection follows a fixed, measurable rule that chooses zero whenever zero is a minimizer. The segments joining selected local solutions to one another, to their weighted average, and to zero are admissible, as is the segment from zero to $-\nabla\ell_{m,k}(\mathbf0)/\beta$. Each $\ell_m$ has a finite lower bound $\underline\ell_m$ over the current and retained densities.

On these working regions, the observation-level loss is three times continuously differentiable in the fluctuation coordinates. For finite $G_1>0$, $\beta_4\geq0$, and the $\beta$ of Assumption~\ref{ass:smooth},
\begin{equation}\label{eq:personalized-regularity}
\begin{aligned}
\|\nabla_\epsilon L(p(u))(o)\|&\leq G_1, &\|\nabla_\epsilon^2L(p(u))(o)\|_{\mathrm{op}}&\leq\beta,\\
\|\nabla_\epsilon^3L(p(u))(o) -\nabla_\epsilon^3L(p(v))(o)\|_{\mathrm{op}} &\leq\beta_4\|u-v\|.
\end{aligned}
\end{equation}
The bounds hold on a common $P_0$-full set, and tensor norms are induced by Euclidean norms. Upon retirement, a client retains its density and every quantity defining its fluctuation rule.
\end{assumption}

Exact minimization has the same oracle meaning as in Algorithm~\ref{alg:local-averaging}; a finite nonconvex solver need not implement it. The admissibility condition for the step from the origin holds automatically on an unrestricted fluctuation domain. A uniform admissible neighborhood and bounded gradients also suffice after increasing $\beta$. The constant $\beta_4$ controls the fourth-order remainder in the averaging bound. It is distinct from the EIF re-centering constant $\kappa$ and places no sign restriction on curvature.

For a fixed density, clients sampled from the same law have equal expected loss derivatives, and sampling bounds control the differences between their empirical derivatives. Applying such a comparison at $p^k$ requires care because the algorithm chooses that density using the same samples. The next assumption provides a uniform bound over a state family fixed before sampling. Its covering condition controls how many derivative functions must be distinguished throughout the possible trajectories.

\begin{assumption}[Common-law sampling and derivative complexity]
\label{ass:local-complexity}
The silo partition is independent of the observations, which are i.i.d.\ from $p_0$ conditional on that partition. Before observing the targeting sample, fix a working family of density states and associated fluctuation regions that contains all possible iterates of both weighting modes, including any same-sample initial fit and all retained targeting quantities. The bounds in \eqref{eq:personalized-regularity} hold throughout this family. Over these states and regions and all $\|v\|=1$, the scalar functions
\begin{equation*}
\frac{v^\top\nabla_\epsilon L(p(\epsilon))(\cdot)}{G_1}, \qquad \frac{v^\top\nabla_\epsilon^2L(p(\epsilon))(\cdot)v}{\beta}
\end{equation*}
are measurable and admit an internal $1/n$-cover of size $\mathcal N_n$ in the uniform norm on the common full-measure set. Thus each function is within $1/n$ of one of $\mathcal N_n$ representatives belonging to this family.
\end{assumption}

For example, a compact, fixed-dimensional representation of the complete state with uniformly Lipschitz derivative maps gives $\log\mathcal N_n=O(\log n)$, as shown in Section~\ref{app:local-derivatives}. Such uniform derivative control is used in empirical nonconvex optimization \citep{mei2018landscape,vandervaart1996weak}. It restricts the statistical complexity of the state family while allowing nonconvex objectives and distinct local minimizers. Smoothness in $\epsilon$ alone is insufficient because the base density also varies. The common-law assumption is needed only for the sampling refinement.

\subsection{Communication and statistical accuracy of FedTMLE-G}
\label{sec:theory-gradient}
Algorithm~\ref{alg:gradient-aggregation} specifies which gradients are aggregated, but not how they are represented numerically nor when an inner solve stops, and both choices matter for communication. Sending every coordinate in full absolute form pays the precision cost at every round, even when successive gradients differ little \citep{alistarh2017qsgd,lin2022differentially}. And a prescribed terminal EIF residual does not require every inner problem to be solved exactly; in a nonconvex problem, a finite GD run is not an exact global minimizer in any case.

Algorithm~\ref{alg:coded-gradient} addresses both. It rounds the local gradients to a common fine grid, runs GD on their rounded aggregate, and checks the EIF only at the origin of a freshly constructed submodel. Each message encodes the change from the previously transmitted rounded value, so a coordinate that moves by one grid unit costs only the codeword for that unit, whatever its current magnitude. The receiver recovers the rounded values exactly, and with them the rounded optimization trajectory: the grid is the only source of numerical error, and differential coding reduces the cost of representing values that the receiver already largely knows.

What is specific to iterative TMLE is re-centering. Smoothness controls how gradients change within an inner solve, but accepting a fluctuation changes the base density and hence the submodel in which the next gradient is queried. Assumption~\ref{ass:recentering} bounds this additional change through the accepted fluctuation, whose total size is in turn controlled by the decrease in empirical loss, so the protocol can keep its code memory across outer iterations instead of transmitting a fresh full-precision initialization at each re-centering. Gradient-difference compression and differential quantization are established tools in distributed optimization \citep{gorbunov2021marina,lin2022differentially}; the bound here has to cover inner steps, submodel changes, and the queries needed for the EIF certificate at once.

Fix an EIF tolerance $\delta>0$ and a public dyadic grid spacing $h$ with
\begin{equation}\label{eq:lattice-spacing}
	\frac{\delta}{8\sqrt d}<h\leq\frac{\delta}{4\sqrt d}.
\end{equation}
Number all gradient evaluations consecutively by $r$, and when round $r$ queries the state $(k,t)$ define the local \emph{mean} gradient and the integer quantities
\begin{equation*}
	\begin{aligned}
		g_{m,r}&:=\nabla\ell_{m,k}(\epsilon^{k,t}), &z_{m,r}&:=\operatorname{round}(g_{m,r}/h),\\
		s_r&:=\sum_{m=1}^M N_m z_{m,r}, &q_r&:=\operatorname{round}(s_r/n), &\widetilde g_r&:=h q_r.
	\end{aligned}
\end{equation*}
Thus $g_{m,r}$ is a mean, related to the batch sum of Algorithm~\ref{alg:gradient-aggregation} by $g_{m,r}=g_m^{k,t}/N_m$ for a full-batch query, and all rounding is coordinatewise to a nearest integer with a fixed tie-breaking rule. Silo $m$ transmits the difference $z_{m,r}-z_{m,r-1}$ and the server broadcasts $q_r-q_{r-1}$, both losslessly encoded, starting from $z_{m,0}=q_0=s_0=\mathbf0$. The server therefore needs only the streaming update
\begin{equation*}
	s_r=s_{r-1}+\sum_{m=1}^M N_m(z_{m,r}-z_{m,r-1}),
\end{equation*}
and the integer memories are kept across re-centering even though the fluctuation parameter is reset to zero.

Write $\operatorname{enc}$ for coordinatewise signed Elias-gamma coding \citep{elias1975universal}: an integer $a$ costs at most $3+2\log_2(1+|a|)$ bits, each message consists of exactly $d$ self-delimiting codewords, and Section~\ref{app:communication} gives the bijection and the decoding rule. To its broadcast the server appends the two-bit action
\begin{equation}\label{eq:protocol-action}
	\mathsf a_r=
	\begin{cases}
		00,&\|\widetilde g_r\|>3\delta/4,\\
		01,&\|\widetilde g_r\|\leq3\delta/4,\quad t>0,\\
		10,&\|\widetilde g_r\|\leq3\delta/4,\quad t=0.
	\end{cases}
\end{equation}
Actions $01$ and $10$ respond to a small gradient at two different places. At an endpoint $t>0$, a small gradient indicates stationarity in the old submodel, so action $01$ accepts the inner fit and requests a query at the origin of the reconstructed submodel, where the gradient equals the EIF mean at the new density. A small gradient at that origin, action $10$, certifies termination. The extra origin query counts toward $R$.

The server weights local means by $N_m/n$ at every query, so unequal silo sizes leave the pooled objective unchanged, and the two rounding operations give
\begin{equation}\label{eq:rounding-guarantee}
	\|\widetilde g_r-\nabla\ell_k(\epsilon^{k,t})\| \leq h\sqrt d\leq\delta/4.
\end{equation}
The bound concerns the current gradient only: because decoding recovers every absolute rounded value exactly, earlier rounding errors do not accumulate through the differences. The threshold $3\delta/4$ in \eqref{eq:protocol-action} leaves $\delta/4$ for rounding, so the final origin test certifies $r^K\leq\delta$. Algorithm~\ref{alg:coded-gradient} is thus a full-batch specialization of FedTMLE-G with adaptive inner termination and no outer cap. Throughout this subsection, $K$ is the realized number of accepted outer updates rather than a prescribed budget, and the last returned quantity is an upper bound on $r^K$, since the unrounded residual is never evaluated exactly.

For comparison, suppose every queried coordinate is bounded by $G$. A fixed-width representation then uses $b=\lceil\log_2(2\lceil G/h\rceil+1)\rceil$ bits per integer, and sending $z_{m,r}$ and $q_r$ in absolute form costs $2MdbR+2MR$ bits including control flags, a numerical payload of $2MdbR$ for this upload--broadcast protocol. Differential coding delivers the same rounded values without a global magnitude bound and without paying the full width at every round. The two implementations compute the same rounded trajectory; what separates that trajectory from real-arithmetic GD is the rounding itself.

\begin{algorithm}[!htbp]
	\caption{Finite-precision FedTMLE-G with differential coding}
	\label{alg:coded-gradient}
	\small
	\begin{algorithmic}[1]
		\Require Shared $p^0,L,D,\Psi$, fluctuation rule, $\{N_m\}$, $\beta,\delta,h$ satisfying \eqref{eq:lattice-spacing}.
		\State $k\gets0$, $t\gets0$, $r\gets0$, $\epsilon^{0,0}\gets\mathbf0$; $s_0,q_0,z_{1,0},\ldots,z_{M,0}\gets\mathbf0$.
		\State Construct the shared initial submodel $p^0(\cdot)$.
		\Loop
		\State $r\gets r+1$.
		\For{$m=1,\ldots,M$ \textbf{in parallel}}
		\State $z_{m,r}\gets \operatorname{round}(\nabla\ell_{m,k}(\epsilon^{k,t})/h)$.
		\State $m\xrightarrow{\ \operatorname{enc}(z_{m,r}-z_{m,r-1})\ } \textsf{server}$.
		\EndFor
		\State $s_r\gets s_{r-1}+\sum_{m=1}^M N_m(z_{m,r}-z_{m,r-1})$.
		\State $q_r\gets\operatorname{round}(s_r/n)$, $\widetilde g_r\gets hq_r$; $\mathsf a_r\gets\text{\eqref{eq:protocol-action}}$.
		\State $\textsf{server} \xrightarrow{\ (\operatorname{enc}(q_r-q_{r-1}),\,\mathsf a_r)\ }m, \quad m=1,\ldots,M$.
		\State Each silo reconstructs $q_r$ from the received difference, sets $\widetilde g_r=hq_r$, and retains $z_{m,r}$.
		\If{$\mathsf a_r=00$}
		\State $\epsilon^{k,t+1}\gets \epsilon^{k,t}-\widetilde g_r/\beta$, $t\gets t+1$.
		\ElsIf{$\mathsf a_r=01$}
		\State $\epsilon^k\gets\epsilon^{k,t}$, $p^{k+1}\gets p^k(\epsilon^k)$.
		\State $k\gets k+1$, $t\gets0$, $\epsilon^{k,0}\gets\mathbf0$; reconstruct $p^k(\cdot)$.
		\Else
		\State \Return $(p^k,\Psi(p^k), \|\widetilde g_r\|+h\sqrt d)$.
		\EndIf
		\EndLoop
	\end{algorithmic}
\end{algorithm}

\begin{theorem}[Amortized communication]\label{thm:communication}
	Under Assumptions~\ref{ass:smooth}--\ref{ass:recentering}, Algorithm~\ref{alg:coded-gradient} terminates at a density $p^K$ and, for a universal constant $C$, satisfies
	\begin{equation}\label{eq:main-communication}
		\begin{aligned}
			r^K&\leq\delta,\qquad R\leq1+\frac{64\beta\Delta}{3\delta^2},\\
			\mathsf B&\leq C Md\left\{ \left(1+\frac{\beta\Delta}{\delta^2}\right) \log_2\!\left(4+\frac{\kappa}{\beta}\right) +\log_2\!\left(2+\frac{G\sqrt d}{\delta}\right)\right\}.
		\end{aligned}
	\end{equation}
	The bit count includes every inner query, re-centering check, upload, broadcast, and control flag. The protocol and the fixed-width reference implementation have identical decoded iterates whenever the latter's range is sufficient.
\end{theorem}

The round bound counts every inner GD iteration and every EIF check; the bit bound then accounts for how much each of those exchanges transmits. When $\kappa/\beta$ is bounded, the leading bit cost is $O(Md\beta\Delta/\delta^2)$, that is, a bounded number of bits per coordinate and silo per round, plus a single logarithmic charge for the initial gradient. A fixed-width implementation pays a per-round precision cost of order $\log_2(G\sqrt d/\delta)$ bits per coordinate instead. Tightening $\delta$ therefore raises the worst-case round count quadratically but does not multiply the differential protocol's bit bound by a precision logarithm.

An individual message can still be long when a gradient changes a great deal; the bound is amortized, controlling the total length through the loss decrease over the whole execution. To see the dependence on the realized trajectory, let $S$ be the number of GD steps, so that $R=S+K+1$ and $K\leq S$. The proof gives the sharper bound
\begin{equation}\label{eq:pathwise-bit-bound}
	\begin{aligned}
		\mathsf B\leq{}8MdR +4Md\log_2\!\left(2+\frac{8G\sqrt d}{\delta}\right)
		+4Md(R-1)\log_2\!\left( 3+\frac{8(2+\kappa/\beta)\sqrt{6\beta\Delta S}} {(R-1)\delta}\right),
	\end{aligned}
\end{equation}
with the last term absent when $R=1$. After the initial exchange, the cost is driven by the accumulated variation of the local gradients, which the proof controls through
\begin{equation}\label{eq:variation-summary}
	\begin{aligned}
		\sum_{\text{GD steps }r}\|\widetilde g_r\|^2 &\leq6\beta\Delta,\\
		\sum_{r=2}^R\|g_{m,r}-g_{m,r-1}\| &\leq(2+\kappa/\beta)\sqrt{6\beta\Delta S}.
	\end{aligned}
\end{equation}
The first inequality charges the squared rounded gradients of all GD steps, across all submodels, to a single loss budget; the second converts that budget into a bound on each silo's total gradient variation. Re-centering enters through $\kappa$ and requires no separate initialization charge at each outer iteration. The ratio $\kappa/\beta$, which compares how fast local empirical EIF means move under changes of the base density with the smoothness of the inner objective, enters the worst-case bit bound only logarithmically.

Nothing in the bound requires the local gradients to vanish or to align across silos: nonzero local contributions may cancel in the pooled gradient while each changes slowly enough to admit short difference codes. The savings are in transmitted bits, not in synchronization rounds, and the theorem is an upper bound on the cost of this protocol, not a lower bound for federated estimation in general. The server's streaming accumulator adds $O(d)$ integer coordinates to the sample counts and the density representation; their bit lengths can grow over the execution.

For the statistical analysis we describe the accepted targeting history. Given the initial state, the final density is determined by the accepted fluctuations alone; no record of how each inner fit was computed is needed, so an inner solve of many gradient queries contributes a single $d$-dimensional vector. Every GD step is a multiple of $h/\beta$, so each accepted vector satisfies $\beta\epsilon^k/h\in\mathbb Z^d$. We encode $K+1$ by an Elias-gamma code, followed by the $Kd$ signed integers in $(\beta\epsilon^0/h,\ldots,\beta\epsilon^{K-1}/h)$, preceded by the initial-state code when one is required, and write $H_K$ for the resulting description length. This code is an analytical device and is never transmitted. A prefix-free description argument \citep{blumer1987occam,cover2006elements} applied to it controls the adaptivity of the final density without charging for every optimization query.

\begin{theorem}[History complexity and statistical remainder]
	\label{thm:statistical}
	Under the assumptions of Theorem~\ref{thm:communication}, the accepted-history code satisfies $H_0=H_{\mathrm{init}}+1$ and, for $K\geq1$,
	\begin{equation}\label{eq:history-length}
		H_K\leq H_{\mathrm{init}} +C\left\{1+Kd\log_2\!\left( 2+\frac{\beta\Delta}{K\delta^2}\right)\right\},
	\end{equation}
	where $C$ is universal and any initial-state code is as specified above. If Assumption~\ref{ass:statistical-code} also holds, then for each $\xi\in(0,1)$ set $x_K=H_K\log2+\log(2d/\xi)$. With probability at least $1-\xi$,
	\begin{equation}\label{eq:history-statistical-bound}
		\begin{aligned}
			&\|\Psi(p^K)-\Psi(p_0)-P_nD(p_0)\|\\
			&\quad\leq\delta+\|R_2(p^K,p_0)\| +\sqrt{\frac{2x_K}{n}}\, \|D(p^K)-D(p_0)\|_{L^2(p_0)} +\frac{4G\sqrt d\,x_K}{3n}.
		\end{aligned}
	\end{equation}
	The empirical-process event holds simultaneously over all valid decoded histories. It can therefore be evaluated at the algorithm's certified endpoint even when the stopping index $K$ depends on the sample.
\end{theorem}

The four terms in \eqref{eq:history-statistical-bound} have distinct sources. The tolerance $\delta$ bounds the remaining empirical correction through the communication certificate. The nonlinear remainder $R_2(p^K,p_0)$ requires control of the terminal nuisance estimates and is left as it is. The last two terms are the price of evaluating the EIF at a density selected with the same sample \citep{dwork2015preserving}: the description length $H_K$ measures the complexity of that selection, and the $L^2(p_0)$ factor shrinks the leading empirical-process penalty as the estimated EIF approaches the true one. Because the concentration event covers every valid decoded history, the bound survives the algorithm's data-dependent choice of history and stopping index, without conditioning on that index and without a predetermined cap on the number of rounds. Complexity is measured by description length rather than by a Donsker condition on the entire model.

\begin{corollary}[First-order inference at finite communication]
	\label{cor:inference}
	Let $d$ be fixed. Suppose the preceding assumptions hold with bounded $\beta$, $G$, and $\kappa/\beta$, and $\Delta=O_{\mathbb P}(1)$. For deterministic $a_n\to\infty$, use $\delta_n=(a_n\sqrt n)^{-1}$ and \eqref{eq:lattice-spacing}. Then
	\begin{equation}\label{eq:inference-cost}
		R=O_{\mathbb P}(na_n^2),\qquad \mathsf B=O_{\mathbb P}(Md\,na_n^2).
	\end{equation}
	If the actual terminal estimates satisfy
	\begin{equation}\label{eq:inference-conditions}
		\begin{aligned}
			\sqrt n\,\|R_2(p^K,p_0)\|&=o_{\mathbb P}(1),\\
			\sqrt{H_K+\log n}\, \|D(p^K)-D(p_0)\|_{L^2(p_0)}&=o_{\mathbb P}(1),\\
			\frac{G\sqrt d\,(H_K+\log n)}{\sqrt n}&=o_{\mathbb P}(1),
		\end{aligned}
	\end{equation}
	then
	\begin{equation}\label{eq:canonical-limit}
		\begin{aligned}
			\Psi(p^K)-\Psi(p_0)&=P_nD(p_0)+o_{\mathbb P}(n^{-1/2}),\\
			\sqrt n\{\Psi(p^K)-\Psi(p_0)\} &\rightsquigarrow \mathcal N_d\!\left(\mathbf0, P_0[D(p_0)D(p_0)^\top]\right).
		\end{aligned}
	\end{equation}
\end{corollary}

The corollary identifies a communication scale at which numerical targeting error is negligible relative to sampling uncertainty: with $\delta_n=o(n^{-1/2})$, the empirical correction drops out of the leading influence-function expansion, provided the statistical conditions in \eqref{eq:inference-conditions} hold. The round and bit bounds are worst-case sufficient costs, and a particular execution may communicate much less.

The separation between computing an estimator and selecting it is visible in the history term. When the number of accepted outer updates stays bounded, the history remains short however long the inner computation: with independent initialization and subpolynomial $a_n$ it has $H_K=O_{\mathbb P}(\log n)$ bits, while the inner optimization may take $O_{\mathbb P}(na_n^2)$ rounds. In general,
\begin{equation}\label{eq:history-asymptotic}
	H_K-H_{\mathrm{init}} =O_{\mathbb P}\!\left( 1+Kd\log\!\left(2+\frac{na_n^2}{K\vee1}\right)\right),
\end{equation}
so a growing $K$ inflates the history term and can eventually violate the statistical conditions. Those conditions refer to the estimator actually computed, rounded and finitely stopped; properties of an ideal exact-solve trajectory do not establish them. Finally, \eqref{eq:canonical-limit} has the canonical-gradient covariance, but semiparametric efficiency in the regular sense also requires regularity under local alternatives, which the corollary does not address.

\subsection{Record reconstruction and selective disclosure}
\label{sec:robustness}
The gradient responses that FedTMLE-G elicits from a silo, round after round, can reveal features of its data even when reconstructing individual records from them is unstable \citep{zhu2019deep,melis2019exploiting,engl1996regularization}, and equivalence to a centralized computation offers no protection against this. We therefore analyze two tasks separately, under bounded uncertainty in the transcript: reconstruction of the records themselves, and recovery of a specified attribute of the data.

Fix a silo $m$ and let $x=(o_i)_{i\in\mathcal I_m}$ be a candidate dataset, with lower-case $o_i$ denoting candidate record values as distinct from the random observations $O_i$. The other silos' data and everything disclosed earlier, including the initialization, are held fixed, and every candidate must be compatible with this information. At a fixed accepted state and fluctuation parameter, define
\begin{equation*}
	g(o)=\nabla_\epsilon L(p(\epsilon))(o), \qquad G_g(x)=\frac1{N_m}\sum_{i\in\mathcal I_m}g(o_i).
\end{equation*}
The response $G_g$ is the local mean gradient that FedTMLE-G requests; under honest execution it equals $\nabla\ell_{m,k}(\epsilon^{k,t})$, and at $\epsilon=\mathbf0$ the score identity gives $g=D(p)$. Let $\mathcal G$ be the family of queries that a silo accepts under the prescribed model, loss, and fluctuation rule. A corrupted server may issue any finite sequence from this family, chosen adaptively and with repetitions, and the silo answers every accepted query truthfully. For $T\geq1$ full-silo queries, stack the responses and the perturbations as $Y=(Y_1^\top,\ldots,Y_T^\top)^\top$ and $e=(e_1^\top,\ldots,e_T^\top)^\top\in\mathbb R^{dT}$, and consider the channel
\begin{equation}\label{eq:robust-channel}
	Y_t=G_{g_t}(x)+e_t, \qquad \|e\|^2=\sum_{t=1}^T\|e_t\|^2\leq\eta^2.
\end{equation}
The budget $\eta$ bounds the uncertainty in the transcript without assigning it a distribution; $\eta=0$ is the exact-response case. Multiplying responses and error by $N_m$ gives the same channel for raw sums, and since every query uses the whole silo, queries that single out individual records are excluded. A lower bound under bounded error is not a privacy guarantee for any particular rounding rule; the lossless code of Algorithm~\ref{alg:coded-gradient}, for instance, reveals its rounded inputs exactly.

For the candidate classes specified below, $\mathcal R_T^{\mathrm{rec}}(\eta)$ and $\mathcal R_T^{\varphi}(\eta)$ denote the minimax risks of the two tasks under \eqref{eq:robust-channel}: the infimum, over adaptive query policies and estimators, of the worst-case error over candidates and perturbations. Policies are deterministic, any random seed being conditioned on. Record error is measured by $\operatorname{dist}$, the largest recordwise Euclidean distance after optimal matching of the undisclosed block, and attribute error in absolute value. Sections~\ref{app:reconstruction} and~\ref{app:disclosure} give the formal definitions and the proofs.

\begin{theorem}[Instability of record reconstruction]
	\label{thm:reconstruction}
	Let $2\leq s\leq N_m$. Suppose $s$ undisclosed records may vary freely on the compatible segment $\{o_\circ+uv:|u|\leq\rho\}$, where $\rho>0$ and $\|v\|=1$, while the remaining records are fixed. Assume the accepted queries are $C^s$ on a neighborhood of this segment and, for some $0<C_s<\infty$,
	\begin{equation}\label{eq:query-smoothness}
		\sup_{g\in\mathcal G}\sup_{|u|\leq\rho} \left\|\frac{d^s}{du^s}g(o_\circ+uv)\right\|\leq C_s.
	\end{equation}
	For every $0<a\leq\rho$, two compatible datasets satisfy
	\begin{equation}\label{eq:collision-pair}
		\operatorname{dist}(x_+,x_-)\geq\frac{a}{3s}, \qquad \sup_{g\in\mathcal G}\|G_g(x_+)-G_g(x_-)\| \leq\frac{C_s a^s}{N_m2^{s-1}(s-1)!}.
	\end{equation}
	Consequently, for every $\eta\geq0$,
	\begin{equation}\label{eq:reconstruction-risk}
		\mathcal R_T^{\mathrm{rec}}(\eta)\geq \frac1{6s}\min\left\{\rho, \left(\frac{2^sN_m(s-1)!\eta}{C_s\sqrt T}\right)^{1/s}\right\}.
	\end{equation}
\end{theorem}

The two datasets in \eqref{eq:collision-pair} differ by moving $s$ nearby records a distance of order $a$, yet cancellation among them changes the aggregate response by an amount of order $a^s/N_m$, and this holds for the same pair under every accepted query, including those an adaptive server selects after changing its state. For a fixed uncertainty budget, the lower bound \eqref{eq:reconstruction-risk} therefore scales as $\eta^{1/s}T^{-1/(2s)}$ until the radius cap is reached. The mechanism is moment cancellation combined with divided differences, as in collision-dependent inverse problems \citep{deboor2005divided,batenkov2021super}. Three qualifications are in order. The rate describes an obstruction to reconstruction, not the accuracy of any particular attack. It concerns worst-case datasets and says nothing about typical protection under i.i.d.\ sampling. And it is vacuous at $\eta=0$, while at any $\eta$ it does not imply that a sensitive feature is hidden, which is the subject of the next result. The smoothness condition \eqref{eq:query-smoothness} is imposed on the observations, unlike Assumptions~\ref{ass:smooth}--\ref{ass:recentering}, which concern the fluctuation coordinates.

For attribute recovery, let the candidate set be a compatible ball $\overline B(x_\circ,\rho)\subset\mathbb R^q$, whose center and radius represent what the attacker already knows, and let $\varphi(x)$ be a prespecified scalar attribute. With the output-by-input convention for the response Jacobian $\nabla_xG_g$, assume that the Jacobians exist on a neighborhood of the ball and satisfy, with a common $H_g$ for all accepted queries,
\begin{equation}\label{eq:attribute-smoothness}
	\begin{aligned}
		\|\nabla_xG_g(x)-\nabla_xG_g(x')\|_{\mathrm{op}}&\leq H_g\|x-x'\|, &&g\in\mathcal G,\\
		\|\nabla\varphi(x)-\nabla\varphi(x')\|&\leq H_\varphi\|x-x'\|.
	\end{aligned}
\end{equation}
For a fixed accepted query block, stack the Jacobians $\nabla_xG_{g_t}(x_\circ)$ into $A\in\mathbb R^{dT\times q}$; every infimum over $A$ below is restricted to matrices obtained from such blocks. Define
\begin{equation}\label{eq:disclosure-design}
	\omega_T(\eta)=\inf_A\inf_{w\in\mathbb R^{dT}} \bigl\{\rho\|\nabla\varphi(x_\circ)-A^\top w\|+\eta\|w\|\bigr\}, \qquad \eta\geq0.
\end{equation}
The first term is the variation of the attribute that a weighted combination of the responses fails to explain; the second is the uncertainty that those weights amplify. Minimizing their sum is an optimal-recovery problem in the sense of \citet{donoho1994statistical}, and it does not presuppose that the attribute is identified.

\begin{theorem}[Minimax selective disclosure]\label{thm:disclosure}
	Under \eqref{eq:attribute-smoothness}, if $\eta\geq H_g\rho^2\sqrt T/2$, then
	\begin{equation}\label{eq:disclosure-sandwich}
		\begin{aligned}
			\left[\omega_T\!\left(\eta-\frac{H_g\rho^2\sqrt T}{2}\right) -\frac{H_\varphi\rho^2}{2}\right]_+ &\leq\mathcal R_T^{\varphi}(\eta)\\
			&\leq \omega_T\!\left(\eta+\frac{H_g\rho^2\sqrt T}{2}\right) +\frac{H_\varphi\rho^2}{2}.
		\end{aligned}
	\end{equation}
	For every $\eta\geq0$, the design modulus has the representation
	\begin{equation}\label{eq:disclosure-spectrum}
		\omega_T(\eta)^2=\inf_A\inf_{\lambda\geq0} (\rho^2+\lambda\eta^2) \nabla\varphi(x_\circ)^\top(I+\lambda A^\top A)^{-1} \nabla\varphi(x_\circ).
	\end{equation}
	The lower bound applies to adaptive query policies. A fixed query block and an affine decoder achieve the upper bound arbitrarily closely, and this upper bound remains valid when $\eta<H_g\rho^2\sqrt T/2$.
\end{theorem}

The two bounds play different roles. A small upper bound shows that the specified attribute can be recovered accurately, by a fixed query block and an affine decoder; the lower bound limits the accuracy of every attack, adaptive ones included. The quadratic terms are the price of curvature in the responses and in the attribute, and when $\omega_T(\eta)>0$, $H_g\rho^2\sqrt T=o(\eta)$, and $H_\varphi\rho^2=o(\omega_T(\eta))$, both bounds reduce to $\omega_T(\eta)$ to leading order. The bounds are information-theoretic; finding the best query block need not be computationally easy.

The spectral representation \eqref{eq:disclosure-spectrum} explains why the two tasks behave differently. Reconstruction must resolve every record direction, whereas recovery of $\varphi$ depends only on the directions in which $\nabla\varphi(x_\circ)$ has a component. Directions that the responses observe poorly can therefore obstruct reconstruction while leaving the attribute recoverable, which is what we mean by selective disclosure. The same error bound identifies a thresholded feature whenever the attribute lies farther from the threshold than the recovery error; literal record membership, being discontinuous, falls outside this smooth-attribute setting. The analysis is relative to the family of queries the protocol accepts. Enlarging that family gives a corrupted server more options and cannot increase its optimal recovery error, and inconsistent states or manipulated participation, which expose response differences through repeated aggregates \citep{lam2021gradient,pasquini2022eluding}, must be modeled as additional effective queries in $\mathcal G$ that satisfy the same smoothness bounds.

\subsection{Personalized targeting with client retirement}
\label{sec:personalized}
A client can finish targeting before the federation does. If its exact local fit is zero, then, with the targeting state retained and the rule that selects zero whenever zero is a minimizer, every subsequent autonomous TMLE update of its own would be zero as well; yet accepting the other clients' next averaged update could move it away from this locally completed density. FedTMLE-LR (Algorithm~\ref{alg:local-retirement}) lets such a client freeze its density, targeting state, and plug-in estimate and leave, while the others continue. The retained estimate incorporates all shared updates accepted before retirement, and the retirement certificate is evaluated on the client's own observations.

This is a form of personalized federated learning, in which clients may return different models \citep{li2021ditto}; here the personalization is driven by completion of the local targeting problem, which determines when each client leaves. After retirement a client performs no further local solves and exchanges no further messages, and later shared updates neither alter its frozen output nor need to preserve its local estimating equation.

Fix the base weights
\begin{equation}\label{eq:retirement-weights}
	\alpha_m=1/M\quad\text{or}\quad\alpha_m=N_m/n.
\end{equation}
Let $\tau_m$ be the first iteration at which client $m$'s local fit is zero, with $\tau_m=\infty$ if this never happens in the uncapped sequence, and let $\mathcal A_k$ be the set of clients still active after the retirement check of round $k$. Client $m$'s current or retained density is $p_m^k:=p^{\min\{k,\tau_m\}}$. As clients leave, the server renormalizes the fixed base weights over the active set. Each weighting mode generates its own trajectory, and we suppress a mode index. When the budget $K$ is reached, the algorithm returns the unfinished outputs with an explicit flag; for the analysis after all clients have retired, the shared density is extended as a constant, with no further optimization or communication.

\begin{algorithm}[!htbp]
	\caption{FedTMLE-LR: local averaging with retirement}
	\label{alg:local-retirement}
	\small
	\begin{algorithmic}[1]
		\Require $\{O_i:i\in\mathcal I_m\}_{m=1}^M$, $p^0$, budget $K$; weights from \eqref{eq:retirement-weights}.
		\State $\mathcal A_{-1}\gets\{1,\ldots,M\}$; $\tau_m\gets\infty$ for every $m$.
		\For{$k=0,\ldots,K$}
		\State $\textsf{server}\xrightarrow{\ p^k,\,p^k(\cdot)\ }m, \quad m\in\mathcal A_{k-1}$.
		\For{$m\in\mathcal A_{k-1}$ \textbf{in parallel}}
		\State $\displaystyle\hat\epsilon_m^k\in \argmin_{\epsilon\in\mathcal E^k}\ell_{m,k}(\epsilon)$, selecting zero whenever it is a minimizer.
		\If{$\hat\epsilon_m^k=\mathbf0$}
		\State $\tau_m\gets k$; retain $p^k$, its targeting state, and $\Psi(p^k)$ locally.
		\EndIf
		\State $m\xrightarrow{\ \hat\epsilon_m^k\ }\textsf{server}$; retire after this upload if $\tau_m=k$.
		\EndFor
		\State $\mathcal A_k\gets \{m\in\mathcal A_{k-1}:\hat\epsilon_m^k\ne\mathbf0\}$.
		\If{$\mathcal A_k=\varnothing\ \lor\ k=K$}
		\State \Return $\{(p_m^k,\Psi(p_m^k),\tau_m\leq k)\}_{m=1}^M$.
		\EndIf
		\State $\displaystyle\epsilon^k\gets \frac{\sum_{m\in\mathcal A_k}\alpha_m\hat\epsilon_m^k} {\sum_{m\in\mathcal A_k}\alpha_m}$; \quad $p^{k+1}\gets p^k(\epsilon^k)$.
		\EndFor
	\end{algorithmic}
\end{algorithm}

\begin{definition}[Local TMLE stability]\label{def:local-stability}
	For a density $p$ and its retained fluctuation rule, define the local targeting gap by
	\begin{equation*}
		\Delta_m(p):=\ell_m(p)-\min_\epsilon\ell_m(p(\epsilon)),
	\end{equation*}
	the minimum being over admissible fluctuations. The density is locally TMLE-stable for client $m$ if $\Delta_m(p)=0$.
\end{definition}

The gap is the loss improvement available from one exact local TMLE update. When it is zero, the selection rule returns zero, autonomous local TMLE stays at the same density, and \eqref{eq:local-score} gives a zero local empirical EIF mean; the converse fails in a nonconvex problem, where a zero gradient need not minimize the loss. A zero gap is therefore a stronger certificate than the first-order equation alone: it certifies that the local minimization is complete. Every retired client's gap is zero at its retained density.

The analysis of averaging starts by comparing the local empirical gradients and Hessians across clients. Under the common-law assumption they differ only through sampling, and larger client samples make them closer. For $\xi\in(0,1)$, write $\zeta_n(\xi):=\sqrt{8\log(4M^2\mathcal N_n/\xi)}$.

\begin{lemma}[Uniform cross-silo derivative comparison]
	\label{lem:local-derivatives}
	Under Assumptions~\ref{ass:local-regularity}--\ref{ass:local-complexity}, with probability at least $1-\xi$, for every $m\ne j$ and all states and parameters covered by Assumption~\ref{ass:local-complexity},
	\begin{equation}\label{eq:local-derivative-comparison}
		\begin{aligned}
			\|\nabla_\epsilon\{\ell_m(p(\epsilon))-\ell_j(p(\epsilon))\}\| &\leq G_1\zeta_n(\xi)\sqrt{N_m^{-1}+N_j^{-1}},\\
			\|\nabla_\epsilon^2\{\ell_m(p(\epsilon))-\ell_j(p(\epsilon))\}\|_{\mathrm{op}} &\leq\beta\zeta_n(\xi)\sqrt{N_m^{-1}+N_j^{-1}}.
		\end{aligned}
	\end{equation}
	The same event covers both weighting modes and their data-dependent iterates and retirement times.
\end{lemma}

The event is uniform over the state family fixed in Assumption~\ref{ass:local-complexity}, which is why the bounds can be applied at the data-dependent densities and retirement times that the algorithm produces. Nearby derivatives do not make the local losses or their minimizers coincide, however, and the remaining question is what averaging distinct minimizers costs.

That cost is expressed through the pairwise disagreement moments
\begin{equation*}
	V_r^k:= \frac{\sum_{m,j\in\mathcal A_k}\alpha_m\alpha_j \|\hat\epsilon_m^k-\hat\epsilon_j^k\|^r} {2\sum_{m\in\mathcal A_k}\alpha_m}, \qquad r\in\{2,4\},
\end{equation*}
where $r$ is the moment order, not the round index of Section~\ref{sec:theory-gradient}. The second moment equals the base-weighted sum of squared distances between the local fits and the server's update; the fourth is more sensitive to widely separated fits. The sampling coefficient is
\begin{equation}\label{eq:retirement-sampling-modes}
	\chi_k:=
	\begin{cases}
		\displaystyle\frac{|\mathcal A_k|-1}{M|\mathcal A_k|} \sum_{m\in\mathcal A_k}N_m^{-1}, &\alpha_m=1/M,\\[3mm]
		\displaystyle\frac{|\mathcal A_k|-1}{n}, &\alpha_m=N_m/n.
	\end{cases}
\end{equation}
Both quantities are set to zero when the active set is empty, and both vanish when a single client remains. The server can compute the disagreement moments from the uploaded fits and the sampling coefficient from the participating sample sizes.

\begin{lemma}[Fourth-order averaging error]\label{lem:local-fourth-order}
	Under Assumption~\ref{ass:local-regularity}, on the event of Lemma~\ref{lem:local-derivatives}, every iteration satisfies
	\begin{equation}\label{eq:local-fourth-order-error}
		\begin{aligned}
			0\leq{}&\sum_{m\in\mathcal A_k}\alpha_m \{\ell_{m,k}(\epsilon^k)-\ell_{m,k}(\hat\epsilon_m^k)\}\\
			\leq{}&\frac{2G_1\zeta_n(\xi)}3\sqrt{\chi_kV_2^k} +\frac{\beta\zeta_n(\xi)}6\sqrt{\chi_kV_4^k} +\frac{\beta_4}{18}V_4^k.
		\end{aligned}
	\end{equation}
	The same left-hand side is at most $\beta V_2^k/2$, deterministically.
\end{lemma}

Each term on the left is the increase in a client's loss when its own minimizer is replaced by the server's average, and their weighted sum is the improvement forfeited through averaging. Smoothness alone gives the quadratic bound. The fourth-order bound uses in addition the stationarity and nonnegative curvature of each local loss at its own selected minimum, together with the derivative comparisons of Lemma~\ref{lem:local-derivatives}: the comparisons produce the sampling terms, and the Lipschitz third derivative produces the fourth-order remainder. The penalty is small when the local fits are close and their derivatives similar. No sign of curvature is required between the minima, and whichever of the two bounds is smaller may be used.

\begin{theorem}[Personalized targeting bound with retirement]
	\label{thm:personalized-convergence}
	Under Assumptions~\ref{ass:smooth}, \ref{ass:local-regularity}, and \ref{ass:local-complexity}, with probability at least $1-\xi$, simultaneously for every $K\geq1$ and either weighting mode,
	\begin{equation}\label{eq:personalized-convergence-bound}
		\begin{aligned}
			&\frac1{2\beta K}\sum_{k=0}^{K-1}\sum_{m=1}^M\alpha_m \left\|\frac1{N_m}\sum_{i\in\mathcal I_m} D(p_m^k)(O_i)\right\|^2 \leq\frac1K\sum_{k=0}^{K-1}\sum_{m=1}^M \alpha_m\Delta_m(p_m^k)\\
			&\quad\leq\frac1K\sum_{m=1}^M\alpha_m \{\ell_m(p^0)-\underline\ell_m\}\\
			&\qquad\quad+\frac1K\sum_{k=0}^{K-1}\left\{ \frac{2G_1\zeta_n(\xi)}3\sqrt{\chi_kV_2^k} +\frac{\beta\zeta_n(\xi)}6\sqrt{\chi_kV_4^k} +\frac{\beta_4}{18}V_4^k\right\}.
		\end{aligned}
	\end{equation}
	For every retired client, both its local gap and its local empirical EIF mean are zero at $p_m^k$ for all $k\geq\tau_m$.
\end{theorem}

The bound has the form of a nonconvex convergence guarantee: the average remaining local improvement is at most the initial loss budget, which contributes $O(K^{-1})$, plus the accumulated averaging error divided by the number of iterations. Because each client's loss is telescoped only over its own participation interval and stays constant after retirement, changes in the active set introduce no additional term. And because the EIF means are evaluated at each client's own density before being squared and summed, opposing residuals cannot cancel in the criterion, in contrast to the pooled residual of Section~\ref{sec:local-averaging}.

The two weighting modes trade influence against sensitivity. Uniform weights give every client the same influence in the criterion, but the inverse sample sizes in $\chi_k$ make the penalty sensitive to small clients; sample-size weights reduce the contribution of small clients to $\chi_k$ and, at the same time, their weight in the criterion. For the same active set, the sample-size coefficient, normalized by the remaining base weights, is never larger than the uniform one; the ratio of the two is the arithmetic-to-harmonic mean ratio of the active sample sizes. The comparison holds the active set fixed, and the two modes generally follow different trajectories. Section~\ref{app:personalized-convergence} derives the comparison and the participation-duration identities for $\sum_k\chi_k$.

\begin{figure}[!bp]
	\centering \includegraphics[width=\textwidth]{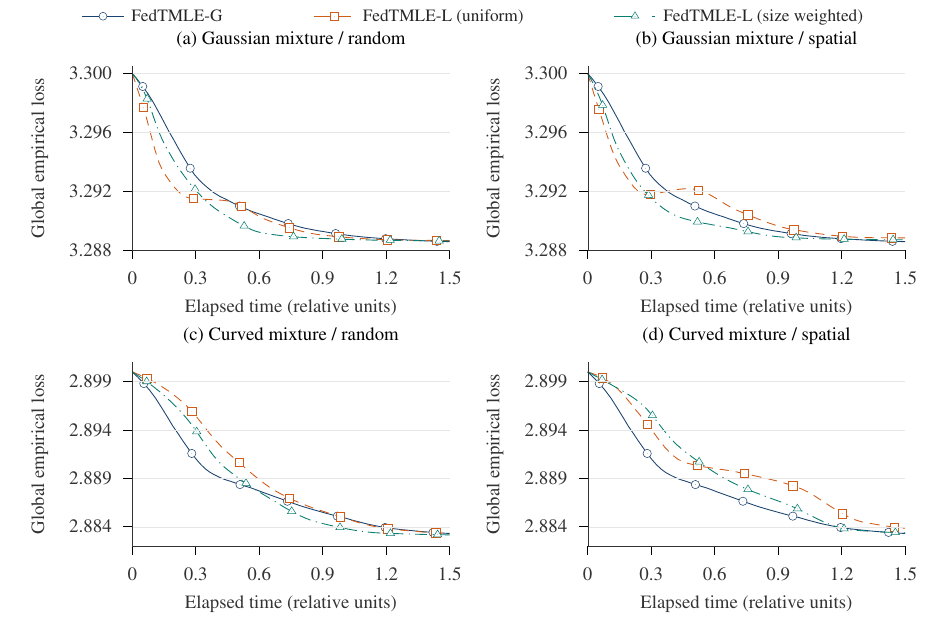} \caption{Global empirical loss against elapsed targeting time with $M=10$. Rows distinguish the Gaussian and curved mixtures; columns distinguish random and spatial partitions. Time uses the same reference duration in all panels and for all methods.}
	\label{fig:experiment-loss}
\end{figure}

Two consequences follow. If the accumulated averaging bound is $o(K)$, the average personalized gap vanishes; if the per-round bounds are summable, every client's gap tends to zero. Both are statements about the realized disagreement, on which the theorem imposes no decay law. Once a single client remains, the averaging penalty is exactly zero and the method reduces to ordinary local TMLE, whereas with several active clients cancellation among nonzero fits can still prevent retirement. The theorem thus certifies the outputs of the clients that leave and bounds the gaps of those that stay. Whether every client retires in finitely many rounds, and whether the retained estimates are doubly robust without further conditions, are questions we leave to future work.

\section{Experiments}\label{sec:experiments}

We compare FedTMLE-G and the two weighting rules of FedTMLE-L through a three-dimensional target defined by entropy-regularized optimal transport \citep{cuturi2013sinkhorn,peyre2019computational,genevay2016stochastic}. An observation $o\in\mathbb R^2$ is softly allocated to four fixed destinations $z_1,\ldots,z_4$, each with capacity $1/4$. With unit regularization and $\theta_4=0$, define the relative allocation offsets by
\begin{equation}\label{eq:experiment-target}
\Psi(p)=\argmin_{\theta\in\mathbb R^3} \left\{\int\log\!\left[\sum_{j=1}^4 e^{\theta_j-\|o-z_j\|_2^2/2}\right]p(o)\,do -\frac14\sum_{j=1}^4\theta_j\right\}.
\end{equation}
The capacity constraints couple the coordinates, and the EIF adjusts allocation imbalance by the curvature of this optimization problem. The four destinations remain fixed as the number of clients changes.

The comparison uses $n=10{,}000$ observations from an anisotropic Gaussian mixture or a curved mixture with location-dependent spread. Each dataset has random and moderately spatially heterogeneous partitions with the same unequal client sizes, whose largest-to-smallest ratio is approximately $5:1$. Initialization, fluctuation coordinates, and numerical accuracy criteria are shared. The two metrics are the global empirical loss $P_nL(p^k)$, where $L(p)(o)=-\log p(o)$, and the pooled EIF residual $r^k=\|n^{-1}\sum_iD(p^k)(O_i)\|_2$. Evaluation always gives each observation weight $1/n$. Section~\ref{app:experiments} specifies the DGPs, partitions, and timing conventions.

\begin{figure}[!htbp]
	\centering \includegraphics[width=\textwidth]{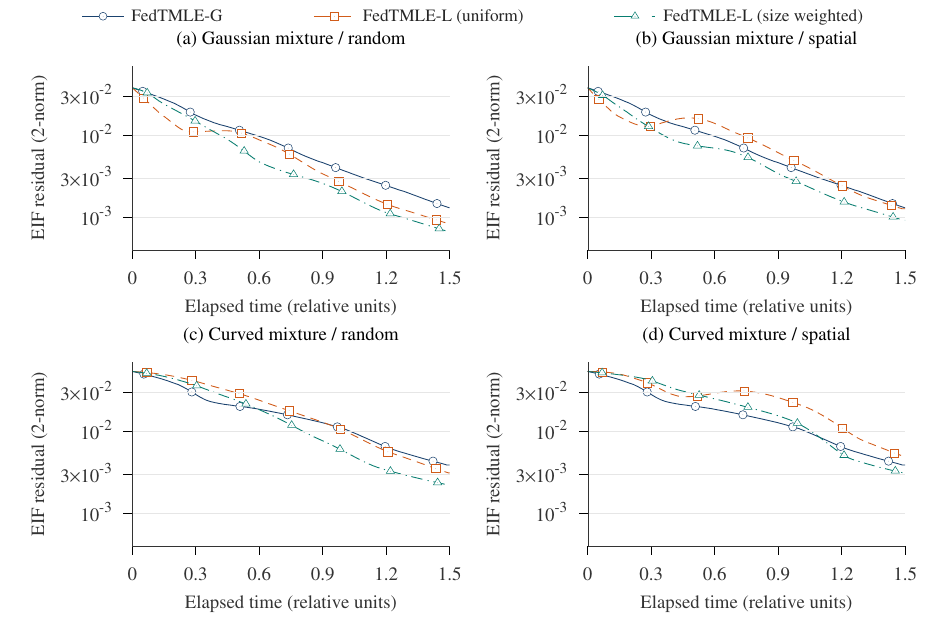} \caption{Pooled EIF residual $r^k$ against elapsed targeting time with $M=10$, on a logarithmic vertical scale. Checkpoints and the time normalization are the same as in Figure~\ref{fig:experiment-loss}.}
	\label{fig:experiment-eif}
\end{figure}
\begin{figure}[!htbp]
	\centering \includegraphics[width=\textwidth]{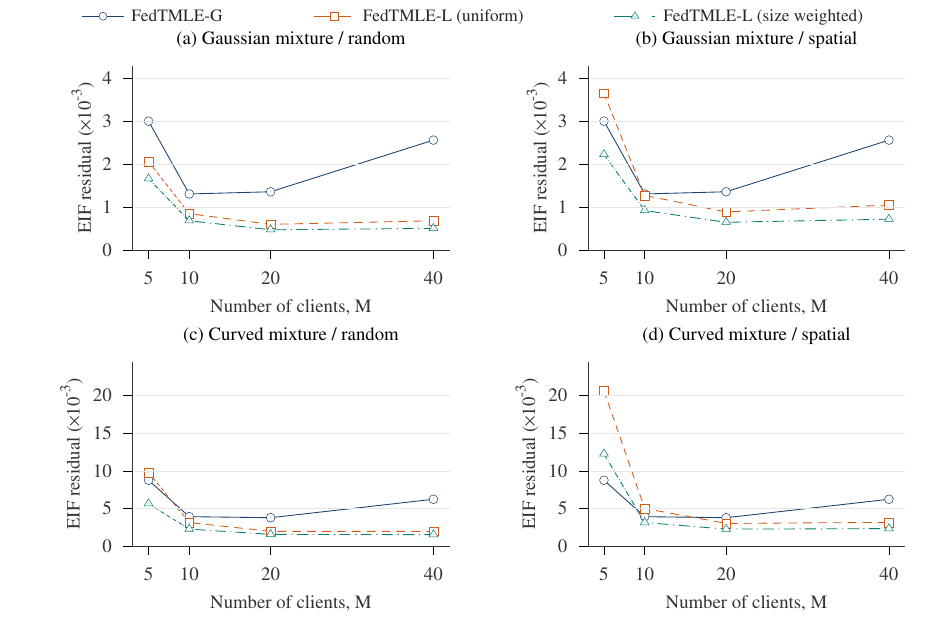} \caption{Pooled EIF residual at a fixed elapsed-time budget of $1.5$ reference units, as the number of clients varies. Total sample size is fixed at $n=10{,}000$; each point uses the last completed shared density within the budget.}
	\label{fig:experiment-clients}
\end{figure}

\begin{figure}[!htbp]
	\centering \includegraphics[width=\textwidth]{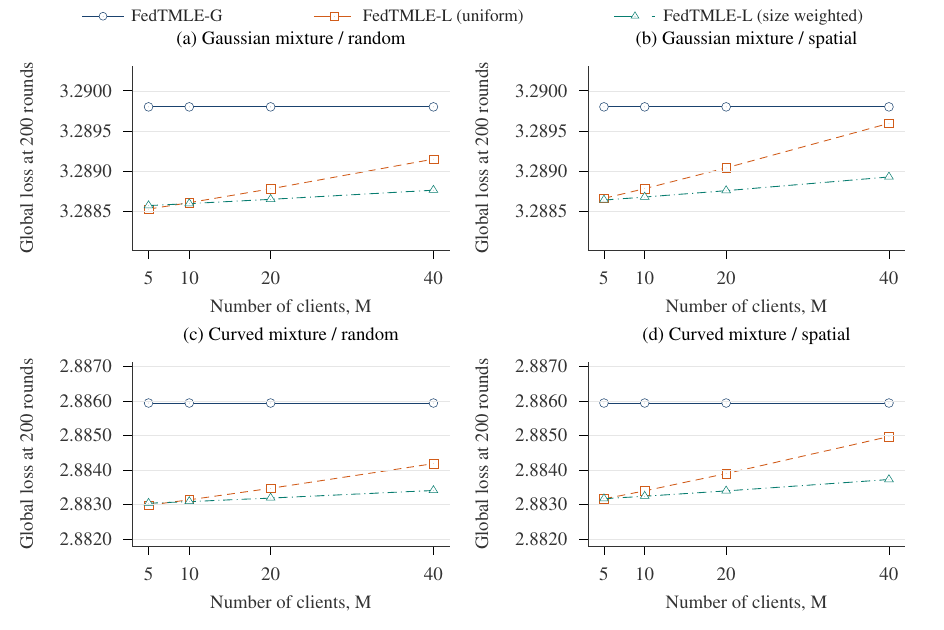} \caption{Global empirical loss after 200 communication rounds, with $n=10{,}000$ fixed. Repeated gradient queries, additional pooled step-selection queries, and required EIF checks all count toward the budget.}
	\label{fig:experiment-budget}
\end{figure}

Figures~\ref{fig:experiment-loss}-\ref{fig:experiment-eif} compare the trajectories at $M=10$ against elapsed targeting time, using one common reference duration for all panels and methods. Comparable final losses are reached through different transient behavior. Local averaging makes faster initial progress on the Gaussian mixture, whereas FedTMLE-G leads early on the curved mixture before the local methods catch up. Spatial partitioning introduces a clearer middle-stage slowdown for local averaging, including a small pooled-loss rebound under uniform weights. The EIF curves retain differences in targeting progress even when the losses are close.

Timing balances repeated gradient synchronization against complete local solves, whose cost depends on conditioning and the slowest client. Unequal client sizes distinguish the averaging rules even under random partitioning. Spatial heterogeneity changes both local gradients and curvature, so sample-size weighting does not recover the pooled minimizer in general. These effects allow temporary reversals in the methods' ordering.

Figures~\ref{fig:experiment-clients}-\ref{fig:experiment-budget} vary $M\in\{5,10,20,40\}$ at fixed total sample size. The fixed-time residuals reflect a balance between parallel local work, synchronization, and smaller local samples. At 200 rounds, uniform weighting becomes less favorable as more clients divide the sample, particularly with spatial heterogeneity. FedTMLE-G retains higher loss under this communication budget: each update requires several pooled gradient and step-selection queries, while FedTMLE-L communicates after a complete local solve. FedTMLE-G's full-batch trajectory remains invariant to repartitioning at a matched round budget. This budget-specific difference does not imply uniformly slower targeting in time, which also includes local computation and shared target and density work.

\section{Conclusion}\label{sec:conclusion} We have developed a cross-silo formulation of iterative TMLE that keeps observations local and preserves the plug-in form of the estimator. FedTMLE-G distributes gradient evaluation while retaining a matched centralized trajectory, whereas FedTMLE-L exchanges locally fitted fluctuation parameters, avoiding repeated synchronization within each inner solve but changing the targeting path. Our analysis makes the consequences of this choice explicit. For FedTMLE-G, differential coding exploits gradient variation to control the bit cost of an EIF certificate without changing the rounded computation, while a description of the accepted targeting history controls statistical adaptivity separately from the full gradient transcript. For local averaging, FedTMLE-LR allows clients to retain their own locally TMLE-stable estimates upon retirement. Its targeting bound quantifies the improvement lost through averaging in terms of client weights, sample sizes, and disagreement between local fits. Together, these results distinguish faithful pooled computation from personalized completion, and numerical targeting accuracy from the additional statistical conditions needed for inference.

Our disclosure analysis further shows why keeping data local is not itself a privacy guarantee: instability of record reconstruction can coexist with accurate recovery of a specified attribute. Future work should extend the personalized analysis to genuinely different silo populations and account for inexact local solves and tolerance-based retirement. Adapting local computation and synchronization to observed disagreement offers another direction for balancing communication savings against targeting progress. Finally, integrating query validation, secure aggregation, and calibrated privacy noise would allow communication cost, disclosure protection, and inferential accuracy to be studied jointly.

\bibliographystyle{imsart-nameyear} \bibliography{ref}

\newpage \providecommand{\theHaddressref}{} \renewcommand{\theHaddressref}{supp.address.\arabic{addressref}} \setcounter{addressref}{0}
\csundef{author@1@ead@marks} \csundef{author@2@ead@marks} \csundef{author@3@ead@marks} \thankslabel[a]{se1} \thankslabel[b]{se2} \thankslabel[c]{se3}
\begin{frontmatter}
\title{Supplement to ``Federated Targeted Maximum Likelihood Estimation''} \runtitle{Supplement: Federated TMLE} \runauthor{D. Li, F. Wang and K. Gan}

\begin{aug}
\author[SA]{\fnms{Diyang}~\snm{Li}\ead[label=se1]{diyang01@cs.cornell.edu}}, \author[SA,SB]{\fnms{Fei}~\snm{Wang}\ead[label=se2]{few2001@med.cornell.edu}} \and \author[SA]{\fnms{Kyra}~\snm{Gan}\ead[label=se3]{kyragan@cornell.edu}}

\address[SA]{Cornell University\printead[presep={,\ }]{se1,se3}} \address[SB]{Weill Cornell Medicine\printead[presep={,\ }]{se2}}
\end{aug}
\end{frontmatter}

\setcounter{section}{0} \setcounter{equation}{0} \setcounter{figure}{0} \renewcommand{\thesection}{S.\arabic{section}} \renewcommand{\theequation}{S.\arabic{equation}} \renewcommand{\thefigure}{S.\arabic{figure}}

\renewcommand{\theHsection}{supp.\arabic{section}} \renewcommand{\theHequation}{supp.\arabic{equation}} \renewcommand{\theHfigure}{supp.\arabic{figure}}

\section{Related work}\label{app:related-work}

To our knowledge, this paper introduces the first general-purpose federated algorithms for standard iterative TMLE. The contribution is to federate targeting itself: participating silos repeatedly fit a shared, evolving fluctuation model without pooling observations. FedTMLE-G preserves a matched centralized trajectory through gradient aggregation, whereas FedTMLE-L exchanges local fluctuation solutions. This algorithmic extension of TMLE \citep{vanderlaan2006targeted,li2025optimization} applies across target functionals rather than prescribing a particular causal estimand.

Adjacent federated causal methods address different computational tasks. \citet{han2023multiply} combine multiply robust nuisance estimation with adaptive weighting of site-specific treatment-effect estimates under covariate shift and covariate mismatch. Here, ``targeted'' identifies the population of interest, not an iterative TMLE update. \citet{liu2026privacy} construct influence-function-based survival estimators and select source weights under distribution shift; TMLE is discussed as an extension. \citet{tong2025disc2ohd} approximate pooled nuisance fits through surrogate likelihoods for doubly robust estimation. These methods address transportability or nuisance estimation, rather than the successive fluctuation problems defining standard TMLE. The broader intersection of federated and targeted learning is also addressed in the dissertation of \citet{phillips2023targeted}.

\citet{beclin2026privacy} develop a specialized federated TMLE for survival mediation using backward-recursive hazard targeting followed by regression targeting. Their shared-nuisance formulation also distributes fluctuation fitting. The construction exploits the mediation functional's sequential structure to avoid repeated global targeting iterations; it does not provide a general federated algorithm for standard iterative TMLE. Our algorithms instead repeatedly rebuild and optimize the targeting submodel without relying on that decomposition. The distinction is algorithmic, not simply a difference in application domain.

Federated averaging \citep{mcmahan2017communication} provides a communication pattern, not a TMLE targeting rule. \citet{yang2023federated} develop multi-gradient methods for competing objectives and Pareto stationarity, rather than influence-function calibration of a plug-in estimator. In FedTMLE, local fits concern a changing submodel, and averaging need not preserve centralized targeting. FedTMLE-LR further allows locally completed clients to retain their own estimates. The estimand-specific causal estimators and multi-objective methods above are therefore neighboring constructions, not interchangeable baselines for the general iterative targeting problem addressed here.

\section{Proofs of all results}\label{app:proofs} The communication proof holds the observed sample fixed. Probabilities in the statistical proof refer to the original i.i.d. targeting sample, with conditioning on an independent initialization only where explicitly stated. Throughout, the score identity is used only at a submodel's origin. Sections~\ref{app:reconstruction}--\ref{app:disclosure} analyze deterministic uncertainty in the transcript with public information held fixed.
\begingroup
\allowdisplaybreaks[4]

\subsection{Proof of Theorem~\ref{thm:communication}}\label{app:communication}
\begin{proof}

Let $(k_r,t_r)$ denote the state queried in round $r$, and let $\mathcal S$ denote the rounds with action $00$. At this stage, neither finiteness of $\mathcal S$ nor termination is assumed. To decode signed differences, map $a\in\mathbb Z$ to a positive integer using the bijection
\begin{equation*}
\iota(a)=
\begin{cases}
2a+1,&a\geq0,\\
-2a,&a<0.
\end{cases}
\qquad \iota^{-1}(v)=
\begin{cases}
(v-1)/2,&v\text{ odd},\\
-v/2,&v\text{ even}.
\end{cases}
\end{equation*}
For a positive integer $v$, its gamma code \citep{elias1975universal} consists of $\lfloor\log_2v\rfloor$ zero bits followed by its binary representation. The initial zeros determine how many subsequent binary digits to read; therefore the code is self-delimiting and prefix-free. Applying it to $\iota(a)$ defines $\operatorname{enc}(a)$, with length
\begin{align}
|\operatorname{enc}(a)| &=2\lfloor\log_2\iota(a)\rfloor+1\notag\\*
&\leq1+2\log_2(2|a|+1)\notag\\
&\leq3+2\log_2(1+|a|),\notag\\*
|\operatorname{enc}(v)| &=\sum_{j=1}^d|\operatorname{enc}(v_j)| \leq3d+2\sum_{j=1}^d\log_2(1+|v_j|), \qquad v\in\mathbb Z^d.
\label{eq:proof-code-length}
\end{align}
Knowing the dimension $d$, the decoder reads exactly that many codewords per vector message. After $j$ initial zeros, the decoder reads $j+1$ bits that start with one and specify an integer in $[2^j,2^{j+1}-1]$. Applying $\iota^{-1}$ recovers the signed value, and the next coordinate starts immediately after these $2j+1$ bits. This parsing rule requires neither an integer-magnitude bound nor a separate message-length field and applies to the initial absolute value as well as all subsequent differences.

All integer memories are initialized to zero. Assume inductively that the local memories, server accumulator, and reconstructed density agree with the absolute-message implementation before round $r$. Decoding and adding the local differences gives
\begin{align*}
s_{r-1}+\sum_{m=1}^M N_m(z_{m,r}-z_{m,r-1}) &=\sum_{m=1}^M N_mz_{m,r-1} +\sum_{m=1}^M N_m(z_{m,r}-z_{m,r-1})\\*
&=\sum_{m=1}^M N_mz_{m,r}=s_r,\\
q_{r-1}+(q_r-q_{r-1})&=q_r,\\*
h\{q_{r-1}+(q_r-q_{r-1})\}&=h q_r=\widetilde g_r.
\end{align*}
Both implementations choose the same action because they use the same decoded norm and current inner index. Under action $00$, they reconstruct the next parameter as
\begin{align*}
\epsilon^{k_r,t_r+1} &=\epsilon^{k_r,t_r} -\frac h\beta\{q_{r-1}+(q_r-q_{r-1})\}\\*
&=\epsilon^{k_r,t_r}-\frac h\beta q_r =\epsilon^{k_r,t_r}-\frac1\beta\widetilde g_r.
\end{align*}
Under action $01$, both computations accept the current parameter and apply the shared map $p^{k_r+1}=p^{k_r}(\epsilon^{k_r})$. Under action $10$, they return the same density and certificate. The norm tests therefore agree without a separate transmission of an approximation to $\epsilon^{k_r,t_r}$. The shared reconstruction rule completes the induction through both inner steps and outer updates. Integer memories can be retained through re-centering because they store the previously encoded gradients; they do not specify a parameter in the new submodel.

For the fixed-width comparison, put $J=\lceil G/h\rceil$ when the stated global gradient bound is available. Every local integer then belongs to $\{-J,\ldots,J\}$. Their weighted mean belongs to $[-J,J]$, and rounding it to a nearest integer remains in that interval. Thus the same $b$ bits encode both local and server coordinates. An offset by $J$ maps each coordinate into $\{0,\ldots,2J\}$, whose cardinality is $2J+1\leq2^b$. Counting separately on every directed link,
\begin{align*}
\mathsf B_{\mathrm{fixed}} &=\sum_{r=1}^R\left\{ \sum_{m=1}^M\sum_{j=1}^d b +\sum_{m=1}^M\left(\sum_{j=1}^d b+2\right)\right\}\\*
&=\sum_{r=1}^R\{Mdb+M(db+2)\} =2MdbR+2MR.
\end{align*}
Only the fixed-width comparison uses the global magnitude bound. The differential decoder remains well defined without it. Exact decoding recovers the rounded gradients, whose error relative to the unrounded gradients remains to be bounded. At the state queried in round $r$, write $g_r=\sum_m(N_m/n)g_{m,r}=\nabla\ell_{k_r}(\epsilon^{k_r,t_r})$. The two nearest-integer rounding operations imply
\begin{align}
\|h z_{m,r}-g_{m,r}\|_\infty&\leq h/2,\notag\\*
\|h q_r-hs_r/n\|_\infty&\leq h/2,\notag\\
\|\widetilde g_r-g_r\|_\infty &=\left\|h(q_r-s_r/n) +\sum_{m=1}^M\frac{N_m}{n}(hz_{m,r}-g_{m,r})\right\|_\infty\notag\\
&\leq\|h(q_r-s_r/n)\|_\infty +\sum_{m=1}^M\frac{N_m}{n} \|hz_{m,r}-g_{m,r}\|_\infty\notag\\
&\leq\frac h2+\frac h2\sum_{m=1}^M\frac{N_m}{n}=h,\notag\\*
\|\widetilde g_r-g_r\| &\leq\sqrt d\,\|\widetilde g_r-g_r\|_\infty \leq h\sqrt d\leq\delta/4.
\label{eq:proof-rounding-error}
\end{align}
The error in \eqref{eq:rounding-guarantee} is thus bounded at each round and does not accumulate with the number of transmitted differences. Assumption~\ref{ass:smooth} also passes to the pooled objective under sample-size weighting,
\begin{align}
\|\nabla\ell_k(u)-\nabla\ell_k(v)\| &=\left\|\sum_{m=1}^M\frac{N_m}{n} \{\nabla\ell_{m,k}(u)-\nabla\ell_{m,k}(v)\}\right\|\notag\\
&\leq\sum_{m=1}^M\frac{N_m}{n} \|\nabla\ell_{m,k}(u)-\nabla\ell_{m,k}(v)\|\notag\\*
&\leq\sum_{m=1}^M\frac{N_m}{n}\beta\|u-v\| =\beta\|u-v\|.
\label{eq:proof-pooled-smoothness}
\end{align}
Both calculations allow heterogeneous local objectives. Fix a round $r\in\mathcal S$ and temporarily write $\epsilon=\epsilon^{k_r,t_r}$. The proposed step segment is feasible, so the fundamental theorem of calculus and \eqref{eq:proof-pooled-smoothness} yield
\begin{align*}
\ell_{k_r}(\epsilon-\widetilde g_r/\beta)-\ell_{k_r}(\epsilon) &=-\frac1\beta\int_0^1 \left\langle\nabla\ell_{k_r} (\epsilon-u\widetilde g_r/\beta),\widetilde g_r\right\rangle du\\
&=-\frac1\beta\langle g_r,\widetilde g_r\rangle -\frac1\beta\int_0^1 \left\langle\nabla\ell_{k_r}(\epsilon-u\widetilde g_r/\beta) -g_r,\widetilde g_r\right\rangle du\\
&\leq-\frac1\beta\langle g_r,\widetilde g_r\rangle +\frac1\beta\int_0^1 \beta\frac{u\|\widetilde g_r\|}{\beta} \|\widetilde g_r\|\,du\\
&=-\frac{\|\widetilde g_r\|^2}{2\beta} +\frac1\beta\langle\widetilde g_r-g_r,\widetilde g_r\rangle\\
&\leq-\frac{\|\widetilde g_r\|^2}{2\beta} +\frac{h\sqrt d}{\beta}\|\widetilde g_r\|\\*
&=-\frac{\|\widetilde g_r\|^2}{2\beta} \left(1-\frac{2h\sqrt d}{\|\widetilde g_r\|}\right) \leq-\frac{\|\widetilde g_r\|^2}{6\beta} <-\frac{3\delta^2}{32\beta}.
\end{align*}
The last two inequalities use the action threshold $\|\widetilde g_r\|>3\delta/4$ and the rounding bound $h\sqrt d\leq\delta/4$. Together, they ensure that rounding preserves a fixed fraction of the descent associated with the communicated gradient. The argument requires no stationarity condition at an inner minimizer. 
The decreases can be summed over the entire execution because accepting an inner fit preserves the current empirical-loss value when the submodel is re-centered,
\begin{equation*}
\ell_{k+1}(\mathbf0) =P_nL(p^{k+1}) =P_nL(p^k(\epsilon^k)) =\ell_k(\epsilon^k).
\end{equation*}
Consider any finite prefix ending after round $u$ at state $(k,t)$. Telescoping through the completed outer fits and the current, possibly unfinished fit gives
\begin{align*}
\frac1{6\beta}\sum_{\substack{r\in\mathcal S\\r\leq u}} \|\widetilde g_r\|^2 &\leq\sum_{\substack{r\in\mathcal S\\r\leq u}} \{\ell_{k_r}(\epsilon^{k_r,t_r}) -\ell_{k_r}(\epsilon^{k_r,t_r+1})\}\\*
&=\sum_{j=0}^{k-1}\{\ell_j(\mathbf0)-\ell_j(\epsilon^j)\} +\ell_k(\mathbf0)-\ell_k(\epsilon^{k,t})\\
&=\sum_{j=0}^{k-1}\{\ell_j(\mathbf0)-\ell_{j+1}(\mathbf0)\} +\ell_k(\mathbf0)-\ell_k(\epsilon^{k,t})\\
&=\ell_0(\mathbf0)-\ell_k(\epsilon^{k,t})\\*
&\leq\ell_0(\mathbf0)-\inf_{p\in\mathcal M}P_nL(p)=\Delta.
\end{align*}
Empty sums are zero. If the prefix ends immediately after acceptance, the current inner parameter is zero and the equality still holds. A prefix may also end inside an inner solve; no termination assumption was used. Every executed step has squared communicated norm greater than $9\delta^2/16$, so the finite-prefix bound implies
\begin{align*}
\frac{3\delta^2}{32\beta} \#\{r\leq u:r\in\mathcal S\} &\leq\frac1{6\beta} \sum_{\substack{r\in\mathcal S\\r\leq u}} \|\widetilde g_r\|^2\\*
&\leq\Delta,\qquad u\geq1,\\
\sup_{u\geq1}\#\{r\leq u:r\in\mathcal S\} &\leq\frac{32\beta\Delta}{3\delta^2},\\*
\sup_{u\geq1}\sum_{\substack{r\in\mathcal S\\r\leq u}} \|\widetilde g_r\|^2 &\leq6\beta\Delta.
\end{align*}
Taking the supremum of the nonnegative partial sums gives
\begin{equation}\label{eq:proof-energy-and-steps}
\sum_{r\in\mathcal S}\|\widetilde g_r\|^2\leq6\beta\Delta, \qquad S:=|\mathcal S|\leq\frac{32\beta\Delta}{3\delta^2}<\infty.
\end{equation}
An action $01$ can occur only after at least one step since the last reset, and each step belongs to at most one accepted fit. There are therefore at most $S$ actions $01$. Since action $10$ terminates immediately, the protocol cannot execute infinitely many rounds. If $T_k$ denotes the number of steps in the $k$th accepted fit, then
\begin{align}
S&=\sum_{k=0}^{K-1}T_k, &T_k&\geq1, &K&\leq S,\notag\\*
R&=\sum_{k=0}^{K-1}(T_k+1)+1 =S+K+1 \leq2S+1 \leq1+\frac{64\beta\Delta}{3\delta^2}.
\label{eq:proof-rounds}
\end{align}
Termination occurs at the origin of submodel $K$, where the score identity applies. An endpoint query in submodel $K-1$ would not by itself justify that identity. Hence
\begin{align*}
r^K &=\|P_nD(p^K)\| =\|\nabla\ell_K(\mathbf0)\|\\*
&\leq\|\widetilde g_R\| +\|\nabla\ell_K(\mathbf0)-\widetilde g_R\|\\*
&\leq\|\widetilde g_R\|+h\sqrt d \leq\frac{3\delta}{4}+\frac{\delta}{4} =\delta.
\end{align*}
When $\Delta=0$, \eqref{eq:proof-energy-and-steps} forces $S=0$; then $K=0$ and $R=1$, so this case is covered as well.

To bound code lengths, we use the squared-gradient bound to control the total length of the fluctuation steps. Cauchy--Schwarz and \eqref{eq:proof-energy-and-steps} give
\begin{align}
\sum_{k=0}^{K-1}\|\epsilon^k\| &=\sum_{k=0}^{K-1} \left\|\sum_{t=0}^{T_k-1} (\epsilon^{k,t+1}-\epsilon^{k,t})\right\|\notag\\*
&\leq\sum_{k=0}^{K-1}\sum_{t=0}^{T_k-1} \|\epsilon^{k,t+1}-\epsilon^{k,t}\|\notag\\
&=\frac1\beta\sum_{r\in\mathcal S}\|\widetilde g_r\|\notag\\
&\leq\frac{\sqrt S}{\beta} \left(\sum_{r\in\mathcal S}\|\widetilde g_r\|^2\right)^{1/2}\notag\\*
&\leq\sqrt{\frac{6\Delta S}{\beta}} \leq\sqrt{\frac{6\Delta}{\beta} \frac{32\beta\Delta}{3\delta^2}} =\frac{8\Delta}{\delta}.
\label{eq:proof-path-length}
\end{align}
The same upper bound applies to the sum of the norms of all individual parameter increments. Successive local gradients can change either within a submodel or when the submodel is re-centered. If action $00$ occurs in round $r$, the next query uses the same submodel, and Assumption~\ref{ass:smooth} gives
\begin{align}
\|g_{m,r+1}-g_{m,r}\| &=\|\nabla\ell_{m,k_r}(\epsilon^{k_r,t_r+1}) -\nabla\ell_{m,k_r}(\epsilon^{k_r,t_r})\|\notag\\*
&\leq\beta\|\epsilon^{k_r,t_r+1}-\epsilon^{k_r,t_r}\| =\|\widetilde g_r\|.
\label{eq:proof-within-variation}
\end{align}
For action $01$, let $k=k_r$. The protocol accepts $\epsilon^k$ and evaluates the next gradient at the origin of the reconstructed submodel. Insert the gradient at the old origin to separate the change in the EIF from the change within the old submodel. The score identities at the two origins and Assumptions~\ref{ass:recentering} and \ref{ass:smooth} then give
\begin{align}
\|g_{m,r+1}-g_{m,r}\| &=\|\nabla\ell_{m,k+1}(\mathbf0) -\nabla\ell_{m,k}(\epsilon^k)\|\notag\\
&=\left\|\frac1{N_m}\sum_{i\in\mathcal I_m} D(p^{k+1})(O_i)-\nabla\ell_{m,k}(\epsilon^k)\right\|\notag\\
&=\left\|\frac1{N_m}\sum_{i\in\mathcal I_m} \{D(p^k(\epsilon^k))-D(p^k)\}(O_i) +\nabla\ell_{m,k}(\mathbf0) -\nabla\ell_{m,k}(\epsilon^k)\right\|\notag\\
&\leq \left\|\frac1{N_m}\sum_{i\in\mathcal I_m} \{D(p^k(\epsilon^k))-D(p^k(\mathbf0))\}(O_i)\right\| +\|\nabla\ell_{m,k}(\mathbf0) -\nabla\ell_{m,k}(\epsilon^k)\|\notag\\*
&\leq\kappa\|\epsilon^k\|+\beta\|\epsilon^k\| =(\kappa+\beta)\|\epsilon^k\|.
\label{eq:proof-recenter-variation}
\end{align}
With $T_k$ as in \eqref{eq:proof-rounds}, the queries within an accepted fit occur at $\epsilon^{k,0},\ldots,\epsilon^{k,T_k}$ and are followed by the origin query in submodel $k+1$. These two types of transition therefore account for the complete local variation,
\begin{align*}
\sum_{r=2}^R\|g_{m,r}-g_{m,r-1}\| &=\sum_{k=0}^{K-1}\sum_{t=0}^{T_k-1} \|\nabla\ell_{m,k}(\epsilon^{k,t+1}) -\nabla\ell_{m,k}(\epsilon^{k,t})\|\\
&\quad+\sum_{k=0}^{K-1} \|\nabla\ell_{m,k+1}(\mathbf0) -\nabla\ell_{m,k}(\epsilon^k)\|\\*
&\leq\beta\sum_{k=0}^{K-1}\sum_{t=0}^{T_k-1} \|\epsilon^{k,t+1}-\epsilon^{k,t}\| +(\kappa+\beta)\sum_{k=0}^{K-1}\|\epsilon^k\|.
\end{align*}
The final origin query contributes the last boundary term, whereas the terminating action creates no additional transition. Summing \eqref{eq:proof-within-variation} and \eqref{eq:proof-recenter-variation}, then using \eqref{eq:proof-path-length} before its final simplification, proves
\begin{align}
\sum_{r=2}^R\|g_{m,r}-g_{m,r-1}\| &=\sum_{r\in\mathcal S}\|g_{m,r+1}-g_{m,r}\| +\sum_{r:\mathsf a_r=01}\|g_{m,r+1}-g_{m,r}\|\notag\\*
&\leq\sum_{r\in\mathcal S}\|\widetilde g_r\| +(\kappa+\beta)\sum_{k=0}^{K-1}\|\epsilon^k\|\notag\\
&\leq\left(1+\frac{\kappa+\beta}{\beta}\right) \sum_{r\in\mathcal S}\|\widetilde g_r\|\notag\\*
&\leq(2+\kappa/\beta)\sqrt{6\beta\Delta S}.
\label{eq:proof-total-variation}
\end{align}
The pooled gradients obey the same bound because their sample-size weights remain fixed,
\begin{align}
\sum_{r=2}^R\|g_r-g_{r-1}\| &=\sum_{r=2}^R\left\|\sum_{m=1}^M\frac{N_m}{n} (g_{m,r}-g_{m,r-1})\right\|\notag\\*
&\leq\sum_{m=1}^M\frac{N_m}{n} \sum_{r=2}^R\|g_{m,r}-g_{m,r-1}\| \leq (2+\kappa/\beta)\sqrt{6\beta\Delta S}.
\label{eq:proof-pooled-variation}
\end{align}
Equations~\eqref{eq:proof-energy-and-steps} and \eqref{eq:proof-total-variation} prove \eqref{eq:variation-summary}. The variation across submodels has been derived from the two Lipschitz assumptions, with no separate bound imposed on successive-submodel differences. Only the first message encodes an absolute value, since all memories start at zero. The initial bound on the unrounded local gradients implies
\begin{align*}
\|z_{m,1}\|_\infty &\leq G/h+1/2,\\*
\|q_1\|_\infty &\leq\left\|\sum_{m=1}^M\frac{N_m}{n}z_{m,1}\right\|_\infty+1/2 \leq G/h+1.
\end{align*}
Combining these inequalities with \eqref{eq:proof-code-length}, the uploaded and broadcast vectors and their control flags in round one cost at most
\begin{align}
&\sum_{m=1}^M|\operatorname{enc}(z_{m,1})| +M|\operatorname{enc}(q_1)|+2M\notag\\*
&\quad\leq6Md +2Md\log_2(3/2+G/h) +2Md\log_2(2+G/h)+2M\notag\\
&\quad\leq8Md+4Md\log_2(2+G/h)\notag\\*
&\quad\leq8Md+4Md\log_2\!\left(2+\frac{8G\sqrt d}{\delta}\right).
\label{eq:proof-initial-bits}
\end{align}
No absolute gradient bound is used after this round. Local and server messages require separate bounds, since the broadcast gradient includes two rounding errors. For $r\geq2$, successive queries satisfy
\begin{align*}
|z_{m,r,j}-z_{m,r-1,j}| &=\frac1h\big|g_{m,r,j}-g_{m,r-1,j} +(hz_{m,r,j}-g_{m,r,j}) -(hz_{m,r-1,j}-g_{m,r-1,j})\big|\\*
&\leq\frac{|g_{m,r,j}-g_{m,r-1,j}|}{h} +\frac{|hz_{m,r,j}-g_{m,r,j}|}{h} +\frac{|hz_{m,r-1,j}-g_{m,r-1,j}|}{h}\\
&\leq h^{-1}|g_{m,r,j}-g_{m,r-1,j}|+1,\\
|q_{r,j}-q_{r-1,j}| &=\frac1h\big|g_{r,j}-g_{r-1,j} +(hq_{r,j}-g_{r,j}) -(hq_{r-1,j}-g_{r-1,j})\big|\\
&\leq\frac{|g_{r,j}-g_{r-1,j}|}{h} +\frac{|hq_{r,j}-g_{r,j}|}{h} +\frac{|hq_{r-1,j}-g_{r-1,j}|}{h}\\*
&\leq h^{-1}|g_{r,j}-g_{r-1,j}|+2.
\end{align*}
The bound on the server difference uses $|h q_{r,j}-g_{r,j}|\leq h$ from \eqref{eq:proof-rounding-error}. Applying concavity of the logarithm and $\|v\|_1\leq\sqrt d\|v\|$ to the code lengths gives
\begin{align*}
|\operatorname{enc}(z_{m,r}-z_{m,r-1})| &\leq3d+2\sum_{j=1}^d \log_2\!\left(2+\frac{|g_{m,r,j}-g_{m,r-1,j}|}{h}\right)\\*
&\leq3d+2d\log_2\!\left( 2+\frac{\|g_{m,r}-g_{m,r-1}\|_1}{dh}\right)\\
&\leq3d+2d\log_2\!\left( 2+\frac{\|g_{m,r}-g_{m,r-1}\|}{h\sqrt d}\right),\\*
|\operatorname{enc}(q_r-q_{r-1})| &\leq3d+2d\log_2\!\left( 3+\frac{\|g_r-g_{r-1}\|}{h\sqrt d}\right).
\end{align*}
If $R>1$, a second application of concavity, now over rounds and silos, gives the complete upload bound
\begin{align}
&\sum_{m=1}^M\sum_{r=2}^R |\operatorname{enc}(z_{m,r}-z_{m,r-1})| \leq3Md(R-1)+2d\sum_{m=1}^M\sum_{r=2}^R \log_2\!\left(2+ \frac{\|g_{m,r}-g_{m,r-1}\|}{h\sqrt d}\right)\notag\\
&\quad\leq3Md(R-1)+2Md(R-1) \log_2\!\left(2+ \frac{\sum_{m=1}^M\sum_{r=2}^R \|g_{m,r}-g_{m,r-1}\|} {M(R-1)h\sqrt d}\right)\notag\\*
&\quad\leq3Md(R-1)+2Md(R-1) \log_2\!\left(2+\frac{(2+\kappa/\beta)\sqrt{6\beta\Delta S}}{(R-1)h\sqrt d}\right).
\label{eq:proof-upload-sum}
\end{align}
Similarly, \eqref{eq:proof-pooled-variation} gives
\begin{align}
M\sum_{r=2}^R|\operatorname{enc}(q_r-q_{r-1})| &\leq3Md(R-1)+2Md\sum_{r=2}^R \log_2\!\left(3+ \frac{\|g_r-g_{r-1}\|}{h\sqrt d}\right)\notag\\*
&\leq3Md(R-1)+2Md(R-1) \log_2\!\left(3+\frac{(2+\kappa/\beta)\sqrt{6\beta\Delta S}}{(R-1)h\sqrt d}\right).
\label{eq:proof-broadcast-sum}
\end{align}
There are exactly $2MR$ control bits over the complete execution. Thus, using \eqref{eq:proof-initial-bits}, \eqref{eq:proof-upload-sum}, and \eqref{eq:proof-broadcast-sum},
\begin{align}
\mathsf B &=\sum_{m=1}^M\sum_{r=1}^R |\operatorname{enc}(z_{m,r}-z_{m,r-1})| +M\sum_{r=1}^R|\operatorname{enc}(q_r-q_{r-1})|+2MR\notag\\*
&\leq8MdR+4Md\log_2(2+G/h) +4Md(R-1)\log_2\!\left(3+ \frac{(2+\kappa/\beta)\sqrt{6\beta\Delta S}}{(R-1)h\sqrt d}\right)\notag\\
&\leq8MdR+4Md\log_2\!\left(2+ \frac{8G\sqrt d}{\delta}\right)\notag\\*
&\qquad+4Md(R-1)\log_2\!\left(3+ \frac{8(2+\kappa/\beta)\sqrt{6\beta\Delta S}} {(R-1)\delta}\right).
\label{eq:proof-pathwise-bits}
\end{align}
For $R=1$, only \eqref{eq:proof-initial-bits} is needed and the last term is absent.

It remains to eliminate the realized step and round counts from the logarithmic term. For this calculation, put $a=\beta\Delta/\delta^2$, $c=2+\kappa/\beta\geq2$, and $x=G\sqrt d/\delta$. If $R>1$, then $a>0$ and $S\leq R-1\leq64a/3$. For $u>0$, the function $u\log_2(3+A/\sqrt u)$ is increasing whenever $A\geq0$, as follows from
\begin{align*}
\frac{d}{du}\{u\log_2(3+A/\sqrt u)\} &=\frac1{\log2}\left\{ \log(3+A/\sqrt u) -\frac{A/\sqrt u}{2(3+A/\sqrt u)}\right\}\\*
&\geq\frac{\log3-1/2}{\log2}>0.
\end{align*}
With $u=R-1$ and $A=8c\sqrt{6a}$, the logarithmic term in \eqref{eq:proof-pathwise-bits} consequently obeys
\begin{align*}
(R-1)\log_2\!\left(3+ \frac{8c\sqrt{6aS}}{R-1}\right) &\leq(R-1)\log_2\!\left(3+ 8c\sqrt{\frac{6a}{R-1}}\right)\\
&\leq\frac{64a}{3}\log_2\!\left(3+ 8c\sqrt{\frac{18}{64}}\right)\\*
&=\frac{64a}{3}\log_2(3+3\sqrt2\,c) \leq\frac{128a}{3}\log_2(c+2).
\end{align*}
The final comparison follows by factoring the difference:
\begin{align*}
(c+2)^2-(3+3\sqrt2\,c) &=c^2+(4-3\sqrt2)c+1\\*
&=(c-2)(c+6-3\sqrt2)+13-6\sqrt2\geq0, \qquad c\geq2.
\end{align*}
Also $\log_2(2+8x)\leq3+\log_2(2+x)$ and $\log_2(c+2)\geq2$. Substituting into the complete cost gives
\begin{align*}
\frac{\mathsf B}{Md} &\leq8+\frac{512a}{3}+4\log_2(2+8x) +\frac{256a}{3}\log_2(3+3\sqrt2\,c)\\*
&\leq20+4\log_2(2+x)+\frac{512a}{3} +\frac{512a}{3}\log_2(c+2)\\
&\leq20+4\log_2(2+x)+256a\log_2(c+2)\\*
&\leq256\{(1+a)\log_2(c+2)+\log_2(2+x)\}.
\end{align*}
For $R=1$, \eqref{eq:proof-initial-bits} gives the same bound. Substituting the definitions of $a,c,x$ proves \eqref{eq:main-communication}, with $C=256$ as one valid universal constant. The count includes the final certificate round. As specified in the accounting model, it excludes initial model transfer, transport headers, and security mechanisms.
\end{proof}

\subsection{Proof of Theorem~\ref{thm:statistical}}\label{app:statistical}
\begin{proof}
The concentration argument uses the full countable output family. 
The accepted fluctuations suffice to reconstruct the terminal density. To encode them, note that the updates start each submodel at zero and satisfy
\begin{align*}
\epsilon^{k,t} &=\epsilon^{k,0} +\sum_{u=0}^{t-1} (\epsilon^{k,u+1}-\epsilon^{k,u})\\*
&=-\frac h\beta \sum_{\substack{r\in\mathcal S\\k_r=k,\ t_r<t}}q_r,\\
\frac{\beta\epsilon^{k,t}}h &=-\sum_{\substack{r\in\mathcal S\\k_r=k,\ t_r<t}}q_r \in\mathbb Z^d,\\*
\frac{\beta\epsilon^k}h &=\frac{\beta\epsilon^{k,T_k}}h =-\sum_{\substack{r\in\mathcal S\\k_r=k}}q_r \in\mathbb Z^d.
\end{align*}
When an initial-state description is required, the decoder reads it first, followed by the gamma code for $K+1$ and exactly $Kd$ signed-integer codewords. It recovers each $\epsilon^k$ by multiplying its integer vector by $h/\beta$. The decoder constructs its submodel from the previously decoded state and applies $p^{k+1}=p^k(\epsilon^k)$, reconstructing $p^K$ by induction. This reconstruction requires no inner stopping times, intermediate gradients, server accumulators, or rejected inner parameters. Nor must the decoder verify that a particular dataset would produce the history.

The description length depends on the magnitudes of the accepted integer coefficients. Equations~\eqref{eq:proof-path-length} and \eqref{eq:proof-energy-and-steps} bound their total absolute magnitude in terms of the communicated GD steps,
\begin{align*}
\sum_{k=0}^{K-1}\sum_{j=1}^d \left|\frac{\beta\epsilon_j^k}{h}\right| &=\frac\beta h\sum_{k=0}^{K-1}\|\epsilon^k\|_1\\*
&\leq\frac{\beta\sqrt d}{h}\sum_{k=0}^{K-1} \left\|\sum_{t=0}^{T_k-1} (\epsilon^{k,t+1}-\epsilon^{k,t})\right\|\\
&\leq\frac{\beta\sqrt d}{h}\sum_{k=0}^{K-1} \sum_{t=0}^{T_k-1} \|\epsilon^{k,t+1}-\epsilon^{k,t}\|\\
&=\frac{\sqrt d}{h}\sum_{r\in\mathcal S}\|\widetilde g_r\|\\
&\leq\frac{\sqrt{6d\beta\Delta S}}h
\leq\frac{8\beta\Delta\sqrt d}{h\delta} \leq\frac{64d\beta\Delta}{\delta^2}.
\end{align*}
For $K=0$ the sum is empty. The gamma code of $1$ is a single bit, so $H_0=H_{\mathrm{init}}+1$. For $K\geq1$, the exact code length and concavity of the logarithm give
\begin{align*}
H_K-H_{\mathrm{init}} &=2\lfloor\log_2(K+1)\rfloor+1 +\sum_{k=0}^{K-1}\sum_{j=1}^d \left|\operatorname{enc} \left(\frac{\beta\epsilon_j^k}h\right)\right|\\*
&\leq1+2\log_2(K+1)+3Kd +2\sum_{k=0}^{K-1}\sum_{j=1}^d \log_2\!\left(1+\left| \frac{\beta\epsilon_j^k}h\right|\right)\\
&\leq1+2\log_2(K+1)+3Kd +2Kd\log_2\!\left(1+\frac1{Kd} \sum_{k=0}^{K-1}\sum_{j=1}^d \left|\frac{\beta\epsilon_j^k}h\right|\right)\\*
&\leq1+2\log_2(K+1)+3Kd +2Kd\log_2\!\left(1+\frac{64\beta\Delta}{K\delta^2}\right).
\end{align*}
Since $\log_2(K+1)\leq K\leq Kd$ and $\log_2(1+64u)\leq6+\log_2(2+u)$ for $u\geq0$, this implies
\begin{align*}
H_K-H_{\mathrm{init}} &\leq1+17Kd +2Kd\log_2\!\left(2+\frac{\beta\Delta}{K\delta^2}\right)\\*
&\leq1+19Kd\log_2\!\left(2+\frac{\beta\Delta}{K\delta^2}\right)\\*
&\leq20\left\{1+Kd\log_2\!\left( 2+\frac{\beta\Delta}{K\delta^2}\right)\right\}.
\end{align*}
This establishes \eqref{eq:history-length}; the larger universal constant in Theorem~\ref{thm:communication} also suffices here. The bound holds deterministically for the realized history and has not used a probability argument or complexity assumption.

For the probability bound, impose Assumption~\ref{ass:statistical-code} and consider all outputs of its fixed decoder. The complete code is prefix-free. The end of a prefix-free initial-state code is recognizable, so the history code can be appended to it; the gamma code of $K+1$ then determines how many signed-integer codewords to read. This remains true for variable $K$. The Kraft mass of the signed-integer code is
\begin{align}
\sum_{a\in\mathbb Z}2^{-|\operatorname{enc}(a)|} &=\sum_{v=1}^{\infty} 2^{-\{2\lfloor\log_2v\rfloor+1\}}\notag\\*
&=\sum_{j=0}^{\infty} \sum_{v=2^j}^{2^{j+1}-1}2^{-(2j+1)}\notag\\*
&=\sum_{j=0}^{\infty}2^j2^{-(2j+1)} =\sum_{j=0}^{\infty}2^{-j-1}=1.
\label{eq:proof-integer-kraft}
\end{align}
Let $\mathcal W$ be the set of all complete codewords that the fixed decoder maps to a valid density, including codewords absent from the observed optimization path. With $w_0$ ranging over initial-state codewords, Kraft's inequality \citep{cover2006elements} and \eqref{eq:proof-integer-kraft} give
\begin{align}
\sum_{w\in\mathcal W}2^{-|w|} &\leq\sum_{w_0}2^{-|w_0|} \sum_{K=0}^{\infty}2^{-\{2\lfloor\log_2(K+1)\rfloor+1\}} \prod_{k=0}^{K-1}\prod_{j=1}^d \left(\sum_{a\in\mathbb Z} 2^{-|\operatorname{enc}(a)|}\right)\notag\\*
&\leq\sum_{K=0}^{\infty} 2^{-\{2\lfloor\log_2(K+1)\rfloor+1\}} =1.
\label{eq:proof-history-kraft}
\end{align}
With independent initialization, condition on the initial state and replace the first sum by one empty initial code of mass one. Under this conditioning, the observations remain i.i.d. from $p_0$. With same-sample initialization, retain its entire description in $w$ and do not condition on the initial state; \eqref{eq:proof-history-kraft} then applies unconditionally. Multiple codewords may represent the same density. Counting them separately is valid and can only enlarge the subsequent probability bound.

We apply a variance-sensitive concentration argument \citep{boucheron2013concentration} to each fixed decoded density before evaluating the bound at the selected output. Fix $w\in\mathcal W$, denote its decoded density by $p_w$, and fix a coordinate $j$. Write
\begin{align*}
f_{w,j}(o)&:=D_j(p_w)(o)-D_j(p_0)(o),\\*
X_i&:=f_{w,j}(O_i)-P_0f_{w,j},\qquad v_{w,j}:=P_0(f_{w,j}-P_0f_{w,j})^2.
\end{align*}
On the probability space under consideration, fixing the codeword fixes $p_w$. Thus the $X_i$ are independent; this argument does not condition on the output selected by the algorithm. The envelope bound and centering imply
\begin{align*}
P_0 f_{w,j}&=P_0D_j(p_w)-P_0D_j(p_0),\\*
|f_{w,j}(o)|&\leq|D_j(p_w)(o)|+|D_j(p_0)(o)|\leq2G,\\
|X_i|&\leq|f_{w,j}(O_i)|+P_0|f_{w,j}|\leq4G,\\
\mathbb E X_i&=P_0f_{w,j}-P_0f_{w,j}=0,\\
\mathbb E X_i^2 &=P_0(f_{w,j}-P_0f_{w,j})^2\\*
&=P_0f_{w,j}^2-(P_0f_{w,j})^2 =v_{w,j}\leq P_0f_{w,j}^2.
\end{align*}
For every integer $\ell\geq2$, the higher moments therefore satisfy
\begin{align}
|\mathbb E X_i^\ell| &\leq\mathbb E|X_i|^\ell =\mathbb E\{|X_i|^{\ell-2}X_i^2\}\notag\\*
&\leq(4G)^{\ell-2}\mathbb E X_i^2 =(4G)^{\ell-2}v_{w,j},\notag\\*
\ell!&=2\prod_{u=3}^\ell u\geq2\,3^{\ell-2}.
\label{eq:proof-moment-control}
\end{align}
The product is one for $\ell=2$. Boundedness justifies termwise expectation in the exponential series. For $4G|\lambda|<3$, the geometric majorant in \eqref{eq:proof-moment-control} yields
\begin{align*}
\mathbb E e^{\lambda X_i} &=1+\lambda\mathbb E X_i +\sum_{\ell=2}^{\infty}\frac{\lambda^\ell\mathbb E X_i^\ell}{\ell!}\\*
&\leq1+\sum_{\ell=2}^{\infty} \frac{|\lambda|^\ell|\mathbb E X_i^\ell|}{\ell!}\\
&\leq1+v_{w,j}\sum_{\ell=2}^{\infty} \frac{|\lambda|^\ell(4G)^{\ell-2}}{\ell!}\\
&\leq1+\frac{\lambda^2v_{w,j}}2 \sum_{\ell=2}^{\infty} \left(\frac{4G|\lambda|}{3}\right)^{\ell-2}\\
&=1+\frac{\lambda^2v_{w,j}} {2(1-4G|\lambda|/3)}\\
&\leq\exp\!\left\{ \frac{\lambda^2v_{w,j}}{2(1-4G|\lambda|/3)}\right\},\\
\log\mathbb E\exp\!\left(\lambda\sum_{i=1}^nX_i\right) &=\sum_{i=1}^n\log\mathbb E e^{\lambda X_i}
\leq\frac{n\lambda^2v_{w,j}}{2(1-4G|\lambda|/3)}.
\end{align*}
For this scalar calculation, write $c=4G/3$ and first assume $c>0$ and $v_{w,j}>0$. For $t>0$, the exponential Markov inequality gives
\begin{align*}
\mathbb P\{(P_n-P_0)f_{w,j}>t\} &=\mathbb P\!\left\{\sum_{i=1}^nX_i>nt\right\}\\*
&\leq\inf_{0<\lambda<1/c} e^{-n\lambda t}\, \mathbb E\exp\!\left(\lambda\sum_{i=1}^nX_i\right)\\*
&\leq\exp\!\left[-n\sup_{0<\lambda<1/c} \left\{\lambda t- \frac{\lambda^2v_{w,j}}{2(1-c\lambda)}\right\}\right].
\end{align*}
The objective in this supremum is strictly concave,
\begin{align*}
\frac{d}{d\lambda} \left\{\lambda t-\frac{\lambda^2v_{w,j}}{2(1-c\lambda)}\right\} &=t-\frac{v_{w,j}\lambda(2-c\lambda)}{2(1-c\lambda)^2}\\*
&=t-\frac{v_{w,j}}{2c} \{(1-c\lambda)^{-2}-1\},\\*
\frac{d^2}{d\lambda^2} \left\{\lambda t-\frac{\lambda^2v_{w,j}}{2(1-c\lambda)}\right\} &=-\frac{v_{w,j}}{(1-c\lambda)^3}<0.
\end{align*}
The derivative is positive at zero and tends to $-\infty$ at $1/c$. Its unique zero is consequently the maximizer,
\begin{align*}
(1-c\lambda_*)^{-2} &=1+\frac{2ct}{v_{w,j}},\\*
\lambda_* &=\frac1c\left(1-\sqrt{\frac{v_{w,j}}{v_{w,j}+2ct}}\right) \in(0,1/c),\\*
1-c\lambda_* &=\sqrt{\frac{v_{w,j}}{v_{w,j}+2ct}}.
\end{align*}
Substituting the maximizing value evaluates the supremum exactly,
\begin{align*}
\lambda_*t- \frac{\lambda_*^2v_{w,j}}{2(1-c\lambda_*)} &=\frac tc\left(1- \sqrt{\frac{v_{w,j}}{v_{w,j}+2ct}}\right) -\frac{v_{w,j}}{2c^2} \frac{\left(1-\sqrt{v_{w,j}/(v_{w,j}+2ct)}\right)^2} {\sqrt{v_{w,j}/(v_{w,j}+2ct)}}\\
&=\frac{ 2ct\{\sqrt{v_{w,j}+2ct}-\sqrt{v_{w,j}}\}} {2c^2\sqrt{v_{w,j}+2ct}} -\frac{\sqrt{v_{w,j}} \{\sqrt{v_{w,j}+2ct}-\sqrt{v_{w,j}}\}^2} {2c^2\sqrt{v_{w,j}+2ct}}\\
&=\frac{\sqrt{v_{w,j}+2ct}-\sqrt{v_{w,j}}} {2c^2\sqrt{v_{w,j}+2ct}} \left[2ct-\sqrt{v_{w,j}} \{\sqrt{v_{w,j}+2ct}-\sqrt{v_{w,j}}\}\right]\\
&=\frac{\sqrt{v_{w,j}+2ct}-\sqrt{v_{w,j}}} {2c^2\sqrt{v_{w,j}+2ct}} \left[v_{w,j}+2ct- \sqrt{v_{w,j}(v_{w,j}+2ct)}\right]\\*
&=\frac{ \{\sqrt{v_{w,j}+2ct}-\sqrt{v_{w,j}}\}^2}{2c^2}.
\end{align*}
The same argument applies to $-X_i$. Thus
\begin{equation*}
\mathbb P\{|(P_n-P_0)f_{w,j}|>t\} \leq2\exp\!\left[-\frac{n}{2c^2} \{\sqrt{v_{w,j}+2ct}-\sqrt{v_{w,j}}\}^2\right].
\end{equation*}
For $x>0$, choose $t_x=\sqrt{2v_{w,j}x/n}+cx/n$. The exponent is exactly $x$, because
\begin{align*}
v_{w,j}+2ct_x &=v_{w,j}+2c\sqrt{\frac{2v_{w,j}x}{n}} +\frac{2c^2x}{n}\\*
&=\left(\sqrt{v_{w,j}}+c\sqrt{\frac{2x}{n}}\right)^2,\\*
\frac{n}{2c^2} \{\sqrt{v_{w,j}+2ct_x}-\sqrt{v_{w,j}}\}^2 &=\frac{n}{2c^2} \left(c\sqrt{\frac{2x}{n}}\right)^2=x.
\end{align*}
Restoring $c=4G/3$ gives the fixed-code bound
\begin{equation*}
\mathbb P\!\left\{|(P_n-P_0)f_{w,j}|> \sqrt{\frac{2v_{w,j}x}{n}}+\frac{4Gx}{3n}\right\} \leq2e^{-x}.
\end{equation*}
If $v_{w,j}=0$, the centered variable vanishes almost surely; if $G=0$, every coordinate vanishes. Retaining the variance of the EIF difference in the square-root term allows this part of the statistical penalty to decrease as the estimated EIF approaches the true EIF. 
To make the concentration event simultaneous over decoded histories of every length, assign $x_w=|w|\log2+\log(2d/\xi)$ to each codeword. The resulting failure probabilities are summable by \eqref{eq:proof-history-kraft}. A union bound over codewords and coordinates gives
\begin{align*}
&\mathbb P\!\left\{\exists w\in\mathcal W,\ 1\leq j\leq d: |(P_n-P_0)f_{w,j}|> \sqrt{\frac{2v_{w,j}x_w}{n}}+\frac{4Gx_w}{3n}\right\} \leq\sum_{w\in\mathcal W}\sum_{j=1}^d2e^{-x_w}\\
&\quad=\sum_{w\in\mathcal W}\sum_{j=1}^d 2\exp\{-|w|\log2-\log(2d/\xi)\}\\*
&\quad=\xi\sum_{w\in\mathcal W}2^{-|w|} \leq\xi.
\end{align*}
On the complementary event, all coordinate bounds hold for every decoded density. Combining them by Cauchy--Schwarz introduces an explicit factor $\sqrt d$ only in the envelope term,
\begin{align}
&\|(P_n-P_0)\{D(p_w)-D(p_0)\}\| =\left\{\sum_{j=1}^d |(P_n-P_0)f_{w,j}|^2\right\}^{1/2}\notag\\
&\quad\leq\left\{\sum_{j=1}^d \left(\sqrt{\frac{2v_{w,j}x_w}{n}} +\frac{4Gx_w}{3n}\right)^2\right\}^{1/2}\notag\\
&\quad=\left\{\frac{2x_w}{n}\sum_{j=1}^dv_{w,j} +\frac{8Gx_w}{3n}\sqrt{\frac{2x_w}{n}} \sum_{j=1}^d\sqrt{v_{w,j}} +\frac{16dG^2x_w^2}{9n^2}\right\}^{1/2}\notag\\
&\quad\leq\left\{\frac{2x_w}{n}\sum_{j=1}^dv_{w,j} +\frac{8G\sqrt d\,x_w}{3n} \sqrt{\frac{2x_w}{n}\sum_{j=1}^dv_{w,j}} +\frac{16dG^2x_w^2}{9n^2}\right\}^{1/2}\notag\\
&\quad=\sqrt{\frac{2x_w}{n}} \left(\sum_{j=1}^dv_{w,j}\right)^{1/2} +\frac{4G\sqrt d\,x_w}{3n}\notag\\
&\quad=\sqrt{\frac{2x_w}{n}} \left\{P_0\sum_{j=1}^df_{w,j}^2 -\sum_{j=1}^d(P_0f_{w,j})^2\right\}^{1/2} +\frac{4G\sqrt d\,x_w}{3n}\notag\\*
&\quad\leq\sqrt{\frac{2x_w}{n}}\, \|D(p_w)-D(p_0)\|_{L^2(p_0)} +\frac{4G\sqrt d\,x_w}{3n}.
\label{eq:proof-vector-concentration}
\end{align}
There are countably many complete codewords and finitely many coordinates, so the envelope and fixed-code statements hold on a common probability-one set. The algorithm's realized output has a codeword of length $H_K$. We may substitute it into the simultaneous concentration bound even though its length, accepted fluctuations, and stopping index depend on the sample. This gives the empirical-process bound with $x_w=x_K=H_K\log2+\log(2d/\xi)$. Because the failure probabilities were allocated across all descriptions, no a priori deterministic bound on $K$ or $R$ is needed, and the argument does not rely on fixed-time confidence statements.

The empirical-process event covers every valid decoded density, including nonstationary densities; the inequality $r^K\leq\delta$ is used only at the endpoint certified by the algorithm. By \eqref{eq:remainder} and the centering identity $P_0D(p_0)=\mathbf0$,
\begin{align*}
&\Psi(p^K)-\Psi(p_0)-P_nD(p_0) =-P_0D(p^K)+R_2(p^K,p_0)-P_nD(p_0)\\
&\quad=-P_nD(p^K)+(P_n-P_0)D(p^K) +R_2(p^K,p_0)-P_nD(p_0)\\
&\quad=-P_nD(p^K)+(P_n-P_0)\{D(p^K)-D(p_0)\} +R_2(p^K,p_0)-P_0D(p_0)\\*
&\quad=-P_nD(p^K)+(P_n-P_0)\{D(p^K)-D(p_0)\} +R_2(p^K,p_0).
\end{align*}
Taking norms, applying the terminal certificate from Theorem~\ref{thm:communication}, and using \eqref{eq:proof-vector-concentration}, all on the same event, gives
\begin{align*}
&\|\Psi(p^K)-\Psi(p_0)-P_nD(p_0)\|\\*
&\quad\leq\|P_nD(p^K)\| +\|(P_n-P_0)\{D(p^K)-D(p_0)\}\| +\|R_2(p^K,p_0)\|\\*
&\quad\leq\delta+\|R_2(p^K,p_0)\| +\sqrt{\frac{2x_K}{n}}\, \|D(p^K)-D(p_0)\|_{L^2(p_0)} +\frac{4G\sqrt d\,x_K}{3n}.
\end{align*}
The nonlinear remainder remains a separate term in the exact expansion. Bounding its order requires a statistical argument about the terminal density; neither the EIF residual nor the code length supplies that bound.
\end{proof}

\subsection{Proof of Corollary~\ref{cor:inference}}\label{app:inference}
\begin{proof}
Set $\delta_n=(a_n\sqrt n)^{-1}$. Then $\beta\Delta/\delta_n^2=\beta\Delta\,na_n^2=O_{\mathbb P}(na_n^2)$. Applying Theorem~\ref{thm:communication} with the stipulated bounded constants yields
\begin{align*}
R &\leq1+\frac{64}{3}\beta\Delta\,na_n^2 =O_{\mathbb P}(na_n^2),\\*
\frac{\mathsf B}{Md} &\leq C\left\{(1+\beta\Delta\,na_n^2) \log_2(4+\kappa/\beta) +\log_2(2+G\sqrt{dn}\,a_n)\right\}\\*
&=O_{\mathbb P}(na_n^2) +O\!\left(1+\log(na_n^2)\right) =O_{\mathbb P}(na_n^2).
\end{align*}
Since $d$ is fixed and $na_n^2\to\infty$, these estimates prove \eqref{eq:inference-cost}. The cost retains its explicit dependence on $M$, which does not enter the empirical-process argument. To derive \eqref{eq:history-asymptotic} from \eqref{eq:history-length} for random $K$, work on the event $\beta\Delta\leq A$, where $A$ is a fixed finite constant, and put $A_+=\max(1,A)$. For every $u\geq0$,
\begin{align*}
\log_2(2+Au) &\leq\log_2\{A_+(2+u)\}\\*
&=\log_2 A_++\log_2(2+u)\\*
&\leq(1+\log_2 A_+)\log_2(2+u).
\end{align*}
For $K\geq1$, substitute $u=na_n^2/K$ in this deterministic bound; for $K=0$, use $H_0-H_{\mathrm{init}}=1$. Since $\beta\Delta=O_{\mathbb P}(1)$, the probability of the event $\beta\Delta\leq A$ can be made arbitrarily close to one by choosing $A$ large. Thus
\begin{equation*}
H_K-H_{\mathrm{init}} =O_{\mathbb P}\!\left( 1+Kd\log\!\left(2+\frac{na_n^2}{K\vee1}\right)\right).
\end{equation*}
If $K=O_{\mathbb P}(1)$, $H_{\mathrm{init}}=0$, and $\log a_n=o(\log n)$, the right-hand side is $O_{\mathbb P}(\log n)$, establishing the example discussed after the corollary without assuming a deterministic bound on the realized stopping index. For $n\geq2$, apply Theorem~\ref{thm:statistical} with $\xi=1/n$. Its simultaneous bound holds with failure probability at most $1/n$. Since $d$ is fixed,
\begin{equation*}
x_K=H_K\log2+\log(2dn) \leq C_d(H_K+\log n)
\end{equation*}
for a finite constant $C_d$. On this event, multiplication of \eqref{eq:history-statistical-bound} by $\sqrt n$ gives
\begin{align}
&\sqrt n\,\|\Psi(p^K)-\Psi(p_0)-P_nD(p_0)\|\notag\\*
&\quad\leq\frac1{a_n} +\sqrt n\,\|R_2(p^K,p_0)\| +\sqrt{2x_K}\,\|D(p^K)-D(p_0)\|_{L^2(p_0)} +\frac{4G\sqrt d\,x_K}{3\sqrt n}\notag\\
&\quad\leq\frac1{a_n} +\sqrt n\,\|R_2(p^K,p_0)\|\notag\\*
&\qquad\quad+\sqrt{2C_d}\, \sqrt{H_K+\log n}\, \|D(p^K)-D(p_0)\|_{L^2(p_0)} +\frac{4C_dG\sqrt d\,(H_K+\log n)}{3\sqrt n}.
\label{eq:proof-root-n-bound}
\end{align}
By \eqref{eq:inference-conditions} and $a_n\to\infty$, every term on the last two lines tends to zero in probability. In detail, fix $\eta>0$ and let $U_n$ denote the final right-hand side of \eqref{eq:proof-root-n-bound}. A union bound gives
\begin{equation*}
\mathbb P\!\left\{\sqrt n\, \|\Psi(p^K)-\Psi(p_0)-P_nD(p_0)\|>\eta\right\} \leq\frac1n+\mathbb P(U_n>\eta)\longrightarrow0.
\end{equation*}
For independent initialization, the concentration argument is initially conditional on that state; integrating the conditional failure probability retains the bound $1/n$. Finally, $D(p_0)(O_i)$ are i.i.d., centered, and square-integrable. For any $v\in\mathbb R^d$,
\begin{align*}
\mathbb E\{v^\top D(p_0)(O_i)\}&=v^\top P_0D(p_0)=0,\\*
\mathbb E\bigl[\{v^\top D(p_0)(O_i)\}^2\bigr] &=v^\top P_0[D(p_0)D(p_0)^\top]v<\infty,\\
v^\top\sqrt n\,P_nD(p_0) &=\frac1{\sqrt n}\sum_{i=1}^n v^\top D(p_0)(O_i)\\*
&\rightsquigarrow \mathcal N\!\left(0, v^\top P_0[D(p_0)D(p_0)^\top]v\right).
\end{align*}
Applying the scalar central limit theorem, Cram\'er--Wold, and Slutsky's theorem \citep{vandervaart1998asymptotic} proves \eqref{eq:canonical-limit}, with no requirement that the canonical-gradient covariance be nonsingular. The conclusion is pointwise at $p_0$. Regularity under local alternatives requires a separate argument and does not follow from the communication or coding bounds.
\end{proof}

\subsection{Proof of Theorem~\ref{thm:reconstruction}}
\label{app:reconstruction}
\begin{proof}
Let $\mathsf Q$ be a deterministic adaptive query policy, and let $\mathsf T_{\mathsf Q}(x,e)$ denote the transcript of its queries and responses under \eqref{eq:robust-channel}. Write $\mathcal X_s$ for the class of compatible datasets in Theorem~\ref{thm:reconstruction}, and label the $s$ undisclosed records in each dataset by $1,\ldots,s$. The record distance and minimax risk are
\begin{align*}
\operatorname{dist}(x,x') &=\min_{\sigma\in\mathfrak S_s}\max_{1\leq j\leq s} \|o_j-o'_{\sigma(j)}\|,\\
\mathcal R_T^{\mathrm{rec}}(\eta) &=\inf_{\mathsf Q,\widehat x} \sup_{\substack{x\in\mathcal X_s\\\|e\|\leq\eta}} \operatorname{dist}\bigl(\widehat x(\mathsf T_{\mathsf Q}(x,e)),x\bigr).
\end{align*}
The vector $e$ stacks all response perturbations as in \eqref{eq:robust-channel}. The remaining records and the public information stay fixed throughout.

Let $\mathsf C_s(u)=\cos(s\arccos u)$ be the Chebyshev polynomial; this notation is distinct from the derivative bound $C_s$. For $\theta\in[-1/2,1/2]$ and $j=0,\ldots,s-1$, define
\begin{align*}
\phi_j(\theta) &=\frac{j\pi+\arccos((-1)^j\theta)}{s}, \qquad u_j(\theta)=\cos\phi_j(\theta),\\
\mathsf C_s(u_j(\theta)) &=\cos\{j\pi+\arccos((-1)^j\theta)\} =(-1)^j(-1)^j\theta=\theta.
\end{align*}
The defining angles give the ordering
\begin{equation*}
\frac{j\pi}{s}<\phi_j(\theta)<\frac{(j+1)\pi}{s}, \qquad 1>u_0(\theta)>\cdots>u_{s-1}(\theta)>-1.
\end{equation*}
The resulting points are all $s$ roots of $\mathsf C_s-\theta$. None is a critical point of $\mathsf C_s$, since
\begin{equation*}
(\mathsf C_s)'(u_j(\theta)) =\frac{s\sin(s\phi_j(\theta))}{\sin\phi_j(\theta)}, \qquad |\sin(s\phi_j(\theta))|=\sqrt{1-\theta^2}\geq\sqrt3/2.
\end{equation*}
Thus the implicit derivatives used below are finite and continuous throughout the level interval. For $0<a\leq \rho$, define the undisclosed records of $x_+$ and $x_-$ by
\begin{equation*}
o_{j+1}^{+}=o_\circ+a u_j(1/2)v, \qquad o_{j+1}^{-}=o_\circ+a u_j(-1/2)v, \qquad j=0,\ldots,s-1.
\end{equation*}
These indices relabel only the undisclosed block within $\mathcal I_m$. The sample weights and fluctuation family remain unchanged. To bound the distance between these datasets, we must account for the optimal matching of their records. At the two endpoint levels,
\begin{align*}
\phi_j(1/2) &=\frac{(j+1/2)\pi}{s}-\frac{(-1)^j\pi}{6s},\\*
\phi_j(-1/2) &=\frac{(j+1/2)\pi}{s}+\frac{(-1)^j\pi}{6s},\\
|u_j(1/2)-u_j(-1/2)| &=2\sin\!\left(\frac{(j+1/2)\pi}{s}\right) \sin\!\left(\frac{\pi}{6s}\right).
\end{align*}
Ordered matching minimizes the maximum distance between two ordered real lists. If $u_1\geq u_2$ and $v_1\geq v_2$, then
\begin{equation*}
\max\{|u_1-v_1|,|u_2-v_2|\} \leq\max\{|u_1-v_2|,|u_2-v_1|\}.
\end{equation*}
Indeed, $u_1-v_1\leq u_1-v_2$ and $v_1-u_1\leq v_1-u_2$ bound the first absolute value; $u_2-v_2\leq u_1-v_2$ and $v_2-u_2\leq v_1-u_2$ bound the second. Repeatedly uncrossing a permutation thus cannot increase its maximum cost. Norms along the unit direction $v$ equal absolute differences of the corresponding coordinates, so
\begin{align}
	\operatorname{dist}(x_+,x_-)
	&=a\max_{0\leq j<s}|u_j(1/2)-u_j(-1/2)|\notag\\*
	&=2a\sin\!\left(\frac{\pi}{6s}\right)
	\max_{0\leq j<s}\sin\!\left(\frac{(j+1/2)\pi}{s}\right)\notag\\
	&\geq2a\sin\!\left(\frac{\pi}{6s}\right)
	\cos\!\left(\frac{\pi}{2s}\right)
	\geq2a\frac{2}{\pi}\frac{\pi}{6s}\frac12
	=\frac{a}{3s}.
	\label{eq:proof-collision-distance}
\end{align}
For the penultimate inequality, at least one angle lies within $\pi/(2s)$ of $\pi/2$. The last inequality uses $s\geq2$ and the bound $\sin t\geq2t/\pi$ on $[0,\pi/2]$. To verify moment cancellation, use the fact that $\mathsf C_s$ has leading coefficient $2^{s-1}$ to write
\begin{equation*}
\prod_{j=0}^{s-1}(u-u_j(\theta)) =2^{1-s}\{\mathsf C_s(u)-\theta\} =u^s+a_1u^{s-1}+\cdots+a_{s-1}u+a_s(\theta).
\end{equation*}
Only the constant coefficient depends on $\theta$. If $M_\ell(\theta)=\sum_j u_j(\theta)^\ell$, Newton's identities imply, for $1\leq\ell<s$,
\begin{align*}
M_1(\theta)&=-a_1,\\*
M_\ell(\theta) &=-a_1M_{\ell-1}(\theta)-\cdots -a_{\ell-1}M_1(\theta)-\ell a_\ell,\\
M_\ell(1/2)-M_\ell(-1/2) &=-\sum_{j=1}^{\ell-1}a_j \{M_{\ell-j}(1/2)-M_{\ell-j}(-1/2)\}=0.
\end{align*}
The last equality follows by induction, while the degree-zero identity is $M_0(1/2)=M_0(-1/2)=s$. Consequently, every polynomial of degree less than $s$ satisfies
\begin{equation*}
\sum_{j=0}^{s-1}P(u_j(1/2)) -\sum_{j=0}^{s-1}P(u_j(-1/2))=0, \qquad \deg P<s.
\end{equation*}
For a general accepted gradient query, the polynomial cancellation leaves a remainder to control. Fix $g\in\mathcal G$ and define the vector-valued function $f(u)=g(o_\circ+a uv)$. By \eqref{eq:query-smoothness},
\begin{equation*}
f^{(s)}(u)=a^s\frac{d^s}{dz^s}g(o_\circ+zv)\bigg|_{z=au}, \qquad \sup_{|u|\leq1}\|f^{(s)}(u)\|\leq C_s a^s.
\end{equation*}
We use the integral representation of divided differences \citep{deboor2005divided} to bound this remainder. For distinct nodes $\xi_0,\ldots,\xi_k$, Lagrange interpolation gives the order-$k$ divided difference
\begin{equation*}
f[\xi_0,\ldots,\xi_k] =\sum_{j=0}^k \frac{f(\xi_j)}{\prod_{\ell\ne j}(\xi_j-\xi_\ell)}.
\end{equation*}
It is the leading coefficient of $\sum_j f(\xi_j)\prod_{\ell\ne j}(u-\xi_\ell)/ \prod_{\ell\ne j}(\xi_j-\xi_\ell)$. Subtracting the expressions with the first and last nodes omitted, then dividing by $\xi_k-\xi_0$, gives the divided-difference recursion. These identities apply componentwise to vector-valued $f$. 
To derive the integral representation, set $\Lambda_k=\{\lambda\in[0,\infty)^k:\sum_{j=1}^k\lambda_j\leq1\}$ and $\lambda_0=1-\sum_{j=1}^k\lambda_j$. The integral
\begin{equation*}
J_k(\xi_0,\ldots,\xi_k;f) =\int_{\Lambda_k}f^{(k)}\!\left( \lambda_0\xi_0+\sum_{j=1}^k\lambda_j\xi_j\right) d\lambda_1\cdots d\lambda_k
\end{equation*}
has the same recursion. Fixing $\lambda_1,\ldots,\lambda_{k-1}$, write $b=1-\sum_{j=1}^{k-1}\lambda_j$. The fundamental theorem of calculus yields
\begin{align*}
&(\xi_k-\xi_0) \int_0^b f^{(k)}\!\left( (b-\lambda_k)\xi_0+ \sum_{j=1}^{k-1}\lambda_j\xi_j+\lambda_k\xi_k\right)d\lambda_k\\*
&\quad=\int_0^b\frac{\partial}{\partial\lambda_k} f^{(k-1)}\!\left( (b-\lambda_k)\xi_0+ \sum_{j=1}^{k-1}\lambda_j\xi_j+\lambda_k\xi_k\right)d\lambda_k\\
&\quad=f^{(k-1)}\!\left( \sum_{j=1}^{k-1}\lambda_j\xi_j+b\xi_k\right) -f^{(k-1)}\!\left( b\xi_0+\sum_{j=1}^{k-1}\lambda_j\xi_j\right),\\
(\xi_k-\xi_0)J_k(\xi_0,\ldots,\xi_k;f) &\quad=J_{k-1}(\xi_1,\ldots,\xi_k;f) -J_{k-1}(\xi_0,\ldots,\xi_{k-1};f).
\end{align*}
The simplex integral is unchanged by a permutation of its barycentric coordinates; those affine substitutions have absolute determinant one. Together with $J_0(\xi_0;f)=f(\xi_0)$, the recursion proves the Hermite--Genocchi formula
\begin{equation*}
f[\xi_0,\ldots,\xi_k]=J_k(\xi_0,\ldots,\xi_k;f).
\end{equation*}
In particular, $\operatorname{vol}(\Lambda_k)=1/k!$ implies
\begin{equation*}
\|f[\xi_0,\ldots,\xi_k]\| \leq\frac1{k!}\sup_{u\in[\min_j\xi_j,\max_j\xi_j]} \|f^{(k)}(u)\|.
\end{equation*}
Taking the vector norm after integration gives this bound without an additional factor depending on $d$. No sum of separate coordinate bounds is needed.

We can now differentiate the level-set identity. Using the leading coefficient of $\mathsf C_s$ converts the derivative into the divided difference just analyzed,
\begin{align}
u_j'(\theta)&=\frac1{(\mathsf C_s)'(u_j(\theta))},\notag\\*
(\mathsf C_s)'(u_j(\theta)) &=2^{s-1}\prod_{\ell\ne j} \{u_j(\theta)-u_\ell(\theta)\},\notag\\
\frac d{d\theta}\sum_{j=0}^{s-1}f(u_j(\theta)) &=\sum_{j=0}^{s-1} \frac{f'(u_j(\theta))}{(\mathsf C_s)'(u_j(\theta))} \notag\\
&=2^{1-s}f'[u_0(\theta),\ldots,u_{s-1}(\theta)]\notag\\*
&=2^{1-s}\int_{\Lambda_{s-1}} f^{(s)}\!\left(\sum_{j=0}^{s-1}\lambda_j u_j(\theta)\right)d\lambda.
\label{eq:proof-level-derivative}
\end{align}
Integrate \eqref{eq:proof-level-derivative} over the unit-length interval $[-1/2,1/2]$ to obtain the full-silo mean bound
\begin{align}
\|G_g(x_+)-G_g(x_-)\| &=\frac1{N_m} \left\|\sum_{j=0}^{s-1} \{f(u_j(1/2))-f(u_j(-1/2))\}\right\|\notag\\
&=\frac1{N_m2^{s-1}} \left\|\int_{-1/2}^{1/2}\int_{\Lambda_{s-1}} f^{(s)}\!\left(\sum_{j=0}^{s-1}\lambda_j u_j(\theta)\right) d\lambda\,d\theta\right\|\notag\\
&\leq\frac1{N_m2^{s-1}} \int_{-1/2}^{1/2}\int_{\Lambda_{s-1}} C_s a^s\,d\lambda\,d\theta\notag\\*
&=\frac{C_s a^s}{N_m2^{s-1}(s-1)!}.
\label{eq:proof-collision-response}
\end{align}
The first equality cancels the known records. Since the bound applies to every $g\in\mathcal G$ for the same pair $x_+,x_-$, it establishes the uniform response inequality in \eqref{eq:collision-pair}. To obtain the minimax lower bound, fix an arbitrary adaptive policy $\mathsf Q$ and decoder $\widehat x$. Respond to each query $g_t$ with
\begin{equation*}
Y_t=\tfrac12\{G_{g_t}(x_+)+G_{g_t}(x_-)\}, \qquad e_t^\pm=Y_t-G_{g_t}(x_\pm).
\end{equation*}
The candidates have the same public information and therefore produce the same first query. If their histories agree through time $t-1$, the deterministic policy chooses the same $g_t$, and the displayed response preserves that agreement. Induction gives identical complete transcripts, without any continuity or differentiability requirement on the query-selection rule. Along either execution, \eqref{eq:proof-collision-response} bounds the required perturbations by
\begin{align*}
\sum_{t=1}^T\|e_t^\pm\|^2 &=\frac14\sum_{t=1}^T \|G_{g_t}(x_+)-G_{g_t}(x_-)\|^2\\*
&\leq\frac{T C_s^2a^{2s}} {N_m^2 2^{2s}\{(s-1)!\}^2},\\
\|e^\pm\| &\leq\frac{C_s\sqrt T\,a^s}{N_m2^s(s-1)!}.
\end{align*}
When this norm is at most $\eta$, both executions are admissible and the decoder returns the same dataset, denoted by $\widehat x_*$. The bottleneck matching distance is a metric on multisets of fixed cardinality. In particular, its triangle inequality follows by composing matchings and applying the triangle inequality for the observation norm. Together with \eqref{eq:proof-collision-distance}, it gives
\begin{align*}
\sup_{\substack{x\in\mathcal X_s\\\|e\|\leq\eta}} \operatorname{dist}(\widehat x(\mathsf T_{\mathsf Q}(x,e)),x) &\geq\max\{\operatorname{dist}(\widehat x_*,x_+), \operatorname{dist}(\widehat x_*,x_-)\}\\
&\geq\frac12 \{\operatorname{dist}(\widehat x_*,x_+) +\operatorname{dist}(\widehat x_*,x_-)\}\\*
&\geq\frac12\operatorname{dist}(x_+,x_-) \geq\frac{a}{6s}.
\end{align*}
For $\eta>0$, take
\begin{equation*}
a=\min\left\{\rho, \left(\frac{2^sN_m(s-1)!\eta}{C_s\sqrt T}\right)^{1/s}\right\}.
\end{equation*}
This choice satisfies the perturbation budget. Taking the infimum over $\mathsf Q$ and $\widehat x$ proves \eqref{eq:reconstruction-risk}. At $\eta=0$, the stated bound is simply nonnegativity of the minimax error. The proof therefore supplies no positive lower bound for exact real-valued responses without further structure. The midpoint transcript is admissible under bounded uncertainty, but need not be produced by a particular nearest-grid quantizer.

A related cancellation argument bounds local conditioning for a fixed response block. Fix queries $g_1,\ldots,g_T$ and let $A\in\mathbb R^{dT\times s}$ be the response Jacobian with respect to the $s$ physical record coordinates along $v$. Now take arbitrary distinct $u_1,\ldots,u_s\in[-1,1]$, without requiring them to be the Chebyshev roots above. At physical coordinates $a u_j$, we will show
\begin{equation}\label{eq:collision-jacobian}
\inf_{\|w\|=1}\|Aw\| \leq\frac{C_s\sqrt{sT}}{N_m(s-1)!}a^{s-1}.
\end{equation}
Choose a unit vector $w\in\mathbb R^s$ in the nullspace of the $(s-1)\times s$ matrix with entries $u_j^\ell$, $0\leq\ell\leq s-2$. Such a vector exists because the matrix has fewer rows than columns. Set $\gamma_t(z)=g_t(o_\circ+zv)$ and apply Taylor's formula,
\begin{align*}
\gamma_t'(a u_j) &=\sum_{\ell=0}^{s-2} \frac{(a u_j)^\ell}{\ell!}\gamma_t^{(\ell+1)}(0) +\frac{(a u_j)^{s-1}}{(s-2)!} \int_0^1(1-b)^{s-2}\gamma_t^{(s)}(b a u_j)\,db,\\
\sum_{j=1}^s w_j\gamma_t'(a u_j) &=\sum_{\ell=0}^{s-2}\frac{a^\ell}{\ell!} \gamma_t^{(\ell+1)}(0)\sum_{j=1}^sw_ju_j^\ell\\*
&\quad+\frac{a^{s-1}}{(s-2)!}\sum_{j=1}^s w_ju_j^{s-1} \int_0^1(1-b)^{s-2}\gamma_t^{(s)}(b a u_j)\,db\\
&=\frac{a^{s-1}}{(s-2)!}\sum_{j=1}^s w_ju_j^{s-1} \int_0^1(1-b)^{s-2}\gamma_t^{(s)}(b a u_j)\,db.
\end{align*}
The derivative of the local mean with respect to the $j$th physical record coordinate is $N_m^{-1}\gamma_t'(a u_j)$. Differentiation with respect to the scaled coordinate $u_j$ would introduce a factor $a$ and define a different Jacobian. With the stated physical-coordinate convention,
\begin{align*}
\|Aw\|^2 &=\frac1{N_m^2}\sum_{t=1}^T \left\|\sum_{j=1}^s w_j\gamma_t'(a u_j)\right\|^2\\
&\leq\frac1{N_m^2}\sum_{t=1}^T \left\{\frac{C_s a^{s-1}}{(s-2)!} \sum_{j=1}^s|w_j||u_j|^{s-1} \int_0^1(1-b)^{s-2}\,db\right\}^2\\
&=\frac{TC_s^2a^{2s-2}}{N_m^2\{(s-1)!\}^2} \left(\sum_{j=1}^s|w_j||u_j|^{s-1}\right)^2\\
&\leq\frac{TC_s^2a^{2s-2}}{N_m^2\{(s-1)!\}^2} \left(\sum_{j=1}^sw_j^2\right) \left(\sum_{j=1}^s|u_j|^{2s-2}\right)
\leq\frac{sTC_s^2a^{2s-2}}{N_m^2\{(s-1)!\}^2}.
\end{align*}
Taking square roots and then the infimum over unit directions proves \eqref{eq:collision-jacobian} for the fixed response block. The minimax lower bound for adaptive reconstruction follows from the common-transcript argument above and does not rely on differentiating the selected queries.
\end{proof}

\subsection{Proof of Theorem~\ref{thm:disclosure}}
\label{app:disclosure}
\begin{proof}
Retain the policy and transcript conventions from Section~\ref{app:reconstruction}. On the candidate ball, the attribute minimax risk is
\begin{equation*}
\mathcal R_T^{\varphi}(\eta) =\inf_{\mathsf Q,\widehat\varphi} \sup_{\substack{\|x-x_\circ\|\leq\rho\\\|e\|\leq\eta}} \bigl|\widehat\varphi(\mathsf T_{\mathsf Q}(x,e))-\varphi(x)\bigr|.
\end{equation*}
To simplify the derivations, use the proof-local abbreviations
\begin{equation*}
c=\nabla\varphi(x_\circ),\qquad b_T=\frac{H_g\rho^2\sqrt T}{2},\qquad b_\varphi=\frac{H_\varphi\rho^2}{2}.
\end{equation*}
As in \eqref{eq:disclosure-design}, every infimum over $A$ below ranges only over stacked Jacobians of accepted $T$-query blocks.

Fix a matrix $A$ associated with an accepted query block and take $\eta\geq0$. The feasible set $\{u:\|u\|\leq\rho,\|Au\|\leq\eta\}$ is nonempty, compact, and symmetric. Its support in the attribute direction is therefore
\begin{equation}\label{eq:proof-recovery-modulus}
\omega_A(\eta)= \max_{\substack{\|u\|\leq\rho\\\|Au\|\leq\eta}}c^\top u =\max_{\substack{\|u\|\leq\rho\\\|Au\|\leq\eta}}|c^\top u|.
\end{equation}
For every feasible $u$ and every $w\in\mathbb R^{dT}$,
\begin{align*}
c^\top u &=(c-A^\top w)^\top u+w^\top Au\\*
&\leq\|c-A^\top w\|\,\|u\|+\|w\|\,\|Au\|\\*
&\leq\rho\|c-A^\top w\|+\eta\|w\|.
\end{align*}
To show that minimizing over $w$ gives equality, first take $\eta>0$. Introduce $z=Au$ and a multiplier $w$ for this equality,
\begin{align}
\omega_A(\eta) &=\sup_{\substack{\|u\|\leq\rho,\,\|z\|\leq\eta\\Au-z=0}} c^\top u\notag\\
&=\inf_{w\in\mathbb R^{dT}} \sup_{\|u\|\leq\rho,\,\|z\|\leq\eta} \{c^\top u-w^\top(Au-z)\}\notag\\
&=\inf_w\left\{ \sup_{\|u\|\leq\rho}(c-A^\top w)^\top u +\sup_{\|z\|\leq\eta}w^\top z\right\}\notag\\*
&=\inf_w\{\rho\|c-A^\top w\|+\eta\|w\|\}.
\label{eq:proof-recovery-dual}
\end{align}
The point $(u,z)=(\mathbf0,\mathbf0)$ is strictly feasible for the two norm constraints and satisfies the affine equality $Au=z$, so the primal and dual values agree \citep{rockafellar1970convex}. Convex duality is used only for this auxiliary problem with norm balls and a linear observation map; it imposes no convexity requirement on the TMLE loss or fluctuation objective.

Write $\Pi_{\ker A}$ for the orthogonal projection onto $\ker A$, and use a dagger for the Moore--Penrose pseudoinverse. For $\eta=0$, the feasible directions lie in the kernel. Orthogonal decomposition handles this case directly, without a constraint qualification for the zero-radius ball,
\begin{align*}
\omega_A(0) &=\max_{\substack{u\in\ker A\\\|u\|\leq\rho}} \langle\Pi_{\ker A}c,u\rangle =\rho\|\Pi_{\ker A}c\|,\\
\inf_w\rho\|c-A^\top w\| &=\rho\inf_{v\in\operatorname{range}(A^\top)}\|c-v\|\\*
&=\rho\|\Pi_{\ker A}c\|.
\end{align*}
The range is closed in finite dimensions. Hence \eqref{eq:proof-recovery-dual} holds for every $\eta\geq0$, including rank-deficient designs, and
\begin{equation}\label{eq:proof-design-infimum}
\omega_T(\eta)=\inf_A\omega_A(\eta).
\end{equation}

For the spectral representation, first take $\eta>0$ and square the norm constraints in the support problem. Introducing nonnegative multipliers $\mu,\nu$ gives
\begin{align}
\omega_A(\eta) &=\inf_{\mu,\nu\geq0} \left\{\mu\rho^2+\nu\eta^2+ \sup_{u\in\mathbb R^q} \bigl[c^\top u-u^\top(\mu I+\nu A^\top A)u\bigr]\right\}.
\label{eq:proof-quadratic-dual}
\end{align}
Strict feasibility again holds at $u=\mathbf0$. If $Q=\mu I+\nu A^\top A$ is positive definite, complete the square to obtain
\begin{align*}
c^\top u-u^\top Q u &=\frac14c^\top Q^{-1}c -\left\|Q^{1/2}u-\frac12Q^{-1/2}c\right\|^2,\\*
\sup_u\{c^\top u-u^\top Q u\} &=\frac14c^\top Q^{-1}c.
\end{align*}
If $Q$ is singular, a nonzero component of $c$ in $\ker Q$ makes the supremum infinite; otherwise its value is $c^\top Q^\dagger c/4$. Both cases are obtained as limits by replacing $Q$ with $Q+\gamma I$ and letting $\gamma\downarrow0$. Indeed, in an eigenbasis of $Q$ with eigenvalues $q_j\geq0$ and coordinates $\widetilde c_j$,
\begin{equation*}
c^\top(Q+\gamma I)^{-1}c =\sum_{q_j>0}\frac{\widetilde c_j^2}{q_j+\gamma} +\gamma^{-1}\sum_{q_j=0}\widetilde c_j^2 \longrightarrow
\begin{cases}
c^\top Q^\dagger c,&c\in\operatorname{range}(Q),\\
+\infty,&c\notin\operatorname{range}(Q).
\end{cases}
\end{equation*}
This perturbation adds $\gamma$ to $\mu$ and changes the remaining dual cost by $\gamma\rho^2$. We may therefore restrict the infimum in \eqref{eq:proof-quadratic-dual} to $\mu>0$ and recover the boundary values by taking limits. Set $\nu=\mu\lambda$, $\lambda\geq0$. For fixed $\lambda$,
\begin{align*}
\inf_{\mu>0} \left\{\mu(\rho^2+\lambda\eta^2) +\frac1{4\mu}c^\top(I+\lambda A^\top A)^{-1}c\right\} &=\left[(\rho^2+\lambda\eta^2) c^\top(I+\lambda A^\top A)^{-1}c\right]^{1/2}.
\end{align*}
For $c\ne\mathbf0$ the minimizing multiplier is
\begin{equation*}
\mu_*=\frac12\left\{ \frac{c^\top(I+\lambda A^\top A)^{-1}c} {\rho^2+\lambda\eta^2}\right\}^{1/2};
\end{equation*}
for $c=\mathbf0$ the infimum is zero as $\mu\downarrow0$. Thus
\begin{equation}\label{eq:proof-modulus-spectrum}
\omega_A(\eta)=\inf_{\lambda\geq0} \left[(\rho^2+\lambda\eta^2) c^\top(I+\lambda A^\top A)^{-1}c\right]^{1/2}, \qquad \eta>0.
\end{equation}
For each fixed $A$, choose an orthonormal eigenbasis $A^\top A v_j=\lambda_jv_j$ and set $c_j=v_j^\top c$. The indexed $\lambda_j$ are eigenvalues; the unindexed $\lambda$ remains the optimization variable in \eqref{eq:proof-modulus-spectrum}. At zero uncertainty,
\begin{align*}
\inf_{\lambda\geq0}\rho^2 c^\top(I+\lambda A^\top A)^{-1}c &=\rho^2\lim_{\lambda\to\infty} \sum_{j=1}^q\frac{c_j^2}{1+\lambda\lambda_j}\\*
&=\rho^2\sum_{\lambda_j=0}c_j^2 =\omega_A(0)^2.
\end{align*}
Thus \eqref{eq:proof-modulus-spectrum} also holds at $\eta=0$. Together with \eqref{eq:proof-design-infimum}, it proves \eqref{eq:disclosure-spectrum} for singular matrices and for $c=\mathbf0$, without requiring the infimum over accepted query blocks to be attained.

For any fixed accepted query block, define $F_T(x)=(G_{g_1}(x)^\top,\ldots, G_{g_T}(x)^\top)^\top$. Whenever $\|u\|\leq\rho$, the segment $x_\circ+t u$, $0\leq t\leq1$, remains in the candidate ball. We may therefore apply the fundamental theorem of calculus and \eqref{eq:attribute-smoothness} along that segment,
\begin{align*}
G_g(x_\circ+u)-G_g(x_\circ)-\nabla_xG_g(x_\circ)u &=\int_0^1\{\nabla_xG_g(x_\circ+t u) -\nabla_xG_g(x_\circ)\}u\,dt,\\
\|G_g(x_\circ+u)-G_g(x_\circ)-\nabla_xG_g(x_\circ)u\| &\leq\int_0^1 H_g t\|u\|^2\,dt =\frac{H_g}{2}\|u\|^2,\\
\varphi(x_\circ+u)-\varphi(x_\circ)-c^\top u &=\int_0^1 \langle\nabla\varphi(x_\circ+t u)-\nabla\varphi(x_\circ),u\rangle dt,\\
|\varphi(x_\circ+u)-\varphi(x_\circ)-c^\top u| &\leq\int_0^1 H_\varphi t\|u\|^2\,dt =\frac{H_\varphi}2\|u\|^2.
\end{align*}
Denote the stacked response remainder by $R_F(u)\in\mathbb R^{dT}$ and the attribute remainder by $R_\varphi(u)\in\mathbb R$. Both are local to this proof and distinct from the statistical remainder $R_2$. The preceding estimates give
\begin{align}
F_T(x_\circ+u)&=F_T(x_\circ)+Au+R_F(u),\notag\\*
\|R_F(u)\|^2 &=\sum_{t=1}^T \|G_{g_t}(x_\circ+u)-G_{g_t}(x_\circ) -\nabla_xG_{g_t}(x_\circ)u\|^2\notag\\*
&\leq\frac{TH_g^2}{4}\|u\|^4\leq b_T^2,\notag\\
\varphi(x_\circ+u)&=\varphi(x_\circ)+c^\top u+R_\varphi(u),\notag\\*
|R_\varphi(u)|&\leq\frac{H_\varphi}2\|u\|^2\leq b_\varphi.
\label{eq:proof-query-remainders}
\end{align}
Because the smoothness bound is uniform over $\mathcal G$, these remainder estimates hold for every block selected by an adaptive policy. 
For the upper risk bound, fix an accepted query block and weights $w\in\mathbb R^{dT}$. Define the affine decoder
\begin{equation*}
\widehat\varphi(Y) =\varphi(x_\circ)+w^\top\{Y-F_T(x_\circ)\}.
\end{equation*}
For $x=x_\circ+u$ and every $\|e\|\leq\eta$, \eqref{eq:proof-query-remainders} gives
\begin{align}
\widehat\varphi(Y)-\varphi(x) &=(A^\top w-c)^\top u+w^\top R_F(u)+w^\top e-R_\varphi(u),\notag\\
|\widehat\varphi(Y)-\varphi(x)| &\leq\|A^\top w-c\|\,\|u\| +\|w\|\,\|R_F(u)\|+\|w\|\,\|e\| +|R_\varphi(u)|\notag\\*
&\leq\rho\|c-A^\top w\|+(\eta+b_T)\|w\|+b_\varphi.
\label{eq:proof-affine-error}
\end{align}
A fixed block is an admissible policy in the original minimax problem. For any positive tolerance, choose a block and weights within that tolerance of the infimum defining $\omega_T(\eta+b_T)$. Letting the tolerance tend to zero proves
\begin{equation*}
\mathcal R_T^{\varphi}(\eta) \leq \omega_T(\eta+b_T)+b_\varphi
\end{equation*}
without requiring an optimal design or weights to exist. Since the argument imposes no ordering between $\eta$ and $b_T$, the upper bound holds in both regimes.

For the lower risk bound, fix any adaptive policy $\mathsf Q$ and its decoder. Answer each requested query with $Y_t=G_{g_t}(x_\circ)$ to construct a reference history. This history determines a particular accepted block $g_1,\ldots,g_T$ and its matrix $A$. Suppose $\eta\geq b_T$. The feasible set in \eqref{eq:proof-recovery-modulus} is compact, so there is a vector $u$ satisfying
\begin{equation*}
\|u\|\leq\rho,\qquad \|Au\|\leq\eta-b_T,\qquad c^\top u=\omega_A(\eta-b_T)\geq0.
\end{equation*}
Both candidate datasets $x_\pm=x_\circ\pm u$ lie in the compatible ball. Along the reference block, assign them the perturbations
\begin{align*}
e_t^\pm&=G_{g_t}(x_\circ)-G_{g_t}(x_\pm),\\*
e^\pm&=\mp Au-R_F(\pm u),\\
\|e^\pm\|&\leq\|Au\|+\|R_F(\pm u)\| \leq\eta-b_T+b_T=\eta.
\end{align*}
We must verify that these perturbations produce the reference block under each candidate, since the policy is adaptive. The public information is identical initially. If previous responses agree with the reference history, the policy chooses its reference query $g_t$; the identity $G_{g_t}(x_\pm)+e_t^\pm=G_{g_t}(x_\circ)$ then gives the reference response. Induction establishes
\begin{equation*}
\mathsf T_{\mathsf Q}(x_+,e^+)=\mathsf T_{\mathsf Q}(x_-,e^-) =\mathsf T_{\mathsf Q}(x_\circ,\mathbf0).
\end{equation*}
Thus the candidates also produce identical state changes and stopping decisions. The adaptive query labels supply no additional information for distinguishing them.

Denote the common decoder output by $a_*$. Then
\begin{align*}
\sup_{\substack{\|x-x_\circ\|\leq\rho\\\|e\|\leq\eta}} |\widehat\varphi(\mathsf T_{\mathsf Q}(x,e))-\varphi(x)| &\geq\max\{|a_*-\varphi(x_+)|,|a_*-\varphi(x_-)|\}\\
&\geq\frac12|\varphi(x_+)-\varphi(x_-)|\\
&=\frac12 |2c^\top u+R_\varphi(u)-R_\varphi(-u)|\\
&\geq c^\top u-\frac12|R_\varphi(u)| -\frac12|R_\varphi(-u)|\\*
&\geq\omega_A(\eta-b_T)-b_\varphi \geq \omega_T(\eta-b_T)-b_\varphi.
\end{align*}
Combining this inequality with nonnegativity of the risk and taking the infimum over policies and decoders proves the lower bound in \eqref{eq:disclosure-sandwich}. The reference-history block need not minimize $\omega_T$. Taking the infimum over all accepted blocks can only decrease its modulus, so the inequality has the required direction.

To compare the two bounds on a relative scale, use the nonnegativity of each term in the design objective. For fixed $A,w$ and $0\leq b_T\leq\eta$, $\eta>0$,
\begin{align*}
&\rho\|c-A^\top w\|+(\eta-b_T)\|w\| =\left(1-\frac{b_T}{\eta}\right) \{\rho\|c-A^\top w\|+\eta\|w\|\} +\frac{b_T}{\eta}\rho\|c-A^\top w\|\\
&\quad\geq\left(1-\frac{b_T}{\eta}\right) \{\rho\|c-A^\top w\|+\eta\|w\|\},\\
&\rho\|c-A^\top w\|+(\eta+b_T)\|w\| =\left(1+\frac{b_T}{\eta}\right) \{\rho\|c-A^\top w\|+\eta\|w\|\} -\frac{b_T}{\eta}\rho\|c-A^\top w\|\\*
&\quad\leq\left(1+\frac{b_T}{\eta}\right) \{\rho\|c-A^\top w\|+\eta\|w\|\}.
\end{align*}
Taking infima yields
\begin{equation*}
\omega_T(\eta-b_T)\geq(1-b_T/\eta)\omega_T(\eta), \qquad \omega_T(\eta+b_T)\leq(1+b_T/\eta)\omega_T(\eta).
\end{equation*}
Combining these inequalities with \eqref{eq:disclosure-sandwich}, and dividing when $\omega_T(\eta)>0$, gives
\begin{equation*}
1-\frac{b_T}{\eta}-\frac{b_\varphi}{\omega_T(\eta)} \leq\frac{\mathcal R_T^{\varphi}(\eta)}{\omega_T(\eta)} \leq1+\frac{b_T}{\eta}+\frac{b_\varphi}{\omega_T(\eta)}.
\end{equation*}
This is the relative-error bound underlying the leading-term interpretation in Section~\ref{sec:robustness}. It does not apply when $\omega_T(\eta)=0$. For example, even if $c=\mathbf0$, a second-order attribute error may remain and is bounded by $b_\varphi$.

A smaller candidate ball gives a lower bound even when $\eta\geq b_T$ fails at the original radius. Restricting the supremum in the minimax risk to a concentric ball of radius $\rho'\leq\rho$ cannot increase the risk. This leaves the derivative $c$ and the design family at $x_\circ$ unchanged. Apply the preceding argument at each radius for which the curvature term is within the uncertainty budget to obtain
\begin{adjustwidth}{-5pt}{-5pt}
\begin{align*}
&\mathcal R_T^{\varphi}(\eta) \geq \sup_{\substack{0\leq\rho'\leq\rho\\H_g\rho'^2\sqrt T/2\leq\eta}} \left[\inf_A\inf_{w\in\mathbb R^{dT}} \left\{\rho'\|c-A^\top w\| +\left(\eta-\frac{H_g\rho'^2\sqrt T}{2}\right)\|w\|\right\} -\frac{H_\varphi\rho'^2}{2}\right]_+.
\end{align*}
\end{adjustwidth}
At $\rho'=0$, take the value to be zero. This includes the case in which no positive radius satisfies the constraint.
The mean-gradient structure allows $H_g$ to be bounded using observation-level derivatives. In ordinary concatenated record coordinates, write a unit perturbation as $u=(u_i)_{i\in\mathcal I_m}$. If $\nabla_o g$ is $C_{\mathrm{obs}}$-Lipschitz in operator norm, the block derivative formula and Cauchy--Schwarz yield
\begin{align*}
\nabla_xG_g(x) &=\frac1{N_m}\bigl[\nabla_o g(o_i)\bigr]_{i\in\mathcal I_m},\\
\|\{\nabla_xG_g(x)-\nabla_xG_g(x')\}u\| &=\frac1{N_m}\left\| \sum_{i\in\mathcal I_m} \{\nabla_o g(o_i)-\nabla_o g(o_i')\}u_i\right\|\\
&\leq\frac{C_{\mathrm{obs}}}{N_m} \sum_{i\in\mathcal I_m}\|o_i-o_i'\|\,\|u_i\|\\
&\leq\frac{C_{\mathrm{obs}}}{N_m} \left(\sum_{i\in\mathcal I_m}\|o_i-o_i'\|^2\right)^{1/2} \left(\sum_{i\in\mathcal I_m}\|u_i\|^2\right)^{1/2}\\*
&=\frac{C_{\mathrm{obs}}}{N_m}\|x-x'\|\,\|u\|.
\end{align*}
Taking the supremum over unit $u$ gives the admissible constant $H_g=C_{\mathrm{obs}}/N_m$ in these coordinates. Omitting fixed known records from the block leaves the bound unchanged. A nonlinear coordinate chart may introduce additional derivative terms, which must be controlled through \eqref{eq:attribute-smoothness}. The factor $N_m^{-1}$ comes from averaging. Converting to raw sums rescales both the signal and its uncertainty, so this normalization does not itself provide privacy.

Equation~\eqref{eq:proof-modulus-spectrum} reads
\begin{equation*}
\omega_A(\eta)^2 =\inf_{\lambda\geq0} (\rho^2+\lambda\eta^2) \sum_{j=1}^q\frac{c_j^2}{1+\lambda\lambda_j}.
\end{equation*}
When $\Pi_{\ker A}c\ne\mathbf0$, substitute the feasible vector $u=\rho\Pi_{\ker A}c/\|\Pi_{\ker A}c\|$ into \eqref{eq:proof-recovery-modulus}. When the projection vanishes, the same lower bound is zero. Thus in either case,
\begin{equation*}
\omega_A(\eta)\geq\rho\|\Pi_{\ker A}c\|, \qquad \eta\geq0.
\end{equation*}
If instead $c\in\operatorname{range}(A^\top)$, put
\begin{equation*}
w=(A^\dagger)^\top c =\sum_{\lambda_j>0}\frac{c_j}{\lambda_j}Av_j.
\end{equation*}
The orthogonality relations $(Av_i)^\top Av_j=v_i^\top A^\top A v_j =\lambda_j\mathbf1\{i=j\}$ then give
\begin{align*}
A^\top w &=\sum_{\lambda_j>0}\frac{c_j}{\lambda_j}A^\top A v_j =\sum_{\lambda_j>0}c_jv_j=c,\\
\|w\|^2 &=\sum_{\lambda_i>0}\sum_{\lambda_j>0} \frac{c_ic_j}{\lambda_i\lambda_j}(Av_i)^\top Av_j\\*
&=\sum_{\lambda_j>0}\frac{c_j^2}{\lambda_j}.
\end{align*}
Substitution into \eqref{eq:proof-affine-error} gives
\begin{equation*}
\sup_{\substack{\|x-x_\circ\|\leq\rho\\\|e\|\leq\eta}} |\widehat\varphi-\varphi(x)| \leq(\eta+b_T) \left(\sum_{\lambda_j>0}\frac{c_j^2}{\lambda_j}\right)^{1/2} +b_\varphi.
\end{equation*}
If $c_j$ is zero, the corresponding small eigenvalue contributes nothing to this bound. Reconstruction of all coordinates, by comparison, must control every direction. We impose no such alignment on the specified sensitive attribute. For the threshold claim, suppose the displayed error is at most $E$. Then $\varphi(x)>a+E$ implies $\widehat\varphi>a$, while $\varphi(x)<a-E$ implies $\widehat\varphi<a$. This establishes the threshold interpretation in Section~\ref{sec:robustness} using an error bound rather than exact attribute recovery.
\end{proof}

\subsection{Proof of Lemma~\ref{lem:local-derivatives}}
\label{app:local-derivatives}
\begin{proof}
Condition only on the partition, which is independent of the targeting observations. The function classes and their covers are deterministic under this conditioning. If $M=1$, there is no pair to compare and the statement is vacuous. Suppose $M\geq2$, and choose an internal cover $\{f_1,\ldots,f_{\mathcal N_n}\}$ of the normalized scalar derivative functions in Assumption~\ref{ass:local-complexity}, with radius $1/n$. Every covering function has absolute value at most one on the common full-measure set in Assumption~\ref{ass:local-regularity}. For this proof only, put $s_{mj}=N_m^{-1}+N_j^{-1}$; the concentration factor $\zeta_n(\xi)$ is as defined in Section~\ref{sec:personalized}.

Fix $f$ in the cover and distinct clients $m,j$. Under the common-law assumption, the two empirical means have the same expectation. Subtracting it gives
\begin{align*}
\frac1{N_m}\sum_{i\in\mathcal I_m}f(O_i) -\frac1{N_j}\sum_{i\in\mathcal I_j}f(O_i) &= \frac1{N_m}\sum_{i\in\mathcal I_m}\{f(O_i)-P_0f\} -\frac1{N_j}\sum_{i\in\mathcal I_j}\{f(O_i)-P_0f\},\\*
&\frac1{N_m}\sum_{i\in\mathcal I_m}P_0f -\frac1{N_j}\sum_{i\in\mathcal I_j}P_0f =P_0f-P_0f=0.
\end{align*}
The centered summands are independent under the conditioning on the partition. A term from client $m$ has range length at most $2/N_m$, and a term from client $j$ has range length at most $2/N_j$. The bounded-variable exponential-moment estimate of \citet{hoeffding1963probability} therefore gives, for every $u\in\mathbb R$,
\begin{align*}
&\mathbb E\exp\!\left[u\left\{ \frac1{N_m}\sum_{i\in\mathcal I_m}f(O_i) -\frac1{N_j}\sum_{i\in\mathcal I_j}f(O_i)\right\}\right]\\*
&\quad= \prod_{i\in\mathcal I_m} \mathbb E\exp\!\left[\frac{u}{N_m}\{f(O_i)-P_0f\}\right] \prod_{i\in\mathcal I_j} \mathbb E\exp\!\left[-\frac{u}{N_j}\{f(O_i)-P_0f\}\right]\\
&\quad\leq \exp\!\left\{\frac{u^2}{8} \left[N_m\left(\frac2{N_m}\right)^2 +N_j\left(\frac2{N_j}\right)^2\right]\right\}\\*
&\quad=\exp\!\left(\frac{u^2s_{mj}}2\right).
\end{align*}
Markov's inequality with positive $u$, minimized at $u=t/s_{mj}$, bounds the upper tail. Applying the same argument to the negative of the sum gives
\begin{align*}
&\mathbb P\!\left( \left|\frac1{N_m}\sum_{i\in\mathcal I_m}f(O_i) -\frac1{N_j}\sum_{i\in\mathcal I_j}f(O_i)\right|>t\right)\\*
&\quad\leq 2\inf_{u>0}\exp\!\left(-ut+\frac{u^2s_{mj}}2\right) =2\exp\!\left(-\frac{t^2}{2s_{mj}}\right), \qquad t>0.
\end{align*}
There are at most $M(M-1)\mathcal N_n$ comparisons, one for each ordered client pair and cover member. A union bound shows that the probability of any comparison exceeding $\zeta_n(\xi)\sqrt{s_{mj}}/2$ is at most
\begin{align*}
\sum_{m\ne j}\sum_{b=1}^{\mathcal N_n}2e^{-\zeta_n(\xi)^2/8} &=2M(M-1)\mathcal N_n \frac{\xi}{4M^2\mathcal N_n}\\*
&=\frac{M-1}{2M}\xi \leq\xi.
\end{align*}
Work on the finite-cover event just constructed, intersected with the probability-one event that all observations belong to the common full-measure set. For any normalized scalar derivative function $f$ in Assumption~\ref{ass:local-complexity}, choose a cover member $f_b$ with $\|f-f_b\|_\infty\leq1/n$. Then
\begin{align*}
\left|\frac1{N_m}\sum_{i\in\mathcal I_m}f(O_i) -\frac1{N_j}\sum_{i\in\mathcal I_j}f(O_i)\right| &\leq \left|\frac1{N_m}\sum_{i\in\mathcal I_m}f_b(O_i) -\frac1{N_j}\sum_{i\in\mathcal I_j}f_b(O_i)\right|\\
&\quad+ \frac1{N_m}\sum_{i\in\mathcal I_m}|f-f_b|(O_i) +\frac1{N_j}\sum_{i\in\mathcal I_j}|f-f_b|(O_i)\\*
&\leq\frac{\zeta_n(\xi)}2\sqrt{s_{mj}}+\frac2n.
\end{align*}
The covering error is small enough to be included in the stated concentration factor, since
\begin{align*}
s_{mj} &=\frac{N_m+N_j}{N_mN_j} \geq\frac4{N_m+N_j}\geq\frac4n,\\*
\frac2n &\leq\frac{\zeta_n(\xi)}2\sqrt{s_{mj}},\\*
\frac{\zeta_n(\xi)}2\sqrt{s_{mj}}+\frac2n &\leq\zeta_n(\xi)\sqrt{s_{mj}},
\end{align*}
where $n\geq2$ and $\zeta_n(\xi)^2\geq8\log4$ suffice for the middle inequality. The resulting event controls the entire class at once, so it also controls the derivative functions selected along the realized optimization trajectory. For any common base state $p$ and parameter $\epsilon$ covered by Assumption~\ref{ass:local-complexity}, apply this scalar bound to the normalized directional gradients. Differentiating the finite empirical sums and using the dual characterization of the Euclidean norm gives
\begin{align*}
&\left\|\nabla_\epsilon\ell_m(p(\epsilon)) -\nabla_\epsilon\ell_j(p(\epsilon))\right\|\\*
&\quad= G_1\sup_{\|v\|=1} \left|\frac1{N_m}\sum_{i\in\mathcal I_m} \frac{v^\top\nabla_\epsilon L(p(\epsilon))(O_i)}{G_1} -\frac1{N_j}\sum_{i\in\mathcal I_j} \frac{v^\top\nabla_\epsilon L(p(\epsilon))(O_i)}{G_1}\right|\\*
&\quad\leq G_1\zeta_n(\xi)\sqrt{s_{mj}}.
\end{align*}
Twice continuous differentiability makes the Hessian difference symmetric. Its operator norm is the supremum of the absolute quadratic form, and the scalar bound for the normalized Hessian functions therefore gives
\begin{align*}
&\left\|\nabla_\epsilon^2\ell_m(p(\epsilon)) -\nabla_\epsilon^2\ell_j(p(\epsilon))\right\|_{\mathrm{op}}\\*
&\quad= \beta\sup_{\|v\|=1} \left|\frac1{N_m}\sum_{i\in\mathcal I_m} \frac{v^\top\nabla_\epsilon^2 L(p(\epsilon))(O_i)v}{\beta} -\frac1{N_j}\sum_{i\in\mathcal I_j} \frac{v^\top\nabla_\epsilon^2 L(p(\epsilon))(O_i)v}{\beta}\right|\\*
&\quad\leq\beta\zeta_n(\xi)\sqrt{s_{mj}}.
\end{align*}
These inequalities establish \eqref{eq:local-derivative-comparison} on the same event, since the cover includes both normalized gradient and Hessian functions. 
The cover was fixed before fitting $p^0$ or choosing the subsequent states and active sets from the sample. We may therefore evaluate the bound along the entire trajectory of either weighting mode, even when initialization and all later fits use the targeting sample.

We finish by verifying the compact-family example following Assumption~\ref{ass:local-complexity}. For this example only, suppose that rescaling a bounded parameter set indexes each normalized derivative family by a subset of $[-1,1]^q$, with a $C$-Lipschitz map from the Euclidean parameter norm to the uniform function norm and $q,C$ fixed. The parameterization includes the base state, the fluctuation coordinates, and the unit direction; it is not restricted to the $d$ fluctuation coordinates. Partition the cube into cells of side length at most $1/(Cn\sqrt q)$ and choose one valid representative from each occupied cell. For any parameter and the representative of its cell,
\begin{align*}
\|\theta-\theta'\| &\leq\sqrt q\,\frac1{Cn\sqrt q}=\frac1{Cn},\\*
\|f_\theta-f_{\theta'}\|_\infty &\leq C\|\theta-\theta'\|\leq\frac1n,\\*
\mathcal N_n &\leq 2\bigl(2+2Cn\sqrt q\bigr)^q,\\*
\log\mathcal N_n &\leq\log2+q\log\bigl(2+2Cn\sqrt q\bigr)=O(\log n).
\end{align*}
The factor two accounts for the two derivative families. Choosing representatives from occupied cells ensures that they belong to the family, as required for an internal cover. Constant maps need only one representative per family.
\end{proof}

\subsection{Proof of Lemma~\ref{lem:local-fourth-order}}
\label{app:local-fourth-order}
\begin{proof}
Fix any iteration on the uniform event of Lemma~\ref{lem:local-derivatives}. If the active set is empty, all quantities in the claim are zero. Otherwise, suppress the fixed superscript $k$ on $\epsilon^k$ and $\hat\epsilon_m^k$, and write
\begin{equation*}
a=\sum_{m\in\mathcal A_k}\alpha_m,\qquad u_m=\epsilon-\hat\epsilon_m.
\end{equation*}
These abbreviations are confined to this proof, and every client sum is over $\mathcal A_k$. The sampling coefficient satisfies $\chi_k=\sum_m\alpha_m(1-\alpha_m/a)/N_m$ for either weighting rule. At each selected interior minimizer, the first- and second-order necessary conditions hold. Together with the definition of the weighted average, they give
\begin{equation*}
\begin{aligned}
\nabla \ell_{m,k}(\hat\epsilon_m)&=\mathbf0, &\nabla^2\ell_{m,k}(\hat\epsilon_m)&\succeq0,\\
\epsilon&=\frac1a\sum_m\alpha_m\hat\epsilon_m, &\sum_m\alpha_m u_m&=\mathbf0.
\end{aligned}
\end{equation*}
The positive semidefiniteness is used only at $\hat\epsilon_m$; the loss may have negative curvature elsewhere on the connecting segments. 
Two exact integral identities allow us to express the loss difference through endpoint derivatives and a fourth-order remainder. By Assumption~\ref{ass:local-regularity}, the third directional derivative along any admissible segment is Lipschitz and hence absolutely continuous. The scalar restriction therefore has an integrable fourth derivative almost everywhere. For such a scalar function $\phi$ on $[0,1]$, repeated integration by parts yields
\begin{align*}
\int_0^1t(1-t)^2\phi^{(4)}(t)\,dt &=\bigl[t(1-t)^2\phi'''(t)\bigr]_0^1 -\int_0^1(1-4t+3t^2)\phi'''(t)\,dt\\*
&=-\bigl[(1-4t+3t^2)\phi''(t)\bigr]_0^1 +\int_0^1(-4+6t)\phi''(t)\,dt\\
&=\phi''(0)+\bigl[(-4+6t)\phi'(t)\bigr]_0^1 -6\int_0^1\phi'(t)\,dt\\*
&=\phi''(0)+2\phi'(1)+4\phi'(0) -6\{\phi(1)-\phi(0)\}.
\end{align*}
Thus
\begin{equation}\label{eq:proof-retire-interpolation}
\begin{aligned}
\phi(1)-\phi(0) ={}&\frac13\phi'(1)+\frac23\phi'(0)+\frac16\phi''(0)\\
&-\frac16\int_0^1t(1-t)^2\phi^{(4)}(t)\,dt.
\end{aligned}
\end{equation}
Using the kernel $t(1-t)$ in the same integration-by-parts calculation gives
\begin{align*}
\int_0^1t(1-t)\phi^{(4)}(t)\,dt &=\bigl[t(1-t)\phi'''(t)\bigr]_0^1 -\int_0^1(1-2t)\phi'''(t)\,dt\\*
&=-\bigl[(1-2t)\phi''(t)\bigr]_0^1 -2\int_0^1\phi''(t)\,dt\\*
&=\phi''(1)+\phi''(0)-2\{\phi'(1)-\phi'(0)\}.
\end{align*}
Equivalently,
\begin{equation}\label{eq:proof-retire-trapezoid}
\phi'(1)-\phi'(0) =\frac{\phi''(0)+\phi''(1)}2 -\frac12\int_0^1t(1-t)\phi^{(4)}(t)\,dt.
\end{equation}
 Apply \eqref{eq:proof-retire-interpolation} to $\phi(t)=\ell_{m,k}(\hat\epsilon_m+tu_m)$. Its derivatives at the endpoints satisfy
\begin{align*}
\phi'(0)&=\nabla \ell_{m,k}(\hat\epsilon_m)^\top u_m=0,\\*
\phi'(1)&=\nabla \ell_{m,k}(\epsilon)^\top u_m,\qquad \phi''(0)=u_m^\top \nabla^2\ell_{m,k}(\hat\epsilon_m)u_m,\\
|\phi'''(t)-\phi'''(s)| &=\bigl|\{\nabla^3\ell_{m,k}(\hat\epsilon_m+tu_m)-\nabla^3\ell_{m,k}(\hat\epsilon_m+su_m)\} [u_m,u_m,u_m]\bigr|\\*
&\leq\beta_4|t-s|\|u_m\|^4.
\end{align*}
It follows that $|\phi^{(4)}(t)|\leq\beta_4\|u_m\|^4$ almost everywhere. The empirical third derivative inherits the same Lipschitz constant from the observation-level bound. Summing the interpolation identity with the client weights gives
\begin{equation}\label{eq:proof-retire-average-identity}
\begin{aligned}
\sum_m\alpha_m\{\ell_{m,k}(\epsilon)-\ell_{m,k}(\hat\epsilon_m)\} &=\frac13\sum_m\alpha_m\nabla \ell_{m,k}(\epsilon)^\top u_m +\frac16\sum_m\alpha_mu_m^\top \nabla^2\ell_{m,k}(\hat\epsilon_m)u_m\\
&\quad-\frac16\sum_m\alpha_m\int_0^1t(1-t)^2 \frac{d^4}{dt^4}\ell_{m,k}(\hat\epsilon_m+tu_m)\,dt.
\end{aligned}
\end{equation}

For the sample-size factor,
\begin{align}
\sum_{m\ne j}\alpha_m\alpha_j(N_m^{-1}+N_j^{-1}) &=2\sum_m\frac{\alpha_m}{N_m} \sum_{j\ne m}\alpha_j\notag\\*
&=2\sum_m\frac{\alpha_m(a-\alpha_m)}{N_m} =2a\chi_k.
\label{eq:proof-retire-sampling-moment}
\end{align}
Next, expanding the pairwise squared distances gives
\begin{align}
\sum_{m,j}\alpha_m\alpha_j\|\hat\epsilon_m-\hat\epsilon_j\|^2 &=\sum_{m,j}\alpha_m\alpha_j (\|\hat\epsilon_m\|^2+\|\hat\epsilon_j\|^2-2\hat\epsilon_m^\top \hat\epsilon_j)\notag\\*
&=2a\sum_m\alpha_m\|\hat\epsilon_m\|^2 -2\left\|\sum_m\alpha_m\hat\epsilon_m\right\|^2\notag\\*
&=2a\left\{\sum_m\alpha_m\|\hat\epsilon_m\|^2-a\|\epsilon\|^2\right\} =2a\sum_m\alpha_m\|\hat\epsilon_m-\epsilon\|^2 =2aV_2^k.
\label{eq:proof-retire-second-moment}
\end{align}
By definition,
\begin{equation}\label{eq:proof-retire-fourth-moment}
\sum_{m,j}\alpha_m\alpha_j\|\hat\epsilon_m-\hat\epsilon_j\|^4=2aV_4^k.
\end{equation}
Diagonal terms vanish in both distance identities. We omit them from \eqref{eq:proof-retire-sampling-moment} because Lemma~\ref{lem:local-derivatives} compares distinct clients; the diagonal derivative differences are already zero. For the gradient term in \eqref{eq:proof-retire-average-identity}, centering first converts the weighted sum into pairwise differences. Cauchy--Schwarz and Lemma~\ref{lem:local-derivatives} then give
\begin{align*}
&\sum_m\alpha_m\nabla \ell_{m,k}(\epsilon)^\top u_m =\frac1a\sum_{m,j}\alpha_m\alpha_j \nabla \ell_{m,k}(\epsilon)^\top(u_m-u_j)\\
&\quad=\frac1{2a}\sum_{m,j}\alpha_m\alpha_j \{\nabla \ell_{m,k}(\epsilon)-\nabla \ell_{j,k}(\epsilon)\}^\top(u_m-u_j)\\
&\quad\leq\frac1{2a} \left(\sum_{m,j}\alpha_m\alpha_j \|\nabla \ell_{m,k}(\epsilon)-\nabla \ell_{j,k}(\epsilon)\|^2\right)^{1/2} \left(\sum_{m,j}\alpha_m\alpha_j\|u_m-u_j\|^2\right)^{1/2}\\
&\quad\leq\frac{G_1\zeta_n(\xi)}{2a} \left(\sum_{m\ne j}\alpha_m\alpha_j (N_m^{-1}+N_j^{-1})\right)^{1/2} \left(\sum_{m,j}\alpha_m\alpha_j \|\hat\epsilon_m-\hat\epsilon_j\|^2\right)^{1/2}\\*
&\quad=\frac{G_1\zeta_n(\xi)}{2a} (2a\chi_k)^{1/2}(2aV_2^k)^{1/2} =G_1\zeta_n(\xi)\sqrt{\chi_kV_2^k}.
\end{align*}
The weighted directions sum to zero, canceling any gradient component common to all clients. The resulting bound depends on cross-client derivative differences, which supply the sample-size factor; separate bounds on the individual gradients would not retain it.

For the curvature term, we use stationarity of each local loss at its own selected minimum. Fix $m\ne j$ and let $v=\hat\epsilon_j-\hat\epsilon_m$. Use the two scalar restrictions $\phi_m(t)=\ell_{m,k}(\hat\epsilon_m+tv)$ and $\phi_j(t)=\ell_{j,k}(\hat\epsilon_m+tv)$. Their endpoint derivatives are
\begin{align*}
\phi_m'(0)&=0, &\phi_m'(1)&=\nabla \ell_{m,k}(\hat\epsilon_j)^\top v,\\*
\phi_j'(0)&=\nabla \ell_{j,k}(\hat\epsilon_m)^\top v, &\phi_j'(1)&=0,\\
\phi_m''(t)&=v^\top\nabla^2\ell_{m,k}(\hat\epsilon_m+tv)v, &\phi_j''(t)&=v^\top\nabla^2\ell_{j,k}(\hat\epsilon_m+tv)v.
\end{align*}
Both fourth derivatives have absolute value at most $\beta_4\|v\|^4$ almost everywhere. Apply \eqref{eq:proof-retire-trapezoid} to both restrictions and add the identities. Collecting the Hessian of each loss at its own minimizer gives
\begin{adjustwidth}{-11pt}{-11pt}
\begin{align*}
&v^\top\{\nabla^2\ell_{m,k}(\hat\epsilon_m)+\nabla^2\ell_{j,k}(\hat\epsilon_j)\}v =\{\nabla \ell_{m,k}(\hat\epsilon_j)-\nabla \ell_{j,k}(\hat\epsilon_j)\}^\top v +\{\nabla \ell_{m,k}(\hat\epsilon_m)-\nabla \ell_{j,k}(\hat\epsilon_m)\}^\top v\\
&\qquad-\frac12v^\top \{\nabla^2\ell_{m,k}(\hat\epsilon_j)-\nabla^2\ell_{j,k}(\hat\epsilon_j)\}v
-\frac12v^\top \{\nabla^2\ell_{j,k}(\hat\epsilon_m)-\nabla^2\ell_{m,k}(\hat\epsilon_m)\}v\\
&\qquad+\frac12\int_0^1t(1-t) \{\phi_m^{(4)}(t)+\phi_j^{(4)}(t)\}\,dt.
\end{align*}
\end{adjustwidth}
Each cross-client derivative difference is now evaluated at a common fluctuation parameter. We can therefore apply Lemma~\ref{lem:local-derivatives} to those differences and the fourth-derivative bound to the integral, obtaining
\begin{adjustwidth}{-12pt}{-12pt}
\begin{align}
&v^\top\{\nabla^2\ell_{m,k}(\hat\epsilon_m)+\nabla^2\ell_{j,k}(\hat\epsilon_j)\}v \leq \{\|\nabla \ell_{m,k}(\hat\epsilon_j)-\nabla \ell_{j,k}(\hat\epsilon_j)\| +\|\nabla \ell_{m,k}(\hat\epsilon_m)-\nabla \ell_{j,k}(\hat\epsilon_m)\|\}\|v\|\notag\\
&\qquad+\frac12 \|\nabla^2\ell_{m,k}(\hat\epsilon_j)-\nabla^2\ell_{j,k}(\hat\epsilon_j)\|_{\mathrm{op}}\|v\|^2\notag
+\frac12 \|\nabla^2\ell_{j,k}(\hat\epsilon_m)-\nabla^2\ell_{m,k}(\hat\epsilon_m)\|_{\mathrm{op}}\|v\|^2\notag\\
&\qquad+\frac12\int_0^1t(1-t) (2\beta_4\|v\|^4)\,dt\notag\\*
&\quad\leq 2G_1\zeta_n(\xi)\sqrt{N_m^{-1}+N_j^{-1}}\|v\| +\beta\zeta_n(\xi)\sqrt{N_m^{-1}+N_j^{-1}}\|v\|^2 +\frac{\beta_4}{6}\|v\|^4.
\label{eq:proof-retire-pair-curvature}
\end{align}
\end{adjustwidth}
The argument has not required positive curvature along the intervening segment.

We now use the positive semidefinite Hessian at each selected minimizer to pass from pairwise directions to the direction of the aggregate. Since $u_m=a^{-1}\sum_j\alpha_j(\hat\epsilon_j-\hat\epsilon_m)$ and $\nabla^2\ell_{m,k}(\hat\epsilon_m)\succeq0$,
\begin{align*}
&\frac1a\sum_j\alpha_j (\hat\epsilon_j-\hat\epsilon_m)^\top \nabla^2\ell_{m,k}(\hat\epsilon_m)(\hat\epsilon_j-\hat\epsilon_m)-u_m^\top \nabla^2\ell_{m,k}(\hat\epsilon_m)u_m\\*
&\quad=\frac1a\sum_j\alpha_j (\hat\epsilon_j-\hat\epsilon_m-u_m)^\top \nabla^2\ell_{m,k}(\hat\epsilon_m)(\hat\epsilon_j-\hat\epsilon_m-u_m) \geq0.
\end{align*}
Multiplying by $\alpha_m$, summing, and exchanging the two client indices in one copy of the sum gives
\begin{align*}
\sum_m\alpha_mu_m^\top \nabla^2\ell_{m,k}(\hat\epsilon_m)u_m &\leq\frac1a\sum_{m,j}\alpha_m\alpha_j (\hat\epsilon_j-\hat\epsilon_m)^\top \nabla^2\ell_{m,k}(\hat\epsilon_m)(\hat\epsilon_j-\hat\epsilon_m)\\*
&=\frac1{2a}\sum_{m,j}\alpha_m\alpha_j (\hat\epsilon_j-\hat\epsilon_m)^\top(\nabla^2\ell_{m,k}(\hat\epsilon_m)+\nabla^2\ell_{j,k}(\hat\epsilon_j))(\hat\epsilon_j-\hat\epsilon_m).
\end{align*}
Substitute \eqref{eq:proof-retire-pair-curvature} and apply Cauchy--Schwarz separately to its linear and quadratic distance terms:
\begin{align*}
&\sum_m\alpha_mu_m^\top \nabla^2\ell_{m,k}(\hat\epsilon_m)u_m \leq\frac{G_1\zeta_n(\xi)}{a} \sum_{m\ne j}\alpha_m\alpha_j \sqrt{N_m^{-1}+N_j^{-1}}\|\hat\epsilon_m-\hat\epsilon_j\|\\
&\qquad+\frac{\beta\zeta_n(\xi)}{2a} \sum_{m\ne j}\alpha_m\alpha_j \sqrt{N_m^{-1}+N_j^{-1}}\|\hat\epsilon_m-\hat\epsilon_j\|^2 +\frac{\beta_4}{12a} \sum_{m,j}\alpha_m\alpha_j\|\hat\epsilon_m-\hat\epsilon_j\|^4\\
&\quad\leq\frac{G_1\zeta_n(\xi)}{a} \left\{\sum_{m\ne j}\alpha_m\alpha_j (N_m^{-1}+N_j^{-1})\right\}^{1/2} \left\{\sum_{m,j}\alpha_m\alpha_j \|\hat\epsilon_m-\hat\epsilon_j\|^2\right\}^{1/2}\\
&\qquad+\frac{\beta\zeta_n(\xi)}{2a} \left\{\sum_{m\ne j}\alpha_m\alpha_j (N_m^{-1}+N_j^{-1})\right\}^{1/2} \left\{\sum_{m,j}\alpha_m\alpha_j \|\hat\epsilon_m-\hat\epsilon_j\|^4\right\}^{1/2} +\frac{\beta_4}{6}V_4^k\\
&\quad=\frac{G_1\zeta_n(\xi)}{a} (2a\chi_k)^{1/2}(2aV_2^k)^{1/2} +\frac{\beta\zeta_n(\xi)}{2a} (2a\chi_k)^{1/2}(2aV_4^k)^{1/2} +\frac{\beta_4}{6}V_4^k\\*
&\quad=2G_1\zeta_n(\xi)\sqrt{\chi_kV_2^k} +\beta\zeta_n(\xi)\sqrt{\chi_kV_4^k} +\frac{\beta_4}{6}V_4^k.
\end{align*}
The normalization factors follow from \eqref{eq:proof-retire-sampling-moment}-\eqref{eq:proof-retire-fourth-moment}. The active mass $a$ is already included in $\chi_k$, and both disagreement moments have the same pairwise normalization. For the remainder in \eqref{eq:proof-retire-average-identity}, the convexity of $v\mapsto\|v\|^4$, which is unrelated to the curvature of the loss, gives
\begin{adjustwidth}{-14pt}{-14pt}
\begin{align*}
\sum_m\alpha_m\|u_m\|^4 &=\sum_m\alpha_m \left\|\sum_j\frac{\alpha_j}{a}(\hat\epsilon_j-\hat\epsilon_m)\right\|^4\\*
&\leq\sum_m\alpha_m \sum_j\frac{\alpha_j}{a}\|\hat\epsilon_j-\hat\epsilon_m\|^4 =2V_4^k,\\
&\left|\frac16\sum_m\alpha_m\int_0^1t(1-t)^2 \frac{d^4}{dt^4}\ell_{m,k}(\hat\epsilon_m+tu_m)\,dt\right| \leq\frac{\beta_4}{6}\sum_m\alpha_m\|u_m\|^4 \int_0^1t(1-t)^2\,dt\\*
&\quad=\frac{\beta_4}{72}\sum_m\alpha_m\|u_m\|^4 \leq\frac{\beta_4}{36}V_4^k.
\end{align*}
\end{adjustwidth}
Combining the three bounds in \eqref{eq:proof-retire-average-identity} now yields
\begin{align*}
&\sum_m\alpha_m\{\ell_{m,k}(\epsilon)-\ell_{m,k}(\hat\epsilon_m)\}\\*
&\quad\leq\frac13G_1\zeta_n(\xi)\sqrt{\chi_kV_2^k} +\frac16\left\{ 2G_1\zeta_n(\xi)\sqrt{\chi_kV_2^k} +\beta\zeta_n(\xi)\sqrt{\chi_kV_4^k} +\frac{\beta_4}{6}V_4^k\right\} +\frac{\beta_4}{36}V_4^k\\*
&\quad=\frac{2G_1\zeta_n(\xi)}{3}\sqrt{\chi_kV_2^k} +\frac{\beta\zeta_n(\xi)}{6}\sqrt{\chi_kV_4^k} +\frac{\beta_4}{18}V_4^k.
\end{align*}
Each loss difference is nonnegative because its local fit is a global minimizer over the admissible fluctuations, giving the lower bound. The deterministic quadratic bound uses only smoothness and stationarity at the selected minimum. Along the admissible segment from $\hat\epsilon_m$ to $\epsilon$,
\begin{align*}
\ell_{m,k}(\epsilon)-\ell_{m,k}(\hat\epsilon_m) &=\int_0^1\nabla \ell_{m,k}(\hat\epsilon_m+tu_m)^\top u_m\,dt\\*
&=\int_0^1\{\nabla \ell_{m,k}(\hat\epsilon_m+tu_m)-\nabla \ell_{m,k}(\hat\epsilon_m)\}^\top u_m\,dt\\*
&\leq\int_0^1\beta t\|u_m\|^2\,dt =\frac\beta2\|u_m\|^2.
\end{align*}
Summing gives $\beta V_2^k/2$. On the uniform derivative event, both upper bounds hold at every iteration, so their minimum is a valid bound in each round. If only one client remains, $\epsilon=\hat\epsilon_m$, so the exact averaging error and every term in both bounds are zero.
\end{proof}

\subsection{Proof of Theorem~\ref{thm:personalized-convergence}}
\label{app:personalized-convergence}
\begin{proof}
Work on the single uniform event of Lemma~\ref{lem:local-derivatives}; the deterministic parts of the argument hold whenever their optimization hypotheses hold. A finite budget returns a prefix of this iteration. After all clients have retired, continue the accounting by keeping their outputs constant and setting the active-set quantities to zero.

Fix any base state $p$ in the analyzed family, and take all local minima over the admissible fluctuations. Since zero is feasible,
\begin{align*}
\min_{\epsilon}\ell_m(p(\epsilon)) &\leq\ell_m(p(\mathbf0))=\ell_m(p),\\*
\Delta_m(p) &=\ell_m(p)-\min_{\epsilon}\ell_m(p(\epsilon)) \geq0.
\end{align*}
When $\Delta_m(p)=0$, zero is a minimizer and is selected by the fixed rule; conversely, selecting zero gives a zero gap. Retaining both the density and the targeting state preserves the same local fluctuation problem. An autonomous local iteration therefore selects zero again and returns $p(\mathbf0)=p$, so induction gives a constant local sequence. Since zero is an interior minimizer and the loss is differentiable there,
\begin{align*}
\mathbf0 &=\left.\nabla_\epsilon\ell_m(p(\epsilon)) \right|_{\epsilon=\mathbf0}\\*
&=\frac1{N_m}\sum_{i\in\mathcal I_m} \left.\nabla_\epsilon L(p(\epsilon))(O_i) \right|_{\epsilon=\mathbf0}\\*
&=\frac1{N_m}\sum_{i\in\mathcal I_m}D(p)(O_i).
\end{align*}
The final equality uses the score identity \eqref{eq:local-score} at the origin of the retained submodel. In particular, if $\tau_m<\infty$, then for every $k\geq\tau_m$,
\begin{equation}\label{eq:proof-retire-absorption}
p_m^k=p^{\tau_m},\qquad \Delta_m(p_m^k)=0,\qquad \frac1{N_m}\sum_{i\in\mathcal I_m}D(p_m^k)(O_i)=\mathbf0.
\end{equation}
The certificate applies to the client's frozen density and targeting state. To obtain the first inequality in the theorem, we lower-bound each local targeting gap by its squared empirical EIF norm. For this calculation only, write $f(\epsilon)=\ell_m(p(\epsilon))$ and $g=\nabla f(\mathbf0)$. The segment from zero to $-g/\beta$ is feasible by Assumption~\ref{ass:local-regularity}. Integrating the loss derivative along it and using gradient Lipschitzness gives
\begin{align*}
\ell_m(p(-g/\beta)) -\ell_m(p) &=\int_0^1 \nabla f(-tg/\beta)^\top(-g/\beta)\,dt\\*
&=-\frac1\beta\|g\|^2 -\frac1\beta\int_0^1 \{\nabla f(-tg/\beta)-g\}^\top g\,dt\\
&\leq-\frac1\beta\|g\|^2 +\frac1\beta\int_0^1 \|\nabla f(-tg/\beta)-g\|\,\|g\|\,dt\\
&\leq-\frac1\beta\|g\|^2 +\frac1\beta\int_0^1\beta\frac{t\|g\|}{\beta}\|g\|\,dt\\*
&=-\frac1{2\beta}\|g\|^2.
\end{align*}
The exact local minimum cannot exceed the loss at this feasible trial point. Hence
\begin{equation}\label{eq:proof-retire-gap-score}
\begin{aligned}
\Delta_m(p) &\geq\ell_m(p)-\ell_m(p(-g/\beta))\\
&\geq\frac1{2\beta}\|g\|^2 =\frac1{2\beta} \left\|\frac1{N_m}\sum_{i\in\mathcal I_m}D(p)(O_i)\right\|^2.
\end{aligned}
\end{equation}
At a retained output both sides are zero by \eqref{eq:proof-retire-absorption}. Summing these inequalities with the client weights and averaging over iterations proves the first inequality in \eqref{eq:personalized-convergence-bound}. Here the EIF is evaluated at the current submodel origin; no endpoint gradient from an earlier submodel is used. The active set after the round-$k$ check is $\mathcal A_k=\{m:k<\tau_m\}$. Consequently,
\begin{align*}
\sum_{m=1}^M\alpha_m\Delta_m(p_m^k) &=\sum_{m:\,k<\tau_m}\alpha_m\Delta_m(p^k) +\sum_{m:\,k\geq\tau_m}\alpha_m\Delta_m(p^{\tau_m})\\*
&=\sum_{m\in\mathcal A_k}\alpha_m\Delta_m(p^k).
\end{align*}
For each active client, decompose the targeting gap into the signed loss decrease under the shared update and the additional loss incurred by averaging its local fit. Inserting the loss at the aggregate gives
\begin{align*}
\Delta_m(p^k) &=\ell_{m,k}(\mathbf0)-\ell_{m,k}(\hat\epsilon_m^k)\\*
&=\ell_{m,k}(\mathbf0)-\ell_{m,k}(\epsilon^k) +\ell_{m,k}(\epsilon^k)-\ell_{m,k}(\hat\epsilon_m^k)\\*
&=\ell_m(p^k)-\ell_m(p^{k+1}) +\ell_{m,k}(\epsilon^k)-\ell_{m,k}(\hat\epsilon_m^k).
\end{align*}
Summing this exact decomposition over iterations and active clients gives
\begin{align}
\sum_{k=0}^{K-1}\sum_{m=1}^M \alpha_m\Delta_m(p_m^k) &= \sum_{k=0}^{K-1}\sum_{m\in\mathcal A_k} \alpha_m\{\ell_m(p^k)-\ell_m(p^{k+1})\}\notag\\*
&\quad+ \sum_{k=0}^{K-1}\sum_{m\in\mathcal A_k} \alpha_m\{\ell_{m,k}(\epsilon^k) -\ell_{m,k}(\hat\epsilon_m^k)\}.
\label{eq:proof-retire-loss-identity}
\end{align}
The fixed base weights keep each client's coefficient unchanged throughout its participation, allowing the loss differences to telescope even as the active set changes.

Reordering the first sum in \eqref{eq:proof-retire-loss-identity} by client makes those participation intervals explicit,
\begin{align*}
\sum_{k=0}^{K-1}\sum_{m\in\mathcal A_k} \alpha_m\{\ell_m(p^k)-\ell_m(p^{k+1})\} &=\sum_{m=1}^M\alpha_m\sum_{k=0}^{K-1} \mathbf1\{k<\tau_m\} \{\ell_m(p^k)-\ell_m(p^{k+1})\}\\
&=\sum_{m=1}^M\alpha_m \sum_{k=0}^{\min\{K,\tau_m\}-1} \{\ell_m(p^k)-\ell_m(p^{k+1})\}\\
&=\sum_{m=1}^M\alpha_m \{\ell_m(p^0)-\ell_m(p^{\min\{K,\tau_m\}})\}\\
&=\sum_{m=1}^M\alpha_m \{\ell_m(p^0)-\ell_m(p_m^K)\}\\*
&\leq\sum_{m=1}^M\alpha_m \{\ell_m(p^0)-\underline\ell_m\}<\infty.
\end{align*}
The sum for $\tau_m=0$ is empty; the sum for $\tau_m=\infty$ ends at $K-1$. The calculation is a finite-sum identity for each realized path. It remains valid at adaptive retirement times without taking expectations or invoking an optional-stopping result.

Lemma~\ref{lem:local-fourth-order} bounds the averaging loss in the second sum of \eqref{eq:proof-retire-loss-identity}. Combining it with the clientwise telescoping bound yields
\begin{align*}
&\sum_{k=0}^{K-1}\sum_{m=1}^M \alpha_m\Delta_m(p_m^k)\\*
&\quad\leq\sum_{m=1}^M\alpha_m\{\ell_m(p^0)-\underline\ell_m\} +\sum_{k=0}^{K-1}\left\{ \frac{2G_1\zeta_n(\xi)}3\sqrt{\chi_kV_2^k} +\frac{\beta\zeta_n(\xi)}6\sqrt{\chi_kV_4^k} +\frac{\beta_4}{18}V_4^k\right\}.
\end{align*}
Dividing by $K$ and using \eqref{eq:proof-retire-gap-score} proves \eqref{eq:personalized-convergence-bound}. The second-order part of Lemma~\ref{lem:local-fourth-order} may replace the braced expression by its minimum with $\beta V_2^k/2$. The derivative event covers all $K$ and both weighting modes because the fixed family in Assumption~\ref{ass:local-complexity} contains every state either mode can visit. Thus the conclusion is simultaneous over their respective trajectories, with no conditioning on the states selected by the sample.

At one nonempty active set write, for this calculation only, $a=\sum_{m\in\mathcal A_k}\alpha_m$ and $s=|\mathcal A_k|$. The common form of the sampling coefficient is
\begin{equation*}
\chi_k=\sum_{m\in\mathcal A_k} \frac{\alpha_m}{N_m}\left(1-\frac{\alpha_m}{a}\right).
\end{equation*}
For uniform base weights this gives
\begin{align*}
a&=\frac{s}{M},\qquad \frac{\alpha_m}{a}=\frac1s,\\*
\chi_k &=\frac1M\sum_{m\in\mathcal A_k} \left(1-\frac1s\right)\frac1{N_m} =\frac{s-1}{Ms}\sum_{m\in\mathcal A_k}\frac1{N_m},\\*
\frac{\chi_k}{a} &=\frac{s-1}{s^2}\sum_{m\in\mathcal A_k}\frac1{N_m}.
\end{align*}
For sample-size weights, retaining the sample totals explicitly,
\begin{align*}
a&=\frac{\sum_{j\in\mathcal A_k}N_j}{n},\qquad \frac{\alpha_m}{a} =\frac{N_m}{\sum_{j\in\mathcal A_k}N_j},\\*
\chi_k &=\frac1n\sum_{m\in\mathcal A_k} \left(1-\frac{N_m}{\sum_{j\in\mathcal A_k}N_j}\right) =\frac1n\left(s- \frac{\sum_{m\in\mathcal A_k}N_m} {\sum_{j\in\mathcal A_k}N_j}\right) =\frac{s-1}{n},\\*
\frac{\chi_k}{a} &=\frac{s-1}{\sum_{m\in\mathcal A_k}N_m}.
\end{align*}
For $s=1$, both coefficients and both disagreement moments vanish, because the aggregate equals the sole client's solution.

For a fixed active set with $s\geq2$, divide the averaging-error bound by its active mass $a$. Its sampling terms then involve $\sqrt{(\chi_k/a)(V_r^k/a)}$, $r=2,4$. Holding the normalized update geometry fixed, the ratio of the two sampling coefficients is
\begin{equation}\label{eq:retirement-imbalance}
\frac{(\chi_k/a)_{\mathrm{unif}}}{(\chi_k/a)_{\mathrm{size}}} =\frac{\bigl(\sum_{m\in\mathcal A_k}N_m\bigr) \bigl(\sum_{m\in\mathcal A_k}N_m^{-1}\bigr)}{s^2} \geq1.
\end{equation}
Indeed,
\begin{align*}
s^2 &=\left(\sum_{m\in\mathcal A_k} \sqrt{N_m}\frac1{\sqrt{N_m}}\right)^2\\*
&\leq\left(\sum_{m\in\mathcal A_k}N_m\right) \left(\sum_{m\in\mathcal A_k}\frac1{N_m}\right).
\end{align*}
Equality holds exactly when all active sample sizes are equal. The ratio in \eqref{eq:retirement-imbalance} is the arithmetic mean of those sizes divided by their harmonic mean. This comparison concerns only the sampling coefficient: the two algorithms may have different active sets, normalized dispersions, and terminal densities. Also, for every client,
\begin{equation*}
\Delta_m(p_m^k) \leq\frac1{\alpha_m} \sum_{j=1}^M\alpha_j\Delta_j(p_j^k).
\end{equation*}
Thus extracting an individual guarantee costs a factor $M$ under uniform weights and $n/N_m$ under sample-size weights.

To relate the accumulated sampling terms to participation duration, first apply Cauchy--Schwarz over the iterations, for both $r=2$ and $r=4$,
\begin{equation*}
\sum_{k=0}^{K-1}\sqrt{\chi_kV_r^k} \leq\left(\sum_{k=0}^{K-1}\chi_k\right)^{1/2} \left(\sum_{k=0}^{K-1}V_r^k\right)^{1/2}.
\end{equation*}
Under sample-size weights, the active-set identity $\mathcal A_k=\{m:k<\tau_m\}$ gives
\begin{align*}
n\sum_{k=0}^{K-1}\chi_k &=\sum_{k=0}^{K-1}(|\mathcal A_k|-1)_+\\*
&=\sum_{k=0}^{K-1}|\mathcal A_k| -\sum_{k=0}^{K-1}\mathbf1\{|\mathcal A_k|>0\}\\
&=\sum_{m=1}^M\sum_{k=0}^{K-1}\mathbf1\{k<\tau_m\} -\sum_{k=0}^{K-1}\mathbf1\{k<\max_m\tau_m\}
=\sum_{m=1}^M\min\{K,\tau_m\} -\min\{K,\max_m\tau_m\}.
\end{align*}
This identity counts the active client-rounds in excess of one client per nonempty round. Under uniform weights, interchanging the finite sums over nonempty active sets gives
\begin{align*}
\sum_{k=0}^{K-1}\chi_k &=\frac1M\sum_{\substack{0\leq k<K\\\mathcal A_k\ne\varnothing}} \left(1-\frac1{|\mathcal A_k|}\right) \sum_{m\in\mathcal A_k}\frac1{N_m}\\*
&=\frac1M\sum_{m=1}^M\frac1{N_m} \sum_{k=0}^{\min\{K,\tau_m\}-1} \left(1-\frac1{|\mathcal A_k|}\right).
\end{align*}
Together, these calculations establish
\begin{equation*}
\begin{aligned}
\sum_{k<K}\chi_k &=\frac1n\left\{\sum_m\min\{K,\tau_m\} -\min\{K,\max_m\tau_m\}\right\}, &&\alpha_m=N_m/n,\\
\sum_{k<K}\chi_k &=\frac1M\sum_m\frac1{N_m} \sum_{k=0}^{\min\{K,\tau_m\}-1} \left(1-\frac1{|\mathcal A_k|}\right), &&\alpha_m=1/M.
\end{aligned}
\end{equation*}
Truncation by $K$ keeps these identities finite even when some retirement times are infinite; empty sums are zero. In the second identity each inner summand has $k<\tau_m$, so its active set is nonempty. The uniform-weight coefficient retains the inverse sample sizes of the participating clients. At one nonempty active set, let $\operatorname{diam}_k:=\max_{m,j\in\mathcal A_k} \|\hat\epsilon_m^k-\hat\epsilon_j^k\|$ be their diameter. With the proof-local active mass $a$, the common moment definition yields
\begin{align*}
\frac{V_2^k}{a} &=\frac1{2a^2}\sum_{m,j\in\mathcal A_k}\alpha_m\alpha_j \|\hat\epsilon_m^k-\hat\epsilon_j^k\|^2 \leq\frac{\operatorname{diam}_k^2}{2},\\
\frac{V_4^k}{a} &=\frac1{2a^2}\sum_{m,j\in\mathcal A_k}\alpha_m\alpha_j \|\hat\epsilon_m^k-\hat\epsilon_j^k\|^4 \leq\frac{\operatorname{diam}_k^4}{2}.
\end{align*}
Substitution in \eqref{eq:local-fourth-order-error} gives
\begin{align*}
&\frac1a\sum_{m\in\mathcal A_k}\alpha_m \{\ell_{m,k}(\epsilon^k)-\ell_{m,k}(\hat\epsilon_m^k)\}\\*
&\quad\leq\frac{2G_1\zeta_n(\xi)}{3\sqrt2} \operatorname{diam}_k\sqrt{\chi_k/a} +\frac{\beta\zeta_n(\xi)}{6\sqrt2} \operatorname{diam}_k^2\sqrt{\chi_k/a} +\frac{\beta_4}{36}\operatorname{diam}_k^4.
\end{align*}
The theorem does not require the diameter to be small or to decay. When local minima are far apart, the deterministic quadratic bound may be tighter.

Finally, the convergence statements in the main text follow from the nonnegativity of the gaps. If the accumulated averaging-error bound in \eqref{eq:personalized-convergence-bound} is sublinear in the iteration count, their weighted temporal average tends to zero. If that accumulated error is bounded as $K\to\infty$, then
\begin{align*}
\sum_{k=0}^{\infty}\sum_{m=1}^M \alpha_m\Delta_m(p_m^k)&<\infty,\\*
\sum_{k=0}^{\infty}\Delta_m(p_m^k) &\leq\frac1{\alpha_m} \sum_{k=0}^{\infty}\sum_{j=1}^M \alpha_j\Delta_j(p_j^k)<\infty \qquad\text{for  } \forall m.
\end{align*}
It follows that $\Delta_m(p_m^k)\to0$ for each client, and \eqref{eq:proof-retire-gap-score} gives convergence of its local empirical EIF mean. 
\end{proof}

\endgroup

\section{Experimental details}\label{app:experiments}
\subsection{Optimization-defined target and fluctuation} For this benchmark, we consider positive densities on $\mathbb R^2$ with finite second moments. In \eqref{eq:experiment-target}, we fix the destinations at $z_1=(-2,-2)$, $z_2=(2,-2)$, $z_3=(-2,2)$, and $z_4=(2,2)$. Each capacity is $1/4$, and the entropy regularization is one. Setting $\theta_4=0$ removes additive nonidentifiability, leaving the three coordinates of $\Psi(p)$. The assignment probabilities and the vector of their first three coordinates are
\begin{equation*}
s_j^{\mathrm{OT}}(o;\theta) =\frac{e^{\theta_j-\|o-z_j\|_2^2/2}} {\sum_{a=1}^4e^{\theta_a-\|o-z_a\|_2^2/2}}, \qquad s_{1:3}^{\mathrm{OT}}=(s_1^{\mathrm{OT}},s_2^{\mathrm{OT}},s_3^{\mathrm{OT}})^\top.
\end{equation*}
At $\theta=\Psi(p)$, their means are $1/4$. The reduced Hessian is
\begin{equation*}
H_{\mathrm{OT}}(p)=\int\left[ \operatorname{diag}\{s_{1:3}^{\mathrm{OT}}(o;\Psi(p))\} -s_{1:3}^{\mathrm{OT}}(o;\Psi(p)) s_{1:3}^{\mathrm{OT}}(o;\Psi(p))^\top\right]p(o)\,do.
\end{equation*}
All four assignment probabilities are positive, so this reduced Hessian is positive definite. Differentiating the population first-order equation along a regular density submodel yields the nonparametric EIF
\begin{equation*}
D(p)(o)=-H_{\mathrm{OT}}(p)^{-1} \left\{s_{1:3}^{\mathrm{OT}}(o;\Psi(p))-\tfrac14\mathbf1_3\right\},
\end{equation*}
which is centered under $p$. Because the pooled EIF equation is equivalent to the first-order equation of the empirical transport objective, a direct pooled OT solve provides an endpoint check. We use $L(p)(o)=-\log p(o)$ with the normalized logistic density fluctuation
\begin{equation*}
p^k(\epsilon)(o)= \frac{2p^k(o)\{1+e^{2\epsilon^\top D(p^k)(o)}\}^{-1}} {\displaystyle\int 2p^k(u) \{1+e^{2\epsilon^\top D(p^k)(u)}\}^{-1}\,du}.
\end{equation*}

\subsection{Data-generating distributions} In DGP~1, observations are independent draws from a three-component bivariate Gaussian mixture with probabilities $(0.5,0.3,0.2)$. In the same component order, the means are $(-2,0)$, $(1.5,1.5)$, and $(1.5,-1.5)$, and the covariance matrices are
\begin{equation*}
\begin{pmatrix}0.8&0.5\\0.5&1.2\end{pmatrix},\qquad
\begin{pmatrix}1&-0.35\\-0.35&0.4\end{pmatrix},\qquad
\begin{pmatrix}0.35&0\\0&0.8\end{pmatrix}.
\end{equation*}
The unequal mixture masses, correlations, and scales lead to different local allocation imbalances despite the equal destination capacities. In DGP~2, we generate independent variables
\begin{equation*}
U\sim0.4N(-1.5,0.6^2)+0.6N(1.5,0.9^2), \qquad Z\sim N(0,1),
\end{equation*}
and set
\begin{equation*}
O=\bigl(U,\,0.5U^2+0.3\sin(2U)+(0.35+0.15U^2)Z\bigr)^\top.
\end{equation*}
This construction gives a curved distribution with location-dependent conditional spread. The DGPs differ both in their geometry and in the probabilities assigned to their mixture components. The targeting sample size is $n=10{,}000$.

\subsection{Client partitions and common initialization} For $M\in\{5,10,20,40\}$, we set the client proportions proportional to $5^{(m-1)/(M-1)}$, $m=1,\ldots,M$. We multiply the proportions by $n$, take integer floors, and assign the remaining observations to clients with the largest fractional parts, breaking ties by client index. The random and spatial partitions of a dataset use exactly the same counts. 

For the random partition, we randomly permute all sample indices and allocate successive blocks of sizes $N_m$. To construct the spatial partition, we first reserve $\lfloor0.2N_m\rfloor$ positions in each client and select the corresponding total number of observations uniformly without replacement. Within this subset, we standardize the score $O_{i1}+0.5O_{i2}$ by its pooled mean and standard deviation, add independent $N(0,1)$ noise, and sort the resulting scores. We fill the reserved client positions in ascending score order, then randomly permute the remaining observations to fill the remaining positions. The spatially biased subset and the larger randomly allocated subset together produce moderate heterogeneity while preserving client sizes and the pooled dataset. The spatial partition depends on the observations and thus falls outside the observation-independent assignment condition in Assumption~\ref{ass:local-complexity}.

We construct $p^0$ as a Gaussian KDE from an independent pilot sample of size $2{,}000$ drawn from the same DGP. The Gaussian kernel covariance is $2{,}000^{-1/3}$ times the pilot sample covariance, following the multivariate Scott rule in dimension two \citep{scott2015multivariate}. We hold the pilot sample, $p^0$, and all numerical integration settings fixed across methods, partitions, and client counts. Fitting the pilot estimator and initially distributing its common representation are setup costs.

\end{document}